\documentclass[journal,twoside]{IEEEtran}

\usepackage[utf8]{inputenc}
\usepackage{graphics} 
\usepackage{amsmath,color} 
\usepackage{amssymb}  
\usepackage{cite}
\usepackage{graphicx} 
\usepackage{hyperref}
\usepackage{epstopdf}
\usepackage{tikz}
\usepackage{tikz-cd}
\usepackage{wasysym}
\usetikzlibrary{shapes,arrows,positioning,calc, fit}
\usetikzlibrary{shapes,arrows,positioning,calc}
\usepackage{grffile} 

\usepackage{amsthm}
\usepackage{bm}
\usepackage{amsbsy}
\usepackage{textcomp}
\usepackage[capitalise]{cleveref}
\ifCLASSOPTIONcompsoc
\usepackage[caption=false,font=normalsize,labelfont=sf,textfont=sf]{subfig}
\else
\usepackage[caption=false,font=footnotesize]{subfig}
\fi
\usepackage{etoolbox}
\robustify{\subref}
\graphicspath{{figures/}}

\newcommand{\highlight}[1]{\colorbox{lightgray}{$\displaystyle #1$}}

\title{Robust Task-Space Quadratic Programming for Kinematic-Controlled Robots}

\author{
	Mohamed Djeha, 
	Pierre Gergondet, and 
	Abderrahmane Kheddar,~\IEEEmembership{Fellow,~IEEE}
	\thanks{Manuscript received May 17, 2022; revised February 14, 2023; accepted May 10, 2023. Date of publication Xxxxxxxx X, 20XX; date of current version Xxxxxxxxx X, 20XX. This paper was recommended for publication by Associate Editor Alexander Dietrich and Editor Paolo Robuffo Giordano upon evaluation of the reviewers comments. (\emph{Corresponding author: Mohamed Djeha.})}
	\thanks{This work is supported in part by the Research Project I.AM. through the European Union H2020 program (GA 871899).}
	\thanks{M. Djeha is with the \'Ecole Militaire Polytechnique (EMP), Bordj El-Bahri, Algiers, Algeria. {\tt\small mohamed.djeha@lirmm.fr}}
	\thanks{A. Kheddar and P. Gergondet are with the CNRS-AIST Joint Robotics Laboratory, IRL, Tsukuba, Japan. {\tt\small pierre.gergondet@gmail.com}} 
	\thanks{A. Kheddar is also with the CNRS-University of Montpellier LIRMM, Montpellier, France. {\tt\small kheddar@lirmm.fr}}
	\thanks{This paper has supplementary video downloadable material available at http://ieeexplore.ieee.org.}
	\thanks{Color versions of one or more of the figures in this paper are available online at http://ieeexplore.ieee.org.}
	\thanks{Digital Object Identifier 00.0000/XXX.202X.0000000}
}

\def\actConf{\bm{\hat{q}}}
\def\actConfDot{\bm{\dot{\hat{q}}}}
\def\actConfDDot{{\bm{\ddot{\hat{q}}}}}
\def\desConf{\bm{\hat{q}}_\mathrm{d}}
\def\desConfDot{\bm{\dot{\hat{q}}}_\mathrm{d}}
\def\desConfDDot{\bm{\ddot{\hat{q}}}_\mathrm{d}}

\def\actJointDyn{\bm{\hat{x}}}
\def\actJointDynDot{\bm{\dot{\hat{x}}}}

\def\desJointDyn{\bm{\hat{x}}_\mathrm{d}}

\def\jointTrackErr{\boldsymbol{\phi}}
\def\jointDynTrackErr{\boldsymbol{\hat{\phi}}}
\def\jointDynTrackErrDot{\boldsymbol{\dot{\hat{\phi}}}}
\def\jointDisturbIn{\boldsymbol{\tau}_{\rm{l}}} 
\def\jointCrtlIn{\desConfDDot}
\def\taskCrtlIn{\boldsymbol{\mu}}
\def\actJac{\mathbf{J}}

\def\desJac{\mathbf{J}_\mathrm{d}}
\def\desJacDot{\mathbf{\dot{J}}_\mathrm{d}}
\def\fkin{\bm{s}}

\def\fkinRef{\bm{s}_\mathrm{ref}}
\def\fkinRefDot{\bm{\dot{s}}_\mathrm{ref}}
\def\fkinRefDDot{\bm{\ddot{s}}_\mathrm{ref}}
\def\actOut{\bm{e}}
\def\actOutDot{\bm{\dot{e}}}

\def\desOut{\bm{e_\mathrm{d}}}
\def\desOutDot{\bm{\dot{e}_\mathrm{d}}}
\def\desOutDDot{\bm{\ddot{e}_\mathrm{d}}}
\def\desTaskOut{{\boldsymbol{\eta}_\mathrm{d}}}
\def\desTaskOuti{{\boldsymbol{\eta}_\mathrm{d}^{j}}}

\def\desTaskOutDot{{\boldsymbol{\dot{\eta}}_\mathrm{d}}}

\def\desBfuncOut{{\boldsymbol{\eta}_{\mathrm{d}}^{\bfunc}}}
\def\desBfuncOutDot{{\boldsymbol{\dot{\eta}_{\mathrm{d}}^{\bfunc}}}}
\def\actBfuncOut{\boldsymbol{\eta}^{\bfunc}}

\def\actTaskOut{\boldsymbol{\eta}}

\def\taskTrackErr{\actTaskOut_{\jointTrackErr}}

\def\AdesTaskOut{\mathbf{A}_{\desTaskOut}}
\def\BdesTaskOut{\mathbf{B}_{\desTaskOut}}
\def\FdesTaskOut{\mathbf{F}_{\desTaskOut}}
\def\FdesTaskOutRobust{\mathbf{\check{F}}_{\desTaskOut}}

\def\setC{{\cal C}}

\def\setS{{\cal S}}
\def\setR{{\cal R}}
\def\setCd{{{\cal C}_{\mathrm{d}}}}

\def\classKL{{\cal KL}}

\def\bfunc{h}
\def\bfuncDot{\dot{\bfunc}}

\def\desBfunc{h_{\mathrm{d}}}
\def\desBfuncDot{\dot{\bfunc}_{\mathrm{d}}}
\def\desBfuncDDot{\ddot{\bfunc}_{\mathrm{d}}}
\def\actJacBfunc{\mathbf{J}^\bfunc}

\def\desJacBfunc{\mathbf{J}_\mathrm{d}^{\bfunc}}
\def\desJacBfuncDot{\mathbf{\dot{J}}_\mathrm{d}^{\bfunc}}
\def\bfuncCrtlIn{{\mu}^{\bfunc}}
\def\taskCrtlIni{\boldsymbol{\mu}^{j}}
\def\AdesBfuncOut{\mathbf{A}_{\desBfuncOut}}
\def\BdesBfuncOut{\mathbf{B}_{\desBfuncOut}}
\def\CdesBfuncOut{\mathbf{C}_{\desBfuncOut}}
\def\FdesBfuncOut{\mathbf{F}_{\desBfuncOut}}
\def\FdesBfuncOutRobust{\mathbf{\check{F}}_{\desBfuncOut}}
\def\bfuncTrackErr{\boldsymbol{\eta}^{\bfunc}_{\jointTrackErr}}

\def\constraintStiffness{{_{\mathrm{b}}\taskStiffness}}
\def\constraintDamping{{_{\mathrm{b}}\taskDamping}}
\def\constraintIntegralDamping{{_{\mathrm{b}}\taskIntegralDamping}}

\def\bfuncDelta{\delta^{\bfunc}}
\def\bfuncDeltaMax{\delta^{\bfunc}_{\max}}
\def\bfuncPsi{{\boldsymbol{\psi}^{\bfunc}}}
\def\taskPsi{{\boldsymbol{\psi}}}
\def\bfuncPsiSet{{{\Psi}^{\bfunc}}}
\def\PdesBfuncOut{\mathbf{P}_{\desBfuncOut}}
\def\QdesBfuncOut{\mathbf{Q}_{\desBfuncOut}}

\def\KBfuncEq{\mathbf{\check{K}}^{\bfunc}}

\def\genActConf{\bm{q}}
\def\genActConfDot{\boldsymbol{\alpha}_{\genActConf}}
\def\genActConfDDot{\boldsymbol{\dot{\alpha}}_{\genActConf}}
\def\genDesConf{\bm{q}_{\mathrm{d}}}
\def\genDesConfDot{\boldsymbol{\alpha}_{\genDesConf}}
\def\genDesConfDDot{\boldsymbol{\dot{\alpha}}_{\genDesConf}}
\def\genActJointDyn{\bm{x}}

\def\genDesJointDyn{\bm{x}_{\mathrm{d}}}
\def\genDesJointDynDot{{\boldsymbol{\alpha}}_{\genDesJointDyn}}

\def\UgenDesJointDyn{\genDesConfDDot}
\def\actFBDyn{\genActJointDyn^{\mathrm{FB}}}

\def\desFBDyn{\genDesJointDyn^{\mathrm{FB}}}

\def\FB{\boldsymbol{\xi}}
\def\FBDot{\bm{v}}
\def\FBDDot{\bm{\dot{v}}}
\def\desFB{\boldsymbol{\xi}_{\mathrm{d}}}
\def\desFBDot{\bm{v}_{\mathrm{d}}}
\def\desFBDDot{\bm{\dot{v}}_{\mathrm{d}}}
\def\genRefConf{\bm{q}_\mathrm{ref}}
\def\desTau{\boldsymbol{\tau}_{\mathrm{d}}}

\def\Tau{\bm{\tau}}
\def\force{\bm{f}}
\def\massMat{\mathbf{M}}

\def\nnLinTorqueVec{\bm{h}}

\def\selectMat{\mathbf{S}}

\def\quaternion{\mathbf{\underline{q}}}

\def\actCoM{\text{\bf CoM}(\genActJointDyn)}
\def\desCoM{\text{\bf CoM}(\genDesJointDyn)}
\def\weight{w}
\def\taskStiffness{\mathbf{K}_{\rm s}}
\def\taskDamping{\mathbf{K}_{\rm d}}
\def\taskIntegralDamping{\mathbf{K}_{\rm i}}

\def\taskIntegralDampingi{\mathbf{K}^{j}_{\rm i}}
\def\taskGains{\mathbf{K}}
\def\taskGainsPsi{\mathbf{L}}
\def\taskGainsPsii{\mathbf{L}^{j}}
\def\constraintStiffness{{K}^{h}_{\rm s}}
\def\constraintDamping{{K}^{h}_{\rm d}}
\def\constraintIntegralDamping{{K}^{h}_{\rm i}}
\def\constraintGains{{\mathbf{K}}^{h}}
\def\constraintGainsPsi{{\mathbf{L}}^{h}}

\def\desJacContact{\mathbf{J}_\mathrm{d}^{\mathrm{c}}}
\def\actJacContact{\mathbf{J}^{\mathrm{c}}}
\def\desJacContactDot{\mathbf{\dot{J}}_\mathrm{d}^{\mathrm{c}}}

\def\fkini{\fkin^{j}}
\def\weighti{\weight_{j}}
\def\taskPsii{\boldsymbol{\psi}^{j}}
\def\desJacTaski{\mathbf{J}^{j}_{\rm d}}
\def\desJacDotTaski{\mathbf{\dot{J}}^{j}_{\rm d}}
\def\eye{\mathbf{I}}
\def\zero{\mathbf{0}}
\newcommand{\norm}[1]{\begin{Vmatrix}
		#1
\end{Vmatrix}}
\newcommand{\mcrtc}{\texttt{mc\_rtc}}
\newcommand{\tp}{\ensuremath{^{\mathsf{T}}}}	
\newcommand{\inR}{\in \mathbb{R}}
\newtheorem{theorem}{Theorem}
\newtheorem{proposition}{Proposition}
\newtheorem{definition}{Definition}

\newtheorem{assumption}{Assumption}
\newtheorem{remark}{Remark}

\tikzset{
	block/.style = {draw, fill=white, rectangle, minimum height=1.5em, minimum width=2em},
	tmp/.style  = {coordinate}, 
	sum/.style= {draw, fill=white, circle, node distance=1cm},
	input/.style = {coordinate},
	output/.style= {coordinate},
	pinstyle/.style = {pin edge={to-,thin,black}},
	dotted_block/.style={draw=black!50!white, line width=1.5pt, dash pattern=on 3pt off 3pt on 3pt off 3pt, inner ysep=3mm,inner xsep=2mm, rectangle, rounded corners}
}
\newcounter{mycounter} 	

\begin{document}
	
	\IEEEpubid{0000--0000/00\$00.00~\copyright~2023 IEEE}
	
	\maketitle

	\begin{abstract}
		Task-space quadratic programming (QP) is an elegant approach for controlling robots subject to constraints. Yet, in the case of kinematic-controlled (i.e., high-gains position or velocity) robots, closed-loop QP control scheme can be prone to instability depending on how the gains related to the tasks or the constraints are chosen. 
		In this paper, we address such instability shortcomings. First, we highlight the non-robustness of the closed-loop system against non-modeled dynamics, such as those relative to joint-dynamics, flexibilities, external perturbations, etc. Then, we propose a robust QP control formulation based on high-level integral feedback terms in the task-space including the constraints. The proposed method is formally proved to ensure closed-loop robust stability and is intended to be applied to any kinematic-controlled robots under practical assumptions. We assess our approach through experiments on a fixed-base robot performing stable fast motions, and a floating-base humanoid robot robustly reacting to perturbations to keep its balance. 
	\end{abstract}
	
	
	\section{Introduction}
	\IEEEPARstart{T}{ask-space} sensory control~\cite{samson1991book,ramadorai1994tcst} reached a high-level of maturity thanks to advances in numerical optimization methods. Non-linear task-space controllers can be formulated as local quadratic program (in short, QP control); which can handle several task-objectives and constraints using different sensors (embedded or external) for single or multiple different robots, see e.g.,~\cref{fig:HRP tasks and constraints}. QP controllers output desired joint torque $\desTau$ and/or desired robot-state acceleration $\genDesConfDDot$ that minimize at best (least-square sense) each task error, while ensuring that the robot state is within a set $\setC$ of predefined constraints (also called \emph{safety constraints} in control~\cite{ames2017tac}).  More particular to this work, we are interested in kinematic constraints (e.g., motion bounds in the joint or the task spaces, collision avoidance in the Cartesian space, field-of-view bounds in the image space, etc.)~\cite{cortez2021tcst}; the task error is typically steered by a task-space PD controller~\cite{djeha2020ral}. 
	 \IEEEpubidadjcol 
	 	\begin{figure}[t!]
		\centering
		\includegraphics[scale=0.3]{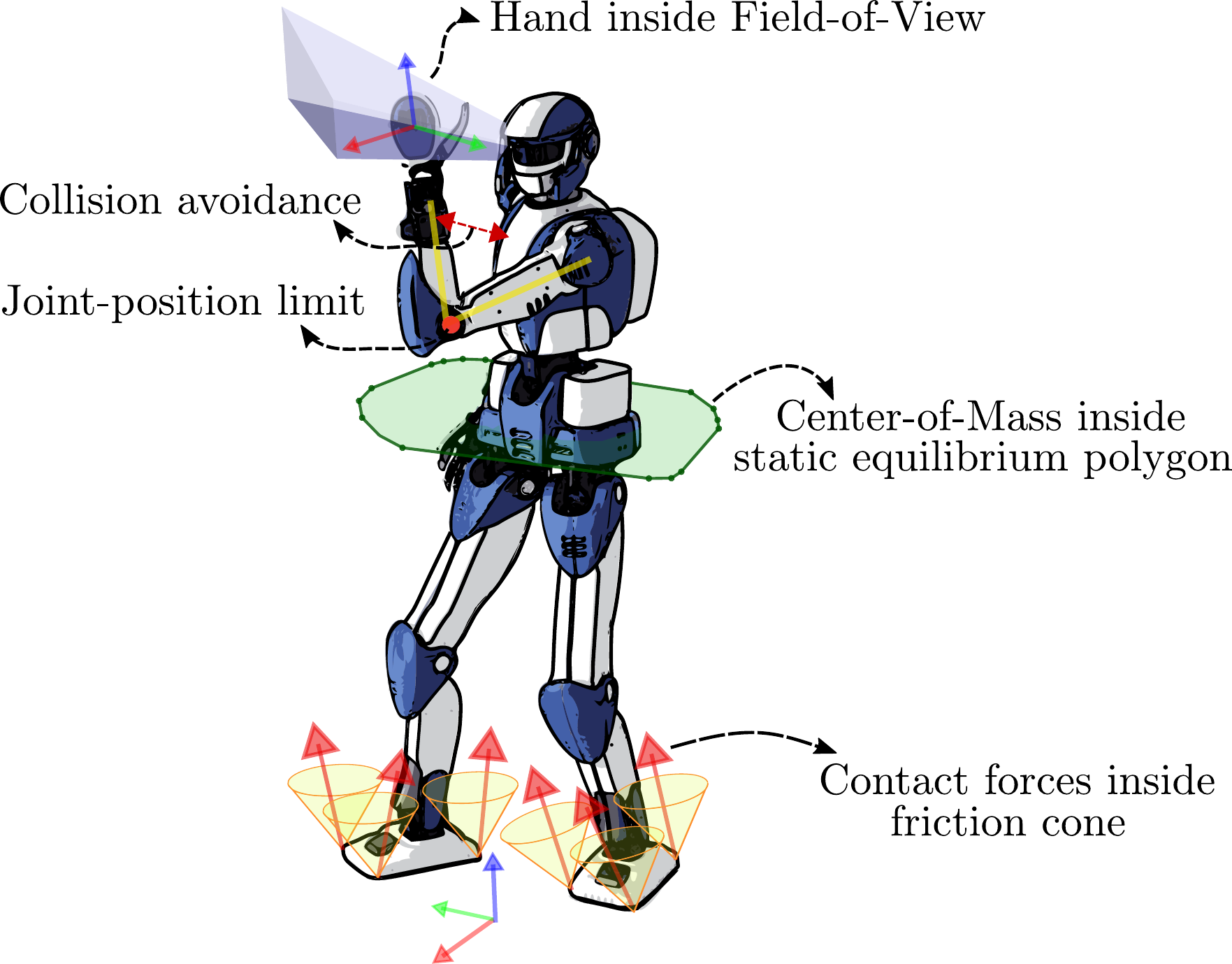}
		\caption{Multi-objective control: HRP-4 robot right hand reaching a Cartesian target while being subject to several constraints.}
		\label{fig:HRP tasks and constraints}
	\end{figure}
	
	QP control has been successfully applied to complex robots and use-cases~\cite{escande2014ijrr,salini2011icra,herzog2016autonomousRobot,kuindersma2016autonomousRobot,nava2020ral,hamed2020ral,englsberger2015tro,klemm2020ral,basso2020ifac}. Yet, several research reported sporadic unstable behaviors of relative severity (e.g., strong sustained oscillations), see e.g.,~\cite{feng2015journalOfFieldRobotics,johnson2015journalOfFieldRobotics,dedonato2017frontiers,koolen2016ijhr}. 
	These works used torque-controlled robots with
	\emph{software-implemented} joint controllers (with the desired joint position and/or velocity as control input; see \cref{fig:leaky integrator QP}) that add a joint-feedback torque to increase the joint stiffness at the expense of pure torque-control compliance~\cite{englsberger2014humanoids}. In particular,~\cite{feng2015journalOfFieldRobotics} noticed that oscillations and undesired behaviors are related to the double integration of the QP output $\genDesConfDDot$. 
	However, no further investigation was made to elucidate the cause. Instead, only workaround solutions have been proposed to mitigate the instability issue. These palliative methods can be sorted into two categories: (i) \emph{low-level joint approaches} that prevent $\genDesConfDDot$ double integration from diverging; typically by implementing a leaky integrator~\cite{hopkins2015icra}; and (ii) \emph{high-level approaches}  where the QP formulation is substantially modified at the expense of a complex control-architecture~\cite{feng2015journalOfFieldRobotics}, or by accounting for the joint feedback terms in the QP to adapt their gains~\cite{lee2022frontiersRobS} or for constraint feasibility concerns~\cite{cisneros2018iros}. Other approaches reported that lowering the task gains helps mitigate the instability~\cite{koolen2016ijhr,johnson2015journalOfFieldRobotics}. This induces that task gains are also an interfering factor. Similar observation is made in~\cite{djeha2020ral,singletary2022csl,molnar2022ral} concerning the gains of the constraint formulation. Unfortunately, since the joint controllers are software-implemented, their model and parameters are known. 	
	
	Conversely, stiff kinematic-controlled robots\footnote{In literature, they are referred to in different ways: position-controlled robots, velocity-controlled robots, low-level impedance-controlled robots~\cite{yang2018humanoids,iskandar2020iros} and stiffness-controlled robots~\cite{pang2022icra}.} are torque-controlled robots with high-gains \emph{hardware-implemented} joint controllers (\cref{fig:position controller}). In this work, we are interested in high-stiffness joint controllers with desired joint position $\desConf$ or velocity $\desConfDot$ as input. This control scheme is widely implemented in robotics and automation industry as it does not require knowledge of the robot's dynamics~\cite{garcia2004iros,rossi2014iros,zanchettin2017elsevier,polverini2017ral,suarez2018sciRob,lim2016journalOfFieldRobotics,singletary2022ral,shi2022machines}. 	
	
	The closed-loop task-space QP controller combined with a kinematic-controlled robot,~\cref{fig:QP scheme for kinematic-control robots}\subref{subfig:feedback QP}, is also prone to instability. This is because the joint-level controllers are not considered at the QP controller level.
	Such instability has been unnoticed in some control implementations that operate in feedforward (\cref{fig:QP scheme for kinematic-control robots}\subref{subfig:feedforward QP}). This leads to a decoupled control (similar to the approach adopted in~\cite{feng2015journalOfFieldRobotics}), delegating the control accuracy to the joint controllers~\cite{bouyarmane2018tac,zanchettin2017elsevier,polverini2017iros_a}. Even though the latter have generally high gains to keep the joint tracking error as small as possible, the user has no guarantee on the accuracy of the performed motion\footnote{\url{https://youtu.be/gTVi1QsLQU4}}. In such cases, frequent initializations of the controller are needed (i.e., start a new instance of the QP control with an update of the model at task switching, or integrator memory reset), to lower the discrepancy between real and control-model states due to non-modeled flexibilities or external disturbances.
	
	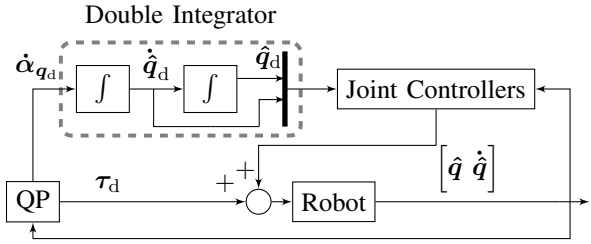
\begin{figure}
		\centering
		\begin{tikzpicture}[auto, node distance=2cm,>=latex']
			\node [input, name=rinput] (rinput) {};
			\node [block](QP) {QP};
			\node [tmp, right of=QP, node distance = 01 cm](InvDyn) {};
			\node [tmp, above of=InvDyn, node distance = 1.5 cm] (tmp1){};
			\node [block, left of = tmp1, node distance = 0.0 cm,xshift=-2pt] (integrator) {$\int$};
			\node [block, right of = integrator, node distance = 1.5 cm,xshift=-2pt] (integrator2) {$\int$};
			
			\node [tmp, right of = integrator2, node distance = 1.0 cm] (mux) {};
			\node [tmp, below of = integrator2, node distance = 0.5 cm] (tmpBelowIntegrator2) {};
			\node [dotted_block, fit = (integrator) (integrator2) (mux) ] (pos controller) {};	
			\node at (pos controller.north) [above, inner sep=1.5mm] {Double Integrator};
			\node [tmp, below of = mux, node distance = 0.5 cm] (tmpBelowMux) {};
			\node [block, right of = mux, node distance = 2 cm] (feedback) {Joint Controllers};	
			\node [tmp, above of=feedback, node distance = 0.5 cm] (xBis){};
			\node [tmp, above of=integrator, node distance = 0.6 cm] (xBisInt){};
			\node [sum, right of = InvDyn, node distance = 2 cm] (sum){};
			\node [tmp, right of=integrator, node distance = 2.2 cm] (xd){};
			\node [block, right of=sum, node distance = 1 cm](robot) {Robot};
			\node [tmp, below of=feedback, node distance = 0.75 cm](sumTmp) {};
			\node [tmp, right of=robot, node distance = 2.5 cm](x){};
			\node [tmp, right of=x, node distance = 0.5 cm](x2){};
			\node [tmp, right of=x2, node distance = 0.4 cm](x3){};
			\node [tmp, below of = QP, node distance = 0.49cm] (aux){};
			
			\draw [line width=2pt](3.35,1.0)--(3.35,2.0);
			\draw [-] (QP) -- node[above,pos=1]{$\desTau$}(InvDyn);
			\draw [->] (InvDyn) -- node[above,pos=0.85]{$+$}(sum);
			\draw [->] (QP.north) |- node[above, pos = 0.55]{$\genDesConfDDot$}(integrator.west);
			\draw [->] (sum) -- node{}(robot);
			\draw [->] (integrator.east) -- node{$
				\desConfDot
				$}(integrator2.west);
			\draw [-] ([xshift=9pt]integrator.east) |- node{}(tmpBelowIntegrator2);
			\draw [-] (tmpBelowIntegrator2) -- node{}([xshift=-10pt]tmpBelowMux);
			\draw [->] ([xshift=-10pt]tmpBelowMux) |- node{}([yshift=-4pt]mux);
			\draw [->] ([yshift=+4pt]integrator2.east) -- node[pos=0.65]{$
				\desConf
				$}([yshift=4pt]mux.west);
			\draw [->] (mux) -- node{}(feedback.west);
			\draw [-] (robot) -- node{$
				\left[\actConf \  \actConfDot\right]
				$}(x2);
			\draw [-] ([xshift = 4pt]x2) |- (aux);
			\draw [->] (x2) -- (x3);
			\draw [->] (aux) -- node{}(QP.south);
			\draw [->] ([xshift = 4pt]x2) |- (feedback.east);
			\draw [-] (feedback.south) -- (sumTmp);
			\draw [->] (sumTmp) -| node[pos=1, above, xshift =-5pt]{$+$}(sum.north);
			
		\end{tikzpicture}
		
		\caption{QP control for torque-controlled robot with additional joint feedback.}
		\label{fig:leaky integrator QP}
	\end{figure}
		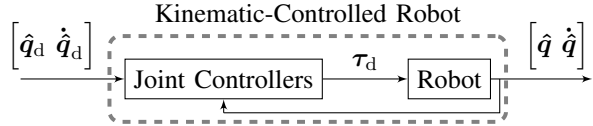
\begin{figure}[t!]
			\centering
			\begin{tikzpicture}[auto, node distance=2cm,>=latex']
				\node [input, name=rinput] (rinput) {};
				\node [tmp, right of=rinput, node distance = 0.0 cm](u){};
				\node [block, right of=u, node distance = 2.7 cm](joint controller) {Joint Controllers};
				\node [tmp, below of= joint controller, node distance = 0.45 cm] (tmp){};
				\node [block, right of=joint controller, node distance = 3 cm](robot) {Robot};
				\node [tmp, right of= robot, node distance = 1.9 cm] (tmp2){};
				\node [dotted_block, fit = (joint controller) (robot) (robot)] (pos controller) {};		
				\node at (pos controller.north) [above, inner sep=1.5mm] {Kinematic-Controlled Robot};
				\draw [->] (rinput) -- node[pos=0.3]{$
					\left[\desConf \ \desConfDot\right]
					$}(joint controller);
				\draw [->] (joint controller) -- node{$\desTau$}(robot);
				\draw [->] (robot) -- node[pos=0.65]{$
					\left[\actConf \  \actConfDot\right]
					$}(tmp2);
				\draw [-] ([xshift=-35pt]tmp2) |- (tmp);
				\draw [->] (tmp) -- (joint controller.south);
			\end{tikzpicture}		
			\caption{Illustrative scheme of a kinematic-controlled robot. The difference between kinematic- and torque-controlled robots lies in how the joint controllers are implemented. In the former, they are strictly built-in by the manufacturer, whereas they can be software implemented or already built-in and proposed by the manufacturer as an additional joint-level control option.}
			\label{fig:position controller}
		\end{figure}
		\begin{figure}[t!]
			\centering
			\subfloat[]{
				\begin{tikzpicture}[auto, node distance=2cm,>=latex']
					\node [input, name=rinput] (rinput) {};
					\node [tmp, right of=rinput, node distance = 0.5 cm](u){};
					\node [block, right of=u, node distance = 0 cm](QP) {QP};
					\node [tmp, right of=QP, node distance = 0.5 cm] (alphaD){};
					\node [block, right of=alphaD, node distance = 2 cm](integrator) {Double Integrator};
					\node [tmp, right of=integrator, node distance = 2.2 cm] (xd){};
					\node [block, right of=xd, node distance = 1 cm](robot) {Robot};
					\node [tmp, right of=robot, node distance = 1 cm](x){};
					\node [tmp, right of=x, node distance = 0.7 cm](x2){};
					\node [tmp, below of = QP, node distance = 0.55cm] (aux){};
					\draw [->] (QP) -- node{$\genDesConfDDot$}(integrator);
					\draw [->] (integrator) -- node{$
						\left[\desConf \ \desConfDot\right]
						$}(robot);
					\draw [->] (robot) -- node{$
						\left[\actConf \  \actConfDot\right]
						$}(x2);
					\draw [-] (x) |- (aux);
					\draw [->] (aux) -- node{}(QP.south);	
				\end{tikzpicture}
				\label{subfig:feedback QP}}
			\hfil
			\subfloat[]{
				\begin{tikzpicture}[auto, node distance=2cm,>=latex']
					\node [input, name=rinput] (rinput) {};
					\node [tmp, right of=rinput, node distance = 0.5 cm](u){};
					\node [block, right of=u, node distance = 0 cm](QP) {QP};
					\node [tmp, right of=QP, node distance = 0.5 cm] (alphaD){};
					\node [block, right of=alphaD, node distance = 2 cm](integrator) {Double Integrator};
					\node [tmp, right of=integrator, node distance = 2.2 cm] (xd){};
					\node [block, right of=xd, node distance = 1 cm](robot) {Robot};
					\node [tmp, right of=robot, node distance = 1 cm](x) {};
					\node [tmp, right of=x, node distance = 0.7 cm](x2){};
					\node [tmp, below of = QP, node distance = 0.55cm] (aux){};
					\draw [->] (QP) -- node{$\genDesConfDDot$}(integrator);
					\draw [->] (integrator) -- node{$
						\left[\desConf \  \desConfDot\right]
						$}(robot);
					\draw [-] (xd) |- (aux);
					\draw [->] (aux) -- node{}(QP.south);
					\draw [->] (robot) -- node{$
						\left[\actConf \  \actConfDot\right]
						$}(x2);
				\end{tikzpicture}
				\label{subfig:feedforward QP}}
			\hfil
			\subfloat[]{
				\begin{tikzpicture}[auto, node distance=2cm,>=latex']
					\node [block, draw = blue, thick](QP) {QP};
					\node [tmp, right of=QP, node distance = 0.5 cm] (alphaD){};
					\node [block, right of=alphaD, node distance = 2 cm, draw = blue, thick](integrator) {Double Integrator};
					\node [tmp, right of=integrator, node distance = 2.2 cm] (xd){};
					\node [block, right of=xd, node distance = 1 cm, draw = blue, thick](robot) {Robot};
					\node [tmp, right of=robot, node distance = 1 cm](x){};
					\node [tmp, right of=x, node distance = 0.7 cm](x2){};
					\node [tmp, below of = QP, node distance = 0.7cm] (aux){};
					\node [tmp, below of = QP, node distance = 0.55cm] (aux2){};
					\draw [->, color=blue, thick] (QP) -- node{$\genDesConfDDot$}(integrator);
					\draw [->, color=blue, thick] (integrator) -- node{$
						\left[\desConf \  \desConfDot\right]
						$}(robot);
					\draw [->, color=blue, thick] (robot) -- node{$
						\left[\actConf \  \actConfDot\right]
						$}(x2);
					\draw [-, color=blue, thick] (x) |- ([xshift=-5pt]aux);
					\draw [-, color=blue, thick] (xd) |- ([xshift=5pt]aux2);
					\draw [->, color=blue, thick] ([xshift=-5pt]aux) -- node{}([xshift=-5pt]QP.south);
					\draw [->, color=blue, thick] ([xshift=5pt]aux2) --([xshift=5pt]QP.south);
				\end{tikzpicture}
				\label{subfig:our QP}}		
			\caption{Different closed-loop QP control schemes for kinematic-controlled robots. The `Double integrator' and `Robot' blocks are detailed in~\cref{fig:leaky integrator QP} and~\cref{fig:position controller}, respectively.
				~\subref{subfig:feedback QP} Feedback QP.
				~\subref{subfig:feedforward QP} Feedforward QP.
				~\subref{subfig:our QP} Proposed robust QP. A detailed overview is shown in~\cref{fig:whole control scheme}.}
			\label{fig:QP scheme for kinematic-control robots}
		\end{figure}
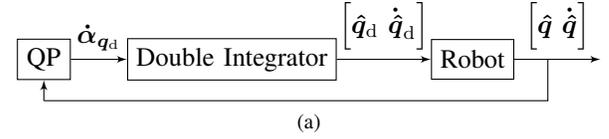
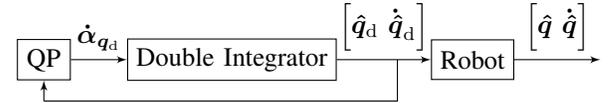
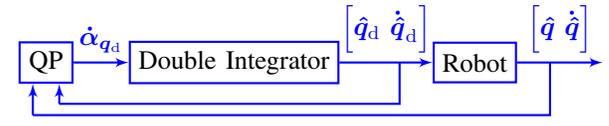
		
		In this paper, we address the stability of closed-loop task-space QP control (\cref{sec:QP formulation}) in the context of kinematic-controlled robots. 
		We show in~\cref{subsec:control challenge} how the closed-loop instability is interpreted in terms of non-robustness of the QP feedback scheme (\cref{fig:QP scheme for kinematic-control robots}\subref{subfig:feedback QP}) against non-modeled joint-dynamics\footnote{The parameters of the joint-dynamics (joint controllers + actuators) are generally not known as they depend on the joint controller gains (fixed by the manufacturer not intended to be modified by the operator in almost all robots~\cite{kim2016humanoids}) and the actuators electro-mechanical constants.}, flexibilities, external disturbances, etc. We propose a robust feedback QP control formulation (\cref{fig:QP scheme for kinematic-control robots}\subref{subfig:our QP}) based on high-level integral feedback terms that robustify the task-space PD controller (task robust stability) and the constraint formulation (set robust stability), see~\cref{fig:whole control scheme}. We extend its implementation to weighted-prioritized multi-objective QP control in~\cref{sec:Robust QP}.
		We assess our controller with experiments on two robots: (i)~a fixed-base robot Panda performing highly-dynamic motion control, and (ii)~floating-based humanoid robot HRP-4 performing robust balance control under non-modeled flexibilities and external disturbances (\cref{sec:Experiments}).
		
		The stability analysis, based on Lyapunov theory, focuses on intrinsic closed-loop QP control (i.e., not considering other sources or nature of instability such as singularities for which we dedicated a specific analysis in~\cite{pfeiffer2023tcst} or discretization). Our solution is elegant as it applies directly to the existing QP-templated tasks. In contrast to~\cite{singletary2022csl} where the joint controllers model is required, our approach is straightforwardly used with any kinematic-controlled robot as it does not require the exact knowledge of the joint-dynamics model which we only assume to be Input-to-State-Stable (ISS). We also further advanced~\cite{bouyarmane2018tac} in two ways: (i) we account for the lack of joint-dynamics in closed-loop;  (ii) we include the constraints in the stability analysis.  Although we define the set robust stability similarly to~\cite{kolathaya2019csl}, our approach is different in three main points: (i) the factors against which the robustness is enforced, (ii) the  formalism adopted to prove robustness, and (iii) the nature of the term added to achieve robustness. 
		
		To sum-up, our contributions are as follows:  
		\begin{itemize}
			\item Robust task formulation;
			\item Robust constraint formulation; 
			\item Task stability investigations in the case of multi-objective weighted-prioritized robust constrained QP;
			\item Integration into our existing framework and validation on a robotic manipulator and a humanoid robot.
		\end{itemize}
		The notations and definitions used in this paper are described in \cref{app:notations}.	
		\begin{figure*}
			\centering
			\begin{tikzpicture}[auto, node distance=2cm,>=latex']
				\node [block](QP) {QP~\eqref{eq:robust QP for combination}};
				\node [tmp, left of=QP, node distance = 4 cm] (tmpLeftQP){};
				\node [block, above of = tmpLeftQP, node distance = 0.35 cm] (Task law) {Robust Task Stabilization~\eqref{eq:heterogeneous feedback mu}};
				\node [block, below of = tmpLeftQP, node distance = 0.35 cm] (RECBF) {RECBF~\eqref{eq:RECBF formulation}};
				\node [tmp, right of=QP, node distance = 1 cm] (alphaD){};
				\node [block, right of=alphaD, node distance = 2.1 cm](integrator) {Double Integrator~\eqref{eq:double integrator}};
				\node [tmp, right of= integrator, node distance = 2.2 cm] (tmpXd){};
				\node [tmp, above of= tmpXd, node distance = 0.5 cm] (tmpFB){};
				\node [tmp, below of= tmpXd, node distance = 1 cm] (tmpxd){};
				\node [tmp, right of=integrator, node distance = 2.2 cm] (xd){};
				\node [block, right of=xd, node distance = 2 cm](robot) {Robot~\eqref{eq:joint-dynamics}};
				\node [above of= robot, node distance=0.8cm](taul-in){$\jointDisturbIn$};
				\node [input, above of = robot, node distance = 1cm] (tmpDisturb) {};
				\node [tmp, right of=robot, node distance = 1 cm](x){};
				\node [tmp, right of=x, node distance = 0.7 cm](x2){};		
				\node [block, below of = QP, node distance = 1.7 cm, text width=3.3cm] (forward kin and vel task) {Forward Kinematics \&  Forward Velocity~\eqref{eq:act output vector}};
				\node [block, below of = forward kin and vel task, node distance = 0.9 cm, text width=3.3cm, thick, draw = blue] (forward vel task) {Forward Velocity~\eqref{eq:des output vector}};
				\node [dotted_block, fit = (forward kin and vel task) (forward vel task)] (task) {};		
				\node at (task.north) [above, inner sep=1.0mm] {Task State};
				
				\node [block, below of = forward vel task, node distance = 2.1 cm, text width=3.3cm] (forward kin and vel constr) {Forward Kinematics \&  Forward Velocity~\eqref{eq:act Bfunc state}};
				\node [block, below of = forward kin and vel constr, node distance = 0.9 cm, text width=3.3cm, thick, draw = blue] (forward vel constr) {Forward Velocity~\eqref{eq:des Bfunc state}};
				\node [dotted_block, fit = (forward kin and vel constr) (forward vel constr)] (constr) {};		
				\node at (constr.north) [above, inner sep=1.0mm] {Constraint State};
				
				\node [block, left of = forward kin and vel task, node distance = 3.5 cm] (K task) {$\taskGains$};
				\node [block, left of = forward vel task, node distance = 3.5 cm, thick, draw = blue] (Lv task) {$\taskIntegralDamping$};
				\node [block, left of = forward kin and vel constr, node distance = 3.5 cm] (K constr) {$\constraintGains$};
				\node [block, left of = forward vel constr, node distance = 3.5 cm, thick, draw = blue] (Lv constr) {$\constraintIntegralDamping$};
				
				\node [tmp, left of = Task law, node distance = 3cm] (tmpLeftTask) {};
				\node [tmp, left of = RECBF, node distance = 3cm] (tmpLeftConstr) {};
				
				\draw [line width=2pt](5.3,-.5)--(5.3,.5);
				\draw [->] ([yshift=0.2cm] tmpXd) -| node[above,yshift = 0.1cm, pos=0.85]  {$\desFBDyn=\begin{bmatrix}
						\desFB \\ \desFBDot
					\end{bmatrix}$}([xshift = 0.3 cm]tmpFB);
				\draw [-] ([yshift=-0.2cm] tmpXd) -- ([yshift=-0.2cm, xshift=0.5cm] tmpXd);
				\draw [->] ([yshift=-0.2cm, xshift=0.5cm] tmpXd) |- node[near end] {$\desJointDyn$} (robot.west);
				\node [tmp, below of = QP, node distance = 0.5cm] (aux){};
				\draw [->] (QP) -- node{$\UgenDesJointDyn$}(integrator);
				\draw [-|] (integrator) -- node{$\genDesJointDyn$}(tmpXd);
				\draw [->] (robot) -- node[pos=0.7]{$\genActJointDyn$}(x2);
				\draw [->] (forward kin and vel task) -- node[pos=0.7, above]{$\actTaskOut$}(K task);
				\draw [->, thick, color=blue] (forward vel task) -- node[pos=0.7, above]{$\desOutDot$}(Lv task);
				\draw [->] (forward kin and vel constr) -- node[pos=0.7, above]{$\actBfuncOut$}(K constr);
				\draw [->, thick, color=blue] (forward vel constr) -- node[pos=0.7, above]{$\desBfuncDot$}(Lv constr);
				
				\draw [->] ([xshift=-0.5cm]x2) |- (forward kin and vel task.east);
				\draw [->] ([xshift=-0.5cm]x2) |- (forward kin and vel constr.east);
				\draw [->, thick, color=blue] ([xshift=-0.25cm]tmpXd) |- (forward vel task.east);
				\draw [->, thick, color=blue] ([xshift=-0.25cm]tmpXd) |- (forward vel constr.east);
				
				\draw [-] (K task) -| node[pos=0.15, above]{$\taskGains\actTaskOut$}([yshift=3pt]tmpLeftTask);
				\draw [-, thick, color=blue] (Lv task) -| node[pos=0.15, above]{$\taskIntegralDamping\desOutDot$}([yshift=-3pt,xshift = 7 pt]tmpLeftTask);
				\draw [-] (K constr) -| node[pos=0.15, above]{$\constraintGains\actBfuncOut$}([yshift=3pt,xshift = 20 pt]tmpLeftConstr);
				\draw [-, thick, color=blue] (Lv constr) -| node[pos=0.15, above]{$\constraintIntegralDamping\desBfuncDot$}([yshift=-3pt,xshift = 27 pt]tmpLeftConstr);
				
				\draw [->] ([yshift=3pt]tmpLeftTask) -- ([yshift=3pt]Task law.west);
				\draw [->, thick, color=blue] ([yshift=-3pt,xshift = 7 pt]tmpLeftTask) -- ([yshift=-3pt]Task law.west);
				\draw [->] ([yshift=3pt,xshift = 20 pt]tmpLeftConstr) -- ([yshift=3pt]RECBF.west);
				\draw [->, thick, color=blue] ([yshift=-3pt,xshift = 27 pt]tmpLeftConstr) -- ([yshift=-3pt]RECBF.west);
				\draw [-] (Task law.east) -- ([xshift = 10pt]Task law.east);
				\draw [-] (RECBF.east) -- ([xshift = 44pt]RECBF.east);
				\draw [->] ([xshift = 10pt]Task law.east) |- node[pos=0.7, above]{$\taskCrtlIn$}([yshift = 3pt]QP.west);
				\draw [->] ([xshift = 44pt]RECBF.east) |- node[pos=0.7, below]{$\bfuncCrtlIn$}([yshift = -3pt]QP.west);
				\draw [->] (taul-in) -- (robot.north);
			\end{tikzpicture}
			\caption{Overview of the proposed robust QP control scheme. The blue thick paths show the high-level integral feedback. Note that if $\taskIntegralDamping=\mathbf{0}$ and $\constraintIntegralDamping=0$, correspond to the feedback control scheme in~\cref{subfig:feedback QP}.}
			\label{fig:whole control scheme}
		\end{figure*}

		\section{Task-Space QP Control Formulation}\label{sec:QP formulation}
		\subsection{Joint-Dynamics}\label{subsec:in-out joint dynamics}
		Consider a robot with a floating or a fixed base having $n\in\mathbb{N}$ actuated Degrees-of-Freedom (DoF). Its state is defined by
			\begin{align}\label{eq:gen act conf formal form}
				\begin{split}
					\genActConf\tp &= \begin{bmatrix}
						\FB\tp & \actConf\tp
					\end{bmatrix}\in\mathbb{R}^{{7+n}} \\ 
					\genActConfDot\tp &= \begin{bmatrix}
						\FBDot\tp & \actConfDot\tp
					\end{bmatrix}\in\mathbb{R}^{6+n} \\ 
					\genActConfDDot\tp &= \begin{bmatrix}
						\FBDDot\tp & \actConfDDot\tp
					\end{bmatrix}\in\mathbb{R}^{6+n}
				\end{split}.
			\end{align}
		${\FB}\tp = \begin{bmatrix}\bm{p}\tp &\quaternion\tp\end{bmatrix}\in\mathbb{R}^{7}$ is the floating base position $\bm{p}\inR^3$ and orientation parameterized with unit quaternion $\quaternion\in S^3$ (3-sphere), $\actConf\in\mathbb{R}^n$ is the joint position. $\FBDot\in\mathbb{R}^6$ (resp. $\FBDDot\in\mathbb{R}^6$) is the linear and angular floating-base velocities (resp. floating-base accelerations), and $\actConfDot\in\mathbb{R}^n$ (resp. $\actConfDDot\in\mathbb{R}^n$) is the joint velocity (resp. joint acceleration).  $\genDesConf$, $\genDesConfDot$ and $\genDesConfDDot$ denote the desired  $\genActConf$, $\genActConfDot$ and $\genActConfDDot$ in~\eqref{eq:gen act conf formal form}, respectively. 

		The robot is governed by the equation of motion 
		\begin{equation}\label{eq:equation of motion}
			\massMat(\genActConf)\UgenDesJointDyn + \nnLinTorqueVec(\genActConf,\genActConfDot) - {\actJacContact}\tp  \force = \selectMat_{\Tau} \Tau,
		\end{equation}
		where $\massMat(\genActConf)\!\in\!\mathbb{R}^{(6+n)\times(6+n)}$ is the inertia matrix, $\nnLinTorqueVec(\genActConf,\genActConfDot)\!\in\!\mathbb{R}^{6+n}$ gathers Coriolis-centrifugal and gravitational torques,  $\actJacContact\!\in\!\mathbb{R}^{3\times(6+n)}$ is the contact Jacobian and $\force\inR^3$ is the contact force. $\selectMat_{\Tau}\tp=\begin{bmatrix}
			\mathbf{0}_{n\times 6} & \mathbf{I}_{n} 
		\end{bmatrix}\!\in\!\mathbb{R}^{n\times (6+n)}$ is the actuation selection matrix, and $\Tau\inR^n$ is the joint torque such that 
		\begin{equation}\label{eq:torque constraint}
			\Tau_{\min}\leq\Tau\leq\Tau_{\max},
		\end{equation} where $\Tau_{\min},\Tau_{\max}\!\in\!\mathbb{R}^n$ are the joint torque bounds.
		Constraints \eqref{eq:equation of motion} and \eqref{eq:torque constraint} can be combined resulting in the  torque-bounded equations of motion 
		\begin{equation}\label{eq:dynamics constraint}
			\selectMat_{\Tau} \Tau_{\min}\leq \massMat(\genActConf)\UgenDesJointDyn + \nnLinTorqueVec(\genActConf,\genActConfDot) - {\actJacContact}\tp  \force \leq \selectMat_{\Tau} \Tau_{\max}.
		\end{equation}
		The contact force is constrained to its linearized friction cone 
		\begin{equation}\label{eq:contact forces}
			\force = \sum_{j=1}^{n_c} \beta_j\bm{\rho}_j, \ \beta_j \geq0 , \ j=1,\ldots,n_c,
		\end{equation} where $\bm{\rho}_j\in\mathbb{R}^3$ is the $j^{\text{th}}$ vector of the linearized friction cone, and $n_c>2$ being the total number of the linearized friction-cones' vectors, see~\cite{bouyarmane2011iros}. 
		
		Let $\genActJointDyn,\genDesJointDyn\in\mathbb{R}^{13+2n}$ defined as 
		\begin{equation}\label{eq:gen joint dyn def}
			\genActJointDyn\tp = \begin{bmatrix}
				\genActConf\tp & \genActConfDot\tp
			\end{bmatrix}, \	\genDesJointDyn\tp = \begin{bmatrix}
				\genDesConf\tp & \genDesConfDot\tp
			\end{bmatrix},
		\end{equation}
		be the actual and desired robot states, respectively. In particular, $\genDesJointDyn$ follows a double integrator dynamics (in the continuous time-domain)
		\begin{align}\label{eq:double integrator}
			&\genDesJointDynDot = 
			\begin{bmatrix}
				\genDesConfDot \\ \genDesConfDDot
			\end{bmatrix} =
			\begin{bmatrix}
				\mathbf{0} & \mathbf{I}_{6+n} \\
				\mathbf{0} & \mathbf{0}
			\end{bmatrix}
			\genDesJointDyn + 
			\begin{bmatrix}
				\mathbf{0} \\
				\mathbf{I}_{6+n}
			\end{bmatrix} \UgenDesJointDyn,
		\end{align}
		with  $\UgenDesJointDyn\!\in\! {\cal U}$ where ${\cal U}\!\subseteq\!\mathbb{R}^{6+n}$ is the set of $\UgenDesJointDyn$ admissible values. The actual floating-base state ${\actFBDyn}\tp \!=\!\begin{bmatrix}
			{\FB}\tp & {\FBDot}\tp
		\end{bmatrix}\!\in\! \mathbb{R}^{13}$ is assumed to be bounded and estimated by an observer. 
		In particular, let us define the robot-state tracking error $\jointTrackErr$ as
		\begin{align}\label{eq:joint tracking error}
			\begin{split}
				\jointTrackErr &= \genActJointDyn - \genDesJointDyn, \\
				&\overset{\text{\small def}}{=} \begin{bmatrix}
					\FB \ominus \desFB \\
					\actConf  - \desConf \\ 
					\FBDot - \desFBDot \\
					\actConfDot - \desConfDot
				\end{bmatrix}\inR^{13+2n} ,
			\end{split}	
		\end{align}
		where $\FB \ominus \desFB \overset{\text{\small def}}{=} \begin{bmatrix}
			\bm{p} - \bm{p}_{\rm d} \\ 
			\quaternion \otimes \quaternion_{\rm d}^{\text{-}1}
		\end{bmatrix}\inR^7$ encompasses the position and orientation errors between $\FB$ and $\desFB$ where $\otimes$ denotes the quaternion product~\cite[Section~2.6]{siciliano2010robotics}.
		
		Given a kinematic-controlled robot (\cref{fig:position controller}), the actual robot state is governed by the dynamics of its actuated joints. By extracting the actuated joints parts from $\genActJointDyn$ and $\genDesJointDyn$ in~\eqref{eq:gen joint dyn def}, let $\desJointDyn,\actJointDyn\in\mathbb{R}^{2n}$ defined as
		\begin{equation}\label{eq:actuated DoF states}
			\actJointDyn\tp = \begin{bmatrix}
				\actConf\tp & \actConfDot\tp
			\end{bmatrix}, \	\desJointDyn\tp = \begin{bmatrix}
				\desConf\tp & \desConfDot\tp
			\end{bmatrix},
		\end{equation} be the actual and desired states of the robot actuated DoF.
		Similarly, we define the joint-dynamics tracking error $\jointDynTrackErr$ as 
		\begin{equation}\label{eq:jointDyn tracking error}
				\jointDynTrackErr = \actJointDyn - \desJointDyn\inR^{2n},
		\end{equation}
		and which dynamics\footnote{For the sake of generality, both $\desConf$ or $\desConfDot$ (joint commands) are encompassed by $\desJointDyn$ in the joint-dynamics $\bm{f}_{\jointDynTrackErr}$ in~\eqref{eq:joint-dynamics}.} is 
		\begin{equation}\label{eq:joint-dynamics}
			\jointDynTrackErrDot = \bm{f}_{\jointDynTrackErr}(\jointDynTrackErr,\jointDisturbIn),
		\end{equation}   
		where  $\jointDisturbIn\in\mathbb{R}^n$ is the bounded joint-space disturbance torque input.  
		In this study, we consider that $\bm{f}_{\jointDynTrackErr}$ is not known exactly but its main property is given by the following assumption.
		\begin{assumption}\label{assum1}
			The joint-dynamics $\bm{f}_{\jointDynTrackErr}$ in~\eqref{eq:joint-dynamics} is ISS w.r.t  $\jointDisturbIn$. Namely, there exist a class ${\cal KL}$  function $\beta$ and a class ${\cal K}$  function $\gamma$, such that for any initial state $\jointDynTrackErr(0)$ and any bounded disturbance input $\jointDisturbIn(t)$, the solution $\jointDynTrackErr(t)$ exists $\forall t\geq 0$ and satisfies~\cite{sontag2008springer,dashkovskiy2011springer}
			\begin{equation}\label{eq:ISS joint-dynamics}
				\norm{\jointDynTrackErr(t)} \leq \beta\left(\norm{\jointDynTrackErr(0)},t\right) + \gamma\left(\norm{\jointDisturbIn}_\infty\right)
			\end{equation}
			
		\end{assumption}
		The IS-Stability ensures that, given  bounded disturbance inputs $\jointDisturbIn$,  the joint-dynamics tracking error $\jointDynTrackErr$ in~\eqref{eq:jointDyn tracking error} evolves in a bounded set containing the origin.	
		\cref{assum1} is largely valid as it is among the main requirement for the well-functioning of a kinematic-controlled robot.
		$\jointTrackErr$ in \eqref{eq:joint tracking error} reflects the effect of several kinds of disturbances and uncertainties. Namely, non-modeled dynamics (e.g., transient joint-dynamics response w.r.t $\desJointDyn$, flexibilities, etc.); hardware imperfections (e.g., joint-dynamics steady-state errors, etc.); external disturbance $\jointDisturbIn\neq0$ (e.g., loads, pushes, unexpected impacts, etc.); measurement and  estimation noises (joint-velocity and floating-base estimations, etc.); and possibly others\footnote{A benchmark problem has been proposed in~\cite{moberg2009tcst} to simulate such disturbances.}.
		\begin{remark}\label{rem1}
			For a fixed-base robot, $\genActConf=\actConf$ and $\genActConfDot= \actConfDot$ leading to $\genActJointDyn = \actJointDyn$ (respectively for the corresponding desired states), and $\jointTrackErr=\jointDynTrackErr$.
		\end{remark}
		
		\subsection{Input-Output Task Dynamics}\label{subsec:in-out task dynamics}
		Let  $\fkin\!:\mathbb{R}^{7+n} \!\rightarrow\!\mathbb{R}^m$ be the forward kinematics for a given task defined by $m$ coordinates, and $\fkinRef(t),\fkinRefDot(t),\fkinRefDDot(t)\!\in\!\mathbb{R}^m$ be the task references; we can define the following states
		\begin{align}
			\label{eq:des output vector}
			\desTaskOut(\genDesJointDyn) &=
			\begin{bmatrix}
				\desOut \\
				\desOutDot
			\end{bmatrix}=
			\begin{bmatrix}
				\fkin(\genDesConf) - \fkinRef(t) \\
				\desJac\genDesConfDot - \fkinRefDot(t) 
			\end{bmatrix} \in H \subset \mathbb{R}^{2m}, 
		\end{align}
		\begin{align}
			\label{eq:act output vector}
			\actTaskOut(\genActJointDyn) &= 
			\begin{bmatrix}
				\actOut \\
				\actOutDot
			\end{bmatrix} =
			\begin{bmatrix}
				\fkin(\genActConf) - \fkinRef(t) \\
				\actJac\genActConfDot - \fkinRefDot(t)
			\end{bmatrix} \in H \subset\mathbb{R}^{2m},
		\end{align}
		where $\desJac,\actJac\!\in\!\mathbb{R}^{m\times(6+n)}$ are the task Jacobians computed w.r.t $\genDesConf$ and $\genActConf$, respectively; $\desTaskOut(\genDesJointDyn)$ and $\actTaskOut(\genActJointDyn)$ denote the desired and actual task dynamics states, respectively; $H$ is the set of their admissible values.

		Using Taylor expansion and~\eqref{eq:joint tracking error}, a relation between $\actTaskOut$ and $\desTaskOut$ is obtained as
		\begin{align}\label{eq:DL act task}
			\begin{split}
				\actTaskOut(\genActJointDyn) &= \actTaskOut(\genDesJointDyn+\jointTrackErr),\\ 	
				&= \left.\begin{matrix}\actTaskOut(\genActJointDyn)\end{matrix}\right|_{\genActJointDyn=\genDesJointDyn} + \underset{\taskTrackErr}{\underbrace{\left.\begin{matrix} \frac{\partial\actTaskOut(\genActJointDyn)}{\partial\genActJointDyn}\end{matrix}\right|_{\genActJointDyn=\check{\genActJointDyn}}\jointTrackErr}}, \ \check{\genActJointDyn} \!= \genDesJointDyn + \theta\jointTrackErr,\\
				&=\desTaskOut(\genDesJointDyn) +\taskTrackErr,\ \frac{\partial\actTaskOut(\genActJointDyn)}{\partial\genActJointDyn} = 
				\begin{bmatrix}
					\actJac & \bm{0} \\ \frac{\partial\left(\actJac\genActConfDot\right)}{\genActConf} & \actJac
				\end{bmatrix}, 				
			\end{split}
		\end{align}
		 $0\leq\theta\leq1$; $\taskTrackErr\inR^m$ being the Lagrange remainder of the Taylor expansion and  denotes the mapping of $\jointTrackErr$ in task-space. Hereafter, the dependency of $\desTaskOut$ on $\genDesJointDyn$ and $\actTaskOut$ on $\genActJointDyn$ is dropped.
		\begin{remark}\label{rem:phi orientation consideration}
			In the multiplication $\left.\begin{matrix} \frac{\partial\actTaskOut(\genActJointDyn)}{\partial\genActJointDyn}\end{matrix}\right|_{\genActJointDyn=\check{\genActJointDyn}}\jointTrackErr$, only the vector part of $\quaternion \otimes\quaternion_{\rm d}^{\text{-}1}$ is considered in $\jointTrackErr$~\eqref{eq:joint tracking error}.  
		\end{remark}
		
		Given that $\desOut$ in~\eqref{eq:des output vector} has a relative degree of 2
		\begin{equation}\label{eq:task relative degree}
			\desOutDDot=  \desJacDot\genDesConfDot + \desJac \UgenDesJointDyn -\fkinRefDDot(t),
		\end{equation} 
		then, the input-output task dynamics is obtained such that
		\begin{align}\label{eq:des task dyn}
			\desTaskOutDot &= \AdesTaskOut\desTaskOut + \BdesTaskOut\taskCrtlIn, \\ 
				\AdesTaskOut &= \begin{bmatrix}
					\mathbf{0} & \mathbf{I}_m \\\mathbf{0}& \mathbf{0}
				\end{bmatrix}, \BdesTaskOut = \begin{bmatrix}
					\mathbf{0} \\ \mathbf{I}_m
				\end{bmatrix}, \\\label{eq:mapping mu to u}
				\taskCrtlIn&= \desJacDot\genDesConfDot + \desJac \UgenDesJointDyn -\fkinRefDDot(t).
		\end{align}
		$\taskCrtlIn\in\mathbb{R}^m$ is the task-space control input affine in $\UgenDesJointDyn$~\eqref{eq:mapping mu to u}. 
		The control objective consists in formulating a task-space controller $\taskCrtlIn$ that steers $\actTaskOut$ to the origin. 

		\subsection{Constraint Formulation with Barrier Functions}
		The general form of a constraint is expressed as
		\begin{equation}\label{eq:def on safey-constraint}
			\mathrm{dist}(\genActJointDyn) \geq \mathrm{dist}_{\min}, 
		\end{equation} where $\mathrm{dist}(\genActJointDyn)\!\in\!\mathbb{R}$ is a distance obtained by forward kinematics and defined in the space of interest, see~\cref{fig:HRP tasks and constraints}, and $\mathrm{dist}_{\min}\!\in\!\mathbb{R}$ is the threshold.
		Let us consider the set $\setC=\left\{\genActJointDyn\in \mathbb{R}^{13+2n}:~\eqref{eq:def on safey-constraint}\right\}$. The fulfillment of the constraint~\eqref{eq:def on safey-constraint} forward in time can be checked by verifying the \emph{forward invariance} of $\setC$~\cite{blanchini1999automatica} (see \cref{app:notations}). Barrier functions are a suitable tool for this purpose. In fact, considering the barrier function $h(\genActJointDyn) = \mathrm{dist}(\genActJointDyn) - \mathrm{dist}_{\min}\geq0$, $\setC$ is forward invariant if $h(\genActJointDyn)$ satisfies Lyapunov-like conditions (see \cite[Definition~3]{ames2017tac}) based on the system's dynamics. 
		Rather than verifying the fulfillment of \eqref{eq:def on safey-constraint}, one may be interested in finding the control input $\UgenDesJointDyn$ that enforces \eqref{eq:def on safey-constraint} forward in time. Exponential Control Barrier Function (ECBF) allows to formulate a constraint on the control input that enforces the asymptotic stability of $\setC$  and thereby its forward invariance. This constraint can be then accounted for in QP\footnote{Conversely to $\bfunc(\genActJointDyn)\geq0$ which does not depend on the decision variables.}
		For an extensive survey about barrier functions\footnote{In the literature, there are reciprocal and zeroing barrier functions. They are equivalent to characterize forward invariance~\cite{ames2017tac}. Albeit, zeroing barrier functions are convenient for robustness study~\cite{xu2015ifac}.}, see~\cite{ames2019ecc}. 
		In \cref{subsec:Robust constraint formulation}, we introduce the set robust stability, then we propose a robust formulation of ECBF that enforces this notion. First, we present in this section the basics of ECBF formulation to follow the same notations afterward. 
		\subsubsection{Exponential Control Barrier Function}\label{subsubsec:ECBF}
		Let us define the sets $\setC,\setCd\subset\mathbb{R}^{13+2n}$ as  
		\begin{align}
			\label{eq:act C}\setC &=\left\{\genActJointDyn\in \mathbb{R}^{13+2n}:\bfunc(\genActJointDyn)\geq0\right\},\\
			\label{eq:des C}\setCd &=\left\{\genDesJointDyn\in \mathbb{R}^{13+2n}:\bfunc(\genDesJointDyn)\geq0\right\}.
		\end{align}
		Let us define the following states 
		\begin{align}
			\label{eq:des Bfunc state}
			\desBfuncOut(\genDesJointDyn) &=
			\begin{bmatrix}
				\desBfunc \\
				\desBfuncDot
			\end{bmatrix}=
			\begin{bmatrix}
				\desBfunc \\
				\desJacBfunc\genDesConfDot
			\end{bmatrix}\in {{_{\mathrm{b}}H}}\subset\mathbb{R}^{2}, \\
			\label{eq:act Bfunc state}
			\actBfuncOut(\genActJointDyn) &= 
			\begin{bmatrix}
				\bfunc \\
				\bfuncDot
			\end{bmatrix}=
			\begin{bmatrix}
				\bfunc \\
				\actJacBfunc\genActConfDot
			\end{bmatrix} \in {{_{\mathrm{b}}H}}\subset\mathbb{R}^{2},
		\end{align}
		where $\desJacBfunc,\actJacBfunc\!\in\!\mathbb{R}^{1\times(6+n)}$ are the barrier function Jacobians computed w.r.t $\genDesConf$ and $\genActConf$, respectively;  $\desBfuncOut(\genDesJointDyn)$ and $\actBfuncOut(\genActJointDyn)$ denote the desired and actual constraint dynamics states, respectively; ${{_{\mathrm{b}}H}}$ is the set of their admissible values.
		
		As in~\eqref{eq:DL act task}, we have the following 
		\begin{align}\label{eq:bfunc track err}
			\begin{split}
				\actBfuncOut(\genActJointDyn) &= \desBfuncOut(\genDesJointDyn) + \bfuncTrackErr, \ \bfuncTrackErr = \left.\begin{matrix} \frac{\partial\actBfuncOut(\genActJointDyn)}{\partial\genActJointDyn}\end{matrix}\right|_{\genActJointDyn=\check{\genActJointDyn}}\jointTrackErr \\
				\frac{\partial\actBfuncOut(\genActJointDyn)}{\partial\genActJointDyn}  &= 
				\begin{bmatrix}
					\actJacBfunc & \mathbf{0} \\ \frac{\partial\left(\actJacBfunc\genActConfDot\right)}{\genActConf} & \actJacBfunc
				\end{bmatrix} 
			\end{split},
		\end{align} where $\bfuncTrackErr$ is the mapping of $\jointTrackErr$ in the constraint-space. As in~\eqref{eq:DL act task}, \cref{rem:phi orientation consideration} is considered for $\left.\begin{matrix} \frac{\partial\actBfuncOut(\genActJointDyn)}{\partial\genActJointDyn}\end{matrix}\right|_{\genActJointDyn=\check{\genActJointDyn}}\jointTrackErr$ in~\eqref{eq:bfunc track err}. 
		
		Similarly to~\eqref{eq:task relative degree}, $\desBfunc$ has a relative-degree of 2
		\begin{equation}\label{eq:barrier function relative degree}
			\desBfuncDDot = \desJacBfuncDot\genDesConfDot + \desJacBfunc \UgenDesJointDyn.
		\end{equation}
		Then as in~\eqref{eq:des task dyn}, from~\eqref{eq:barrier function relative degree} we get
		\begin{align}\label{eq:Bfunc transverse dyn}
			\desBfuncOutDot &= \AdesBfuncOut\desBfuncOut + \BdesBfuncOut\bfuncCrtlIn,\\
			\label{eq:desBfunc def}\desBfunc&=\CdesBfuncOut\desBfuncOut,\\
			\AdesBfuncOut &= \begin{bmatrix}
				0 & 1 \\ 0 & 0
			\end{bmatrix}, \BdesBfuncOut = \begin{bmatrix}
				0 \\ 1
			\end{bmatrix},\CdesBfuncOut = \begin{bmatrix}
				1 & 0 \end{bmatrix},  \\\label{eq:mapping Bfunc mu to u}
			\bfuncCrtlIn &= \desJacBfuncDot\genDesConfDot + \desJacBfunc \UgenDesJointDyn.
		\end{align}
		Let $\bfuncCrtlIn \!=\! -\constraintGains\desBfuncOut\in\mathbb{R}$ with $\constraintGains \!=\! \begin{bmatrix}
			\constraintStiffness & \constraintDamping
		\end{bmatrix}\!\in\!\mathbb{R}^{1\times2}$. From~\eqref{eq:Bfunc transverse dyn} and~\eqref{eq:desBfunc def}, $\desBfunc(t) \!=\!  \CdesBfuncOut \exp\left({\FdesBfuncOut t}\right)\desBfuncOut(t_0)$ with $\FdesBfuncOut \!=\! \AdesBfuncOut\!-\!\BdesBfuncOut \constraintGains$. Hence, if 
		\begin{equation}\label{eq:ECBF constraint}
			\bfuncCrtlIn \geq -\constraintGains \desBfuncOut \overset{\eqref{eq:mapping Bfunc mu to u}}{\Longleftrightarrow}  \desJacBfuncDot\genDesConfDot + \desJacBfunc \UgenDesJointDyn \geq -\constraintGains \desBfuncOut,
		\end{equation} then following the Comparison Lemma~\cite[Lemma~3.4]{khalil2002NonLinearSystems}, $\desBfunc(t) \!\geq\!  \CdesBfuncOut \exp\left({\FdesBfuncOut t}\right)\desBfuncOut(t_0)$. If there exists a gain matrix $\constraintGains$ such that $\desBfunc(t) \!\geq\!  \CdesBfuncOut \exp\left({\FdesBfuncOut t}\right)\desBfuncOut\left(t_0\right)\!\geq\!0$ whenever $\desBfunc\left(t_0\right)\!\geq\!0$, then $\desBfunc$ is an ECBF, and $\setCd$ is made forward invariant (also said `safe') (see~\cite[Definition~7]{ames2019ecc}). The gain matrix $\constraintGains$ needs to satisfy two specifications: (i) $\FdesBfuncOut $ eigenvalues must be real-negative, 
		and (ii) $\desBfunc\left(t\right)\!\geq\!0$, $\forall t \!\geq \! t_0$, $\forall \genDesJointDyn(t_0)\in\setCd$~\cite{djeha2020ral,quiroz-omana2019ral}. 
		Note that  ECBF formulation~\eqref{eq:ECBF constraint} corresponds to the feedforward closed-loop QP control scheme in \cref{fig:QP scheme for kinematic-control robots}\subref{subfig:feedforward QP} where the non-modeled dynamics are not accounted for, and thereby~\eqref{eq:ECBF constraint} is called forward ECBF formulation. Unfortunately, it does not  imply forward invariance of the set $\setCd$ when the feedback closed-loop QP control scheme \cref{fig:QP scheme for kinematic-control robots}\subref{subfig:feedback QP} is adopted  as it will be shown in \cref{subsubsec:control challenge for constraint}. Hence, the goal  is to formulate $\bfuncCrtlIn$ such that $\setCd$ is made robustly stable.
		
		
		\subsection{Combining Tasks and Constraints via QP}\label{subsec:QP formulation}
		Since~\eqref{eq:mapping mu to u} and~\eqref{eq:mapping Bfunc mu to u} are affine in $\UgenDesJointDyn$, the task and set $\setC$ stabilization can be formulated and combined by the following weight-prioritized QP (highlighted terms are changed later)
		\begin{subequations}\label{eq:QP for combination}
			\begin{align}\label{eq:QP task}
				\begin{split}
					\left[\UgenDesJointDyn^*,\force^*\right]\!=~&\!\arg\min  \frac{\weight_0}{2}\!\norm{\selectMat\UgenDesJointDyn\!+ \! \boldsymbol{\kappa}(\actJointDyn)}^2 +\\	&\frac{\weight}{2}\!\norm{\desJac \UgenDesJointDyn\! + \!\desJacDot \genDesConfDot \! -\!\fkinRefDDot(t) \! \highlight{-\taskCrtlIn}}^2 
				\end{split}	\\ \label{eq:QP constraint}
				\text{s.t.} & \;\;\;\eqref{eq:dynamics constraint},\eqref{eq:contact forces}\\\label{eq:safety constraint}
				&-\desJacBfunc \UgenDesJointDyn\leq \desJacBfuncDot\genDesConfDot   \highlight{-\bfuncCrtlIn} \\ \label{eq:contact constraint}
				&\;\;\;\; \desJacContact\UgenDesJointDyn =- \desJacContactDot\genDesConfDot   -k\desJacContact\genDesConfDot
			\end{align}
		\end{subequations}	
		where $w$ and $w_0$ are positive weighting scalars such that $w\geq0$ and $w_0>0$~\cite[Lemma~2]{bouyarmane2018tac}. The first term in~\eqref{eq:QP task} is a secondary task that solves the remaining redundancy where $\selectMat = \left[\mathbf{0}_{n\times6} \quad \mathbf{I}_n\right]$\footnote{For a fixed-base robot, $\selectMat= \mathbf{I}_n$.} is a selection matrix and  $\boldsymbol{\kappa}(\actJointDyn)$ is a given joint-space feedback. 
		Constraint~\eqref{eq:contact constraint} stands for the no-slipping contacts (e.g., at the feet)  to have, along with~\eqref{eq:contact forces}, feasible and dynamically-consistent floating-base solutions $\desFBDDot$ where $\desJacContact$ is the contact Jacobian, and $k>0$.
			
		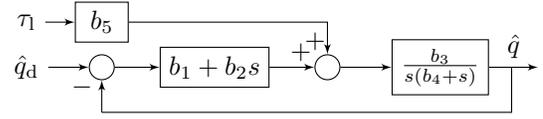
\begin{figure}
			\centering
			\begin{tikzpicture}[auto, node distance=2cm,>=latex']
				\node [input, name=rinput] (rinput) {};
				\node [tmp, right of=rinput, node distance = 0.0 cm](u){};
				\node [right of=rinput, node distance = 0.0 cm](qd-in){$\hat{q}_{\rm d}$};
				\node [above of=rinput, node distance = 0.6 cm](taul-in){$\tau_{\rm l}$};
				\node [sum, right of = rinput](sum) {};
				\node [block, above of=sum, node distance=0.6cm](load Gain) {$b_5$};
				\node [block, right of=sum, node distance=1.5cm] (PD controller) {$b_1 + b_2s $};
				\node [sum, right of = PD controller, node distance=1.5cm](sum2) {};
				\node [block, right of=sum2, node distance=1.5cm] (DC actuator) {$\frac{b_3}{s(b_4+ s)}$};
				\node [tmp, right of=DC actuator, node distance = 1.3 cm](q){};
				\node [tmp, below of=sum, node distance = 0.6 cm](belowSum){};
				
				\draw [->] (sum) -- node{}(PD controller);
				\draw [->] (qd-in) -- node{}(sum);
				\draw [->] (taul-in) -- node{}(load Gain);
				\draw [->] (PD controller) -- node[pos = 0.7, yshift=-1pt]{$+$}(sum2);
				\draw [->] (sum2) -- node{}(DC actuator);
				\draw [->] (DC actuator) -- node{$\hat{q}$}(q);
				\draw [->] (load Gain) -| node[pos = 0.8, xshift=-12pt]{$+$}(sum2);
				\draw [-] ([xshift = -10pt]q) |- node[pos = 0.8, xshift=-12pt]{}(belowSum);
				\draw [->] (belowSum) -- node[pos = 0.8]{$-$}(sum.south);
			\end{tikzpicture}
			
			\caption{1-DoF joint-dynamics block scheme from which the system \eqref{eq:1DoF example} is derived. It is DC motor servoed in position by a PD joint controller. The parameters $a_{i={1,...,5}}$ in \cref{tab:parameters} are as follows: $a_1 = -b_1b_3$, $a_2 = -b_2b_3-b_4$, $a_3 = - a_1$, $a_4 = b_2b_3$, $a_5 = b_3b_5$. The latter show the coupling between the PD gains ($b_1$ and $b_2$) of the joint controller and the electro-mechanical constants ($b_3$, $b_4$ and $b_5$) of the DC motor. The systems 1 and 2 in \cref{tab:parameters} have the same electro-mechanical constants but different PD gains.} 
			\label{fig:1DoF system block scheme}
		\end{figure}
		
		\section{Instability of Feedback QP Control Scheme}
		\label{subsec:control challenge} 
		
		Now, we show that formulating QP~\eqref{eq:QP for combination} according to the 
		closed-loop control scheme in~\cref{fig:QP scheme for kinematic-control robots}\subref{subfig:feedback QP} may lead to instability.
		
		\subsection{Task Feedback Formulation}\label{subsubsec:control challenge for tasks}
		Let us formulate $\taskCrtlIn$ as an output feedback control
		\begin{align}\label{eq:mu output feedback}
			\begin{split}
				\taskCrtlIn &= -\taskGains\actTaskOut, \ \taskGains = \begin{bmatrix}
					\taskStiffness & \taskDamping
				\end{bmatrix}\in\mathbb{R}^{m\times2m},\\ 
				\eqref{eq:DL act task} \Rightarrow\taskCrtlIn&= -\taskGains\desTaskOut - \taskGains\taskTrackErr.
			\end{split}
		\end{align}
		By replacing~\eqref{eq:mu output feedback} in~\eqref{eq:des task dyn}, we get
		\begin{align}\label{eq:des task dyn around eq}
			\begin{split}
				\desTaskOutDot &={\FdesTaskOut} \desTaskOut -\BdesTaskOut \taskGains\taskTrackErr, \ \FdesTaskOut = \AdesTaskOut - \BdesTaskOut \taskGains,
			\end{split}
		\end{align}
		where $\taskGains$ is chosen such that $\FdesTaskOut$ is Hurwitz, and $\taskGains\taskTrackErr\in\mathbb{R}^m$ is a perturbation term showing the coupling between the task gains and the effect of all the non-modeled dynamics $\taskTrackErr$. By the virtue of Lyapunov's indirect method~\cite[Theorem~4.7]{khalil2002NonLinearSystems}, $\desTaskOut$ is only locally stable for  \emph{sufficiently} bounded $\taskGains\taskTrackErr $. Otherwise, the closed-loop stability is not guaranteed.

		To exemplify this claim, let us consider a 1-DoF system in \cref{fig:1DoF system block scheme} ($n=1$, $\dot{\alpha}_{q_{\rm d}}=\ddot{\hat{q}}_{\rm d}$) governed by the following dynamics 
		\begin{equation}\label{eq:1DoF example}
			\actJointDynDot \!=\! 
			\begin{bmatrix}
				0 & 1 \\
				a_1 & a_2
			\end{bmatrix}\actJointDyn \!+\! 
			\begin{bmatrix}
				0 & 0 \\
				a_3 & a_4
			\end{bmatrix}\desJointDyn + \begin{bmatrix}
				0 \\ a_5
			\end{bmatrix}\tau_{\rm l}, \ \actJointDyn \!=\! \begin{bmatrix}
				\hat{q} \\ \dot{\hat{q}}
			\end{bmatrix}\!\in\!\mathbb{R}^2,
		\end{equation} where $a_{i=\{1,\ldots,5\}}\in\mathbb{R}$ denote the system parameters. The control objective is to drive the motor position $\hat{q}$ to reach a reference position $\hat{q}_{\mathrm{ref}}=1$~rad with no external disturbance ($\tau_{\rm l}=0$). \cref{eq:des output vector,eq:act output vector} become ($s(q) = \hat{q}$)
		\begin{align*}
			\boldsymbol{\eta}(\actJointDyn) = \begin{bmatrix}
				\hat{q} - \hat{q}_{\mathrm{ref}} \\ \dot{\hat{q}}
			\end{bmatrix} , \  \boldsymbol{\eta}_{\rm d}(\desJointDyn) = \begin{bmatrix}
				\hat{q}_{\rm d} - \hat{q}_{\mathrm{ref}} \\ \dot{\hat{q}}_{\rm d}
			\end{bmatrix}, \ 
			\mu = \ddot{\hat{q}}_{\rm d} \in\mathbb{R}.
		\end{align*}
		$a_i$ are taken according to System~1 in \cref{tab:parameters}. Output feedback control~\eqref{eq:mu output feedback} is performed (implemented on system~\eqref{eq:des task dyn}) with two sets of gains: $\mathbf{K}_1=\begin{bmatrix} 10 & 2\sqrt{10}\end{bmatrix}$ and $\mathbf{K}_2=\begin{bmatrix}30 & 2\sqrt{30}\end{bmatrix}$. \Cref{fig:feedback QP 1DoF - OutFdbk Non robust 1DoF}  shows that simply increasing the task gains leads to instability of the closed-loop system. In \cref{subsec:Robust task formulation}, we propose a task feedback to ensure \emph{global robust stability}.
		\begin{table}
			\centering
			\captionsetup{justification=centering, labelsep=newline, textfont=sc}

			\caption{Parameters Used for System~\eqref{eq:1DoF example} Numerical Simulations.}
			\label{tab:parameters}
			\begin{tabular}{|c||c|c|c|}\hline
				& System 1   & System 2 & Units\\\hline
				$a_1$&$-376.5977$ &$-2380.6356$ & s$^{\text{-}2}$ \\\hline
				$a_2$&$-158.5073$ &$-173.5712$ & s$^{\text{-}1}$ \\\hline
				$a_3$&$376.5977$  &$2380.6356$ &  s$^{\text{-}2}$ \\\hline
				$a_4$&$2.8245$    &$17.8884$ &s$^{\text{-}1}$ \\\hline
				$a_5$&$4.7034$    &$4.7034$ & rad / N.m.s$^2$ \\\hline
			\end{tabular}
		\end{table}
		\begin{figure}
			\centering
			\includegraphics[width=0.4939\columnwidth]{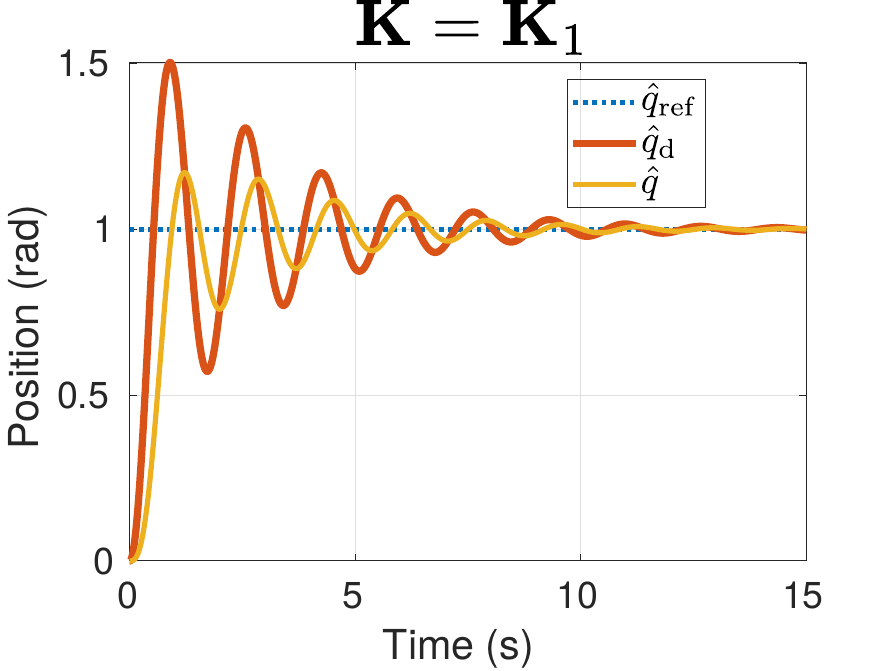}
			\includegraphics[width=0.4939\columnwidth]{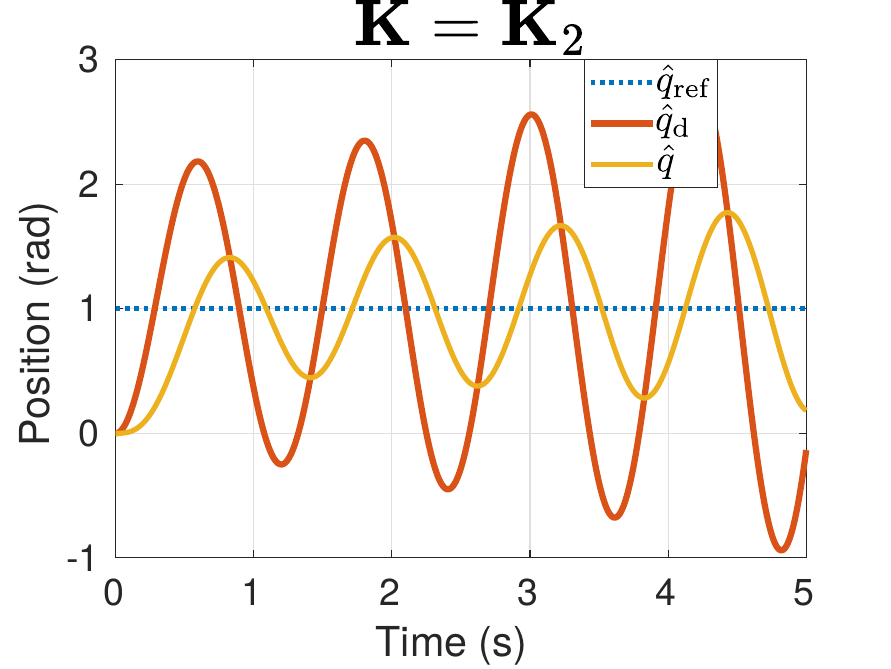}
			\caption{Non-robustness of output feedback control~\eqref{eq:mu output feedback} in system~\eqref{eq:1DoF example} with $\mathbf{K}=\mathbf{K}_1$ (left) and $\mathbf{K}=\mathbf{K}_2$ (right). Increasing the task gains with the same system parameters leads the closed-loop system to instability.}
			\label{fig:feedback QP 1DoF - OutFdbk Non robust 1DoF}
		\end{figure}
		
		\subsection{Feedback ECBF Formulation}\label{subsubsec:control challenge for constraint}

		Let us consider the feedback QP control scheme in~\cref{fig:QP scheme for kinematic-control robots}\subref{subfig:feedback QP}, ECBF constraint~\eqref{eq:ECBF constraint} becomes
		\begin{align}\label{eq:ECBF in feedback}
			\begin{split}
				\bfuncCrtlIn &\geq-\constraintGains\actBfuncOut, \\ 
				\overset{\eqref{eq:bfunc track err}}{\Longleftrightarrow} \bfuncCrtlIn &\geq-\constraintGains\desBfuncOut - \constraintGains\bfuncTrackErr.
			\end{split}
		\end{align}
		As in \cref{subsubsec:control challenge for tasks}, feedback ECBF formulation~\eqref{eq:ECBF in feedback} may not ensure the stability of sets $\setCd$ and $\setC$ if the term $\constraintGains\bfuncTrackErr $ is not sufficiently bounded; which shows the coupling between the ECBF gains and the disturbance $\bfuncTrackErr$. In Fig.~\ref{fig:feedback QP 1DoF - ECBF Non robust 1DoF} we compare feedforward and feedback ECBF formulations where $\bfunc = 3 - \hat{q}$, $\desBfunc = 3 - \hat{q}_{\rm d}$ and a steady-state error $\hat{\phi}\neq0$ is simulated by applying a constant disturbance $\tau_{\rm l}=5$~N.m. $a_{i}$ parameters are taken according to System~2 in \cref{tab:parameters} and $\constraintGains$ is computed as in~\cite{djeha2020ral}. The task is formulated as in \cref{subsubsec:control challenge for tasks} except  $\hat{q}_{\mathrm{ref}}\!=\!5$~rad. Feedforward ECBF~\eqref{eq:ECBF constraint} does not account for the disturbance which leads to a constraint violation. On the other hand, feedback ECBF~\eqref{eq:ECBF in feedback} is not robust to non-modeled joint-dynamics resulting in sustained oscillations around the set boundary $\hat{q}_{\max}=3$~rad. In \cref{subsec:Robust constraint formulation}, we propose a new formulation of $\bfuncCrtlIn$ to guarantee \emph{robust stability} of  $\setCd$.

		\begin{figure}[!t]
			\centering
			\includegraphics[width=0.4939\columnwidth]{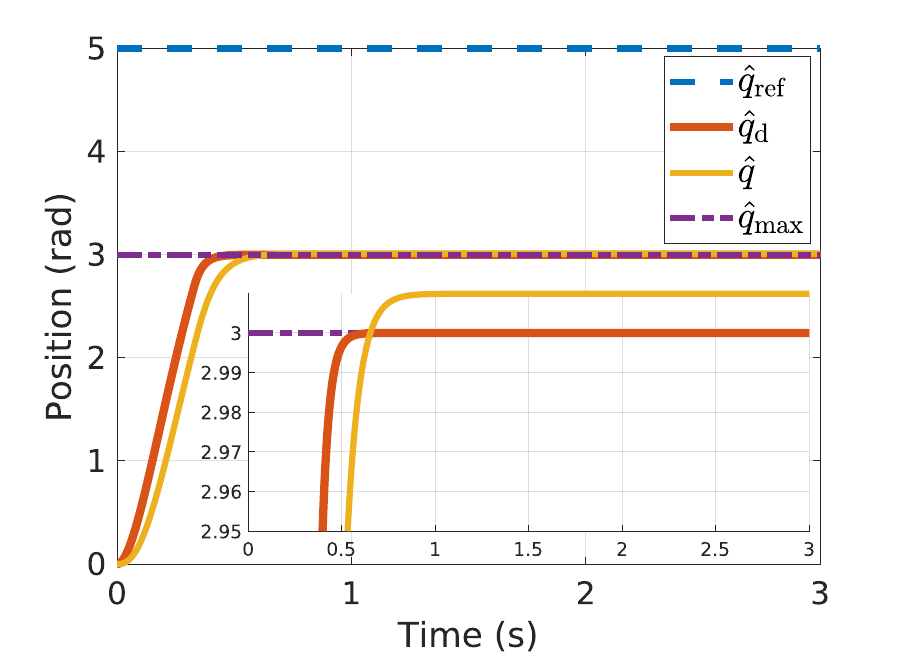}
			\includegraphics[width=0.4939\columnwidth]{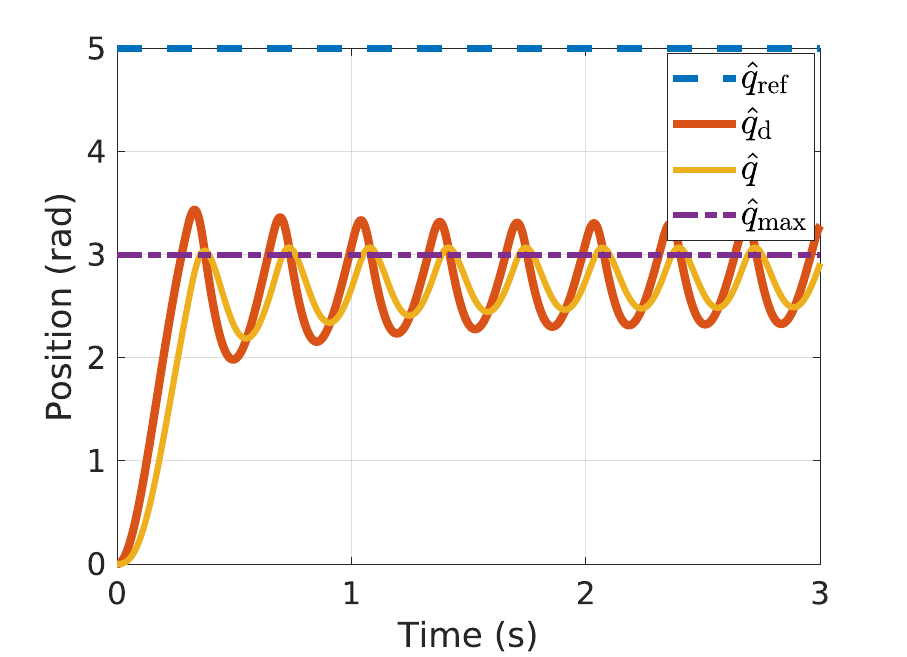}
			\caption{ECBF formulation performed on system~\eqref{eq:1DoF example} in feedforward~\eqref{eq:ECBF constraint}  (left) and feedback~\eqref{eq:ECBF in feedback}  (right) QP control schemes.} 
		\label{fig:feedback QP 1DoF - ECBF Non robust 1DoF}
	\end{figure}
	
	\section{Robust Feedback QP Control Formulation}\label{sec:Robust QP}
	The proposed robust design of the QP controller consists of the following steps. The main result is given by \cref{thm:heterogeneous feedback,thm:RECBF} that show how $\taskCrtlIn$ and $\bfuncCrtlIn$ are formulated including integral feedback terms  to ensure $\desTaskOut $ and  $\desBfuncOut $ convergence to their respective residual sets, and thereby their boundedness. Then, \cref{prop 1} makes the bridge between the desired task-space state $\desTaskOut $ and the corresponding desired robot state $\genDesJointDyn$ by showing that if the former is bounded so is the latter. Finally, based on \cref{assum1} and \cref{prop 1}, \cref{prop 2} establishes the boundedness relationship between $\desTaskOut$ and $\actTaskOut$. 
	\begin{proposition}\label{prop 1}
		If $\desTaskOut $ (resp. $\desBfuncOut $) is bounded, so is $\genDesJointDyn $. 
	\end{proposition}
	\begin{proof}
		See \cref{proof:prop1}.
	\end{proof}
	
	\begin{proposition}\label{prop 2}
		If $\desTaskOut $ (resp. $\desBfuncOut $) is (uniformly) ultimately bounded 
		then $\actTaskOut $ (resp. $\actBfuncOut$) is (uniformly) ultimately bounded.
	\end{proposition}
	\begin{proof}
		See \cref{proof:prop2}.
	\end{proof}
	
	Let us introduce the following states
	\begin{align*}\label{eq:psi}
		\taskPsi &= \begin{bmatrix}
			{\actOut}\tp & 	{\actOutDot}\tp& {\desOutDot}\tp
		\end{bmatrix}\tp, \ \taskPsi\in\Psi\subset\mathbb{R}^{3m}, \\
		\bfuncPsi &= \begin{bmatrix}
			{\bfunc}\tp & 	{\bfuncDot}\tp& {\desBfuncDot}\tp
		\end{bmatrix}\tp, \ \bfuncPsi\in\bfuncPsiSet\subset\mathbb{R}^{3},
	\end{align*} 
	where $\Psi$ and $\bfuncPsiSet$ are the admissible values of $\taskPsi$ and $\bfuncPsi$, respectively; they are used for $\taskCrtlIn$ and $\bfuncCrtlIn$ formulations, resp.
	
	\subsection{Global Robust Stable Task Formulation}\label{subsec:Robust task formulation} 
	When the solutions of~\eqref{eq:des task dyn} converge to a residual set $\Omega\!\subset \! H$ with $0\!\in\!\Omega$ for all initial conditions and admissible perturbations,	$\desTaskOut $ is said to be Robustly Globally Uniformly Asymptotically Stable w.r.t $\Omega$ (RGUAS-$\Omega$). The robust stabilization problem consists in finding  $\taskCrtlIn$ such that $\desTaskOut $  is RGUAS-$\Omega$. 
	If  $\Omega$ can be made arbitrarily small (but still not equal to the origin), $\desTaskOut $ is said to be robustly \emph{practically} stable (see Definition in \cref{app:notations}).
	An illustrative scheme of RGUA-Stability is shown in~\cref{fig:RGUAS}.
	\begin{figure}[!htb]
		\centering
		\includegraphics[width=0.7\columnwidth]{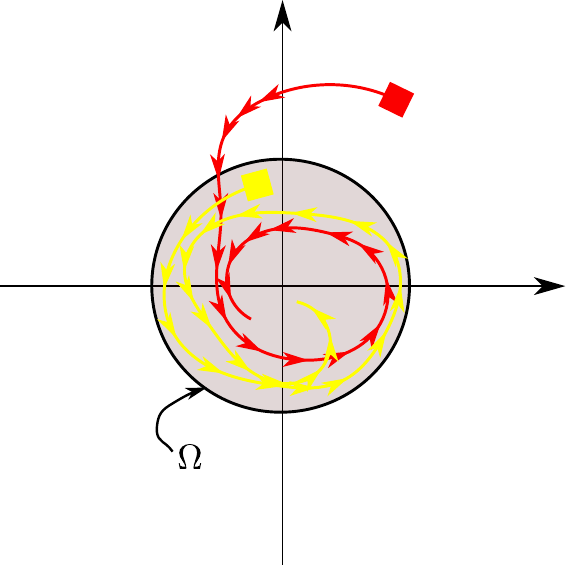}
		\caption{RGUAS-$\Omega$ illustrative scheme. Two state trajectories are shown: the red one starts (squares) outside the residual set $\Omega$ then it converges to $\Omega$ over time, whereas the yellow one starts inside $\Omega$ and remains within it.} 
	\label{fig:RGUAS}
\end{figure}
In the following, we state the main result of this subsection.
\begin{theorem}\label{thm:heterogeneous feedback}
	Let us assume that $\taskTrackErr $ is bounded. If 
	\begin{equation}\label{eq:heterogeneous feedback mu}
		\taskCrtlIn = -\taskGainsPsi\taskPsi, \ \taskGainsPsi \!=\!\begin{bmatrix}
			\taskStiffness& 	\taskDamping& \taskIntegralDamping 
		\end{bmatrix}\in\mathbb{R}^{m\times3m},
	\end{equation}
	where $\taskStiffness,\taskIntegralDamping  \!\in\!\mathbb{R}^{m\times m}$ are diagonal positive-definite matrices, and $\taskDamping$ chosen such that $\FdesTaskOutRobust \!=\! \AdesTaskOut\!-\!\BdesTaskOut \mathbf{\check{K}}$ is Hurwitz, with $\mathbf{\check{K}}=\left[\taskStiffness \quad \taskDamping+\taskIntegralDamping\right]$, 
	then there exists $\taskIntegralDamping$ such that $\desTaskOut $ is robustly practically stable. 
\end{theorem}
\begin{proof}
	See \cref{proof:thm hetero feedback}.
\end{proof}
From \cref{prop 2} and \cref{thm:heterogeneous feedback}, $\actTaskOut $ is uniformly ultimately bounded with ultimate bound~$\tilde{\varrho}$. 
Comparatively to \eqref{eq:mu output feedback}, \eqref{eq:heterogeneous feedback mu} penalizes the integral term $\desOutDot(t) \!=\! \int_{0}^{t}\taskCrtlIn(s)ds$ growth ensuring the desired task state $\desTaskOut$ is bounded despite the presence of disturbances and non-modeled dynamics $\taskTrackErr$. This is the key feature for robust stability. 
The ‘\emph{global}’ property  comes from the fact that stability is ensured for every bounded $\taskGains\taskTrackErr$ whereas in~\eqref{eq:des task dyn around eq} it is required to be ‘\emph{sufficiently}’ bounded. Also, the ‘\emph{practical}’ aspect   denotes the ability of the integral gain (that can be tuned independently from the stiffness and damping) to reduce the effect of the perturbation as shown in~\eqref{eq:cdt on norm - Lv} making the residual set $\Omega_\desTaskOut$  arbitrarily small. 
Nevertheless, \eqref{eq:cdt on norm - Lv} does only prove the existence of a set of integral gains $\taskIntegralDamping$ enforcing robust stability without providing a constructive method to compute them. In practice, we do not know the task stiffness that will turn the closed-loop system instable as it depends on the executed task and the robot non-modeled joint-dynamics. Hence, the necessary amount of integral gain for robust stability is not known \emph{a priori}.

To show the effect of the integral gains on the closed-loop system dynamics, let us take the system in \cref{subsubsec:control challenge for tasks} with the instable configuration (see \cref{fig:feedback QP 1DoF - OutFdbk Non robust 1DoF}) and add the integral feedback term.
Figure~\ref{fig:1-DOF}\subref{subfig:1dof-integral0.01} shows that low integral gain values do not achieve robust stability since the condition~\eqref{eq:cdt on norm - Lv} is not satisfied. If the integral gain is increased, robust stability is recovered, as shown in \cref{fig:1-DOF}\subref{subfig:1dof-integral0.1}, but oscillations still exist at the reference target $\hat{q}_{\rm ref}$. These oscillations can be damped by increasing the integral gain to smoothly reach $\hat{q}_{\rm ref}$ while counterbalancing the external perturbation (\cref{fig:1-DOF}\subref{subfig:1dof-integral1}). Figure~\ref{fig:1-DOF}\subref{subfig:1dof-integral2} shows that further increasing the integral gain overdamps the system response. The convergence smoothness is preserved at the expense of a higher settling time.

Another interesting aspect is that  $\FdesTaskOutRobust$ being Hurwitz  implies that  $\taskDamping$ can be chosen negative-definite as long as   $\taskDamping + \taskIntegralDamping$ is positive-definite. Indeed, let us assume  $\taskDamping \!<\!0$ and $\taskIntegralDamping \! =\! \varepsilon|\taskDamping|$ (element-wise) with $\varepsilon \!=\! 1 \!+\!\varepsilon_0, \varepsilon_0>0$ yielding to 
\begin{align}\label{eq:heterogeneous feedback with negative Kv}
	\begin{split}
		\taskCrtlIn &= -\taskStiffness\actOut  -\taskDamping\actOutDot - \taskIntegralDamping\desOutDot, \\
		&= -\taskStiffness\actOut  -\varepsilon_0|\taskDamping|\desOutDot - |\taskDamping|(\desOutDot- \actOutDot),
	\end{split}		
\end{align}
$\varepsilon_0$ is tuned to meet robustness condition~\eqref{eq:cdt on norm - Lv}. \cref{eq:heterogeneous feedback with negative Kv} shows that $\desOutDot$ is enforced to converge to zero while drifting toward $\actOutDot$ due to $(\desOutDot- \actOutDot)$ feedback term.  This is similar to the leaky-integrator proposed in~\cite{hopkins2015icra} when \eqref{eq:heterogeneous feedback with negative Kv} is performed in joint-space (i.e., posture task). Our claim is that performing a \emph{task-space} integral feedback is more intuitive and enables a better understanding of the underlying conditions on task gains tuning while not affecting the QP solution optimality\footnote{In~\cite{hopkins2015icra}, feedback term related to $(\desOutDot- \actOutDot)$ is added to $\genDesConfDDot$ post QP computation, and thereby it may not be feasible. Moreover, only experimental observations have been reported about the effect of the joint-space leaky integrator gain  without any explicit condition on its values.}. It also enables robust stability of each task in multi-objective control case, see~\cref{subsec:Robust QP Lipschitz continuity} (\cref{prop:relaxed heterogeneous feedback}).

The term $(\desOutDot- \actOutDot)$ in \eqref{eq:heterogeneous feedback with negative Kv} induces compliance w.r.t robot response in terms of velocity. Figure~\ref{fig:1-DOF-compliance} shows this behavior. When the external disturbance is applied, the desired task velocity $\dot{\hat{q}}_{\rm d}$ immediately converges to zero in \cref{fig:1-DOF-compliance}\subref{subfig:1dof-integral1-WoCompliance}. Conversely, $\dot{\hat{q}}_{\rm d}$ drifts first toward $\dot{\hat{q}}$ (since the amplitude of $(\dot{\hat{q}}_{\rm d}-\dot{\hat{q}})$ is predominant) then converges back to zero in \cref{fig:1-DOF-compliance}\subref{subfig:1dof-integral1-WCompliance}. In terms of position, the desired task position $\hat{q}_{\rm d}$ slightly drifts toward $\hat{q}$ then counterbalances the external disturbance. This feature is exploited in \cref{subsec:expreiment_robust safety}.

Note that by setting $\taskIntegralDamping \!=\!\mathbf{0}$, the output feedback~\eqref{eq:mu output feedback} is recovered. Namely, the proposed approach does not constitute a substantial modification of the QP controller. 

In the next subsection, we take inspiration from~\eqref{eq:heterogeneous feedback mu} to formulate Robust ECBF (RECBF) to enforce set robust stability. 
\begin{figure}[!htb]
	\centering
	\subfloat[$\varepsilon=0.01$]{
		\includegraphics[width=0.5\columnwidth]{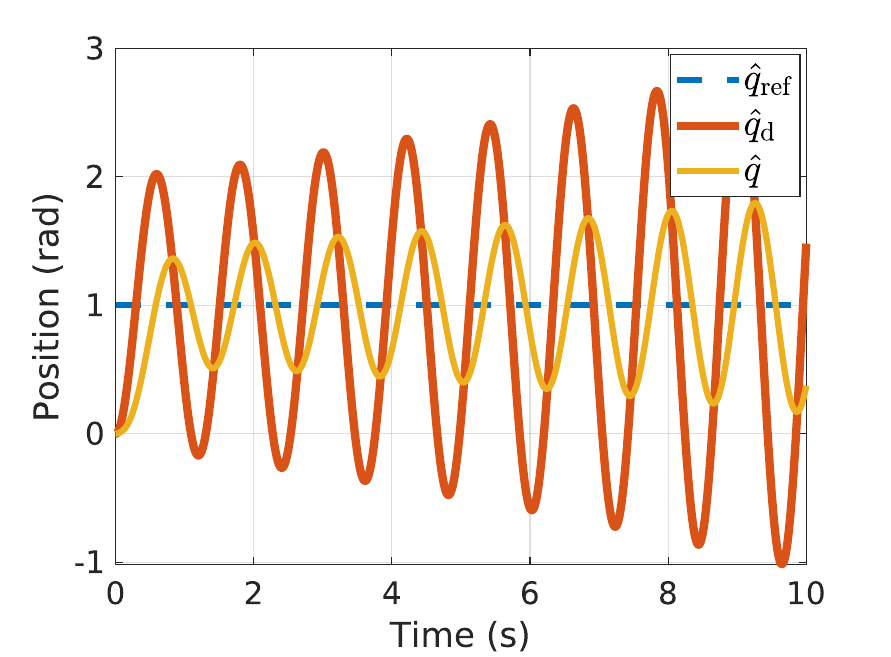}
		\label{subfig:1dof-integral0.01}}
	\subfloat[$\varepsilon=0.1$]{
		\includegraphics[width=0.5\columnwidth]{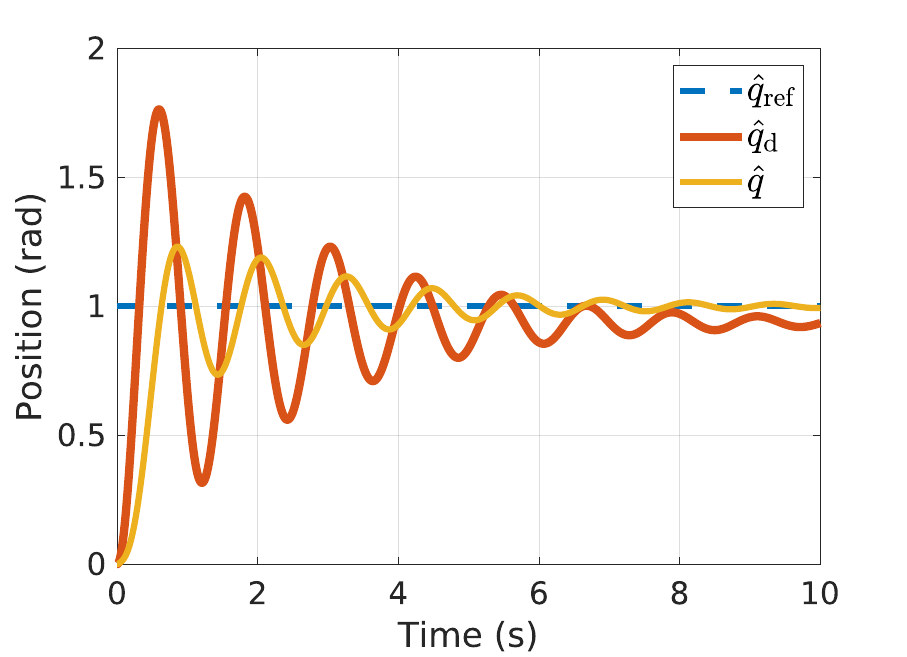}
		\label{subfig:1dof-integral0.1}}
	\hfil
	\subfloat[$\varepsilon=1$]{
		\includegraphics[width=0.5\columnwidth]{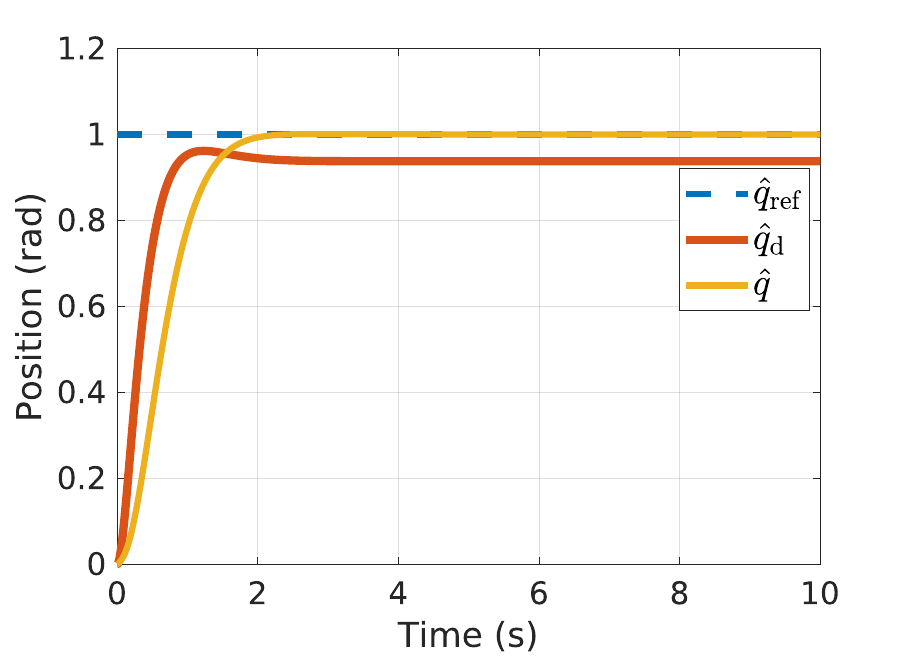}
		\label{subfig:1dof-integral1}}
	\subfloat[$\varepsilon=2$]{
		\includegraphics[width=0.5\columnwidth]{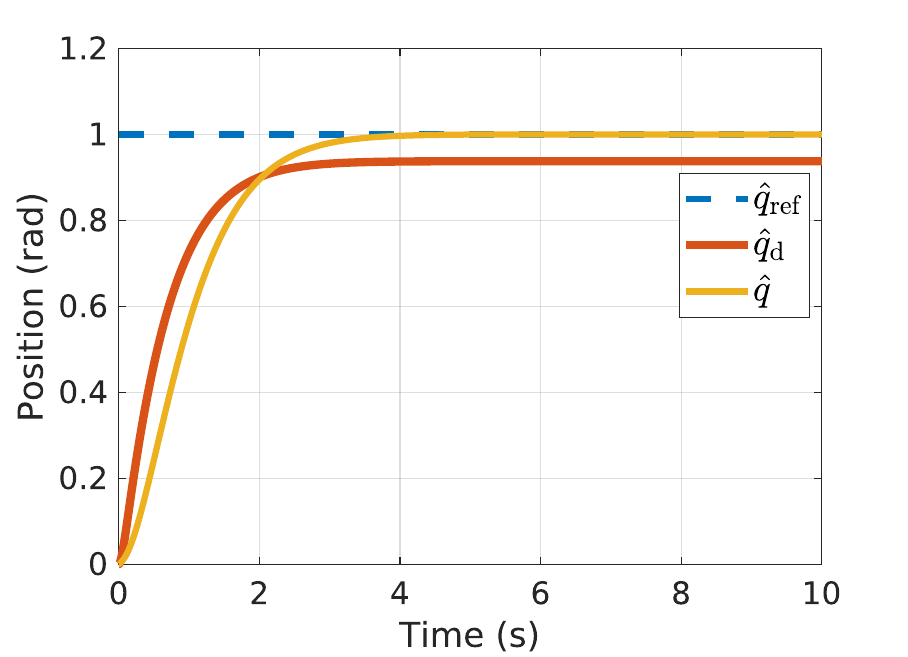}
		\label{subfig:1dof-integral2}}
	\caption{System~\eqref{eq:1DoF example} response, with $a_i$ of System 1 in \cref{tab:parameters} and $\tau_{\rm l} = 5$~N.m, under heterogeneous feedback~\eqref{eq:heterogeneous feedback mu}. The integral gain $K_{\rm i}=\varepsilon K_{\rm d}$. This choice follows the fact that both $K_{\rm i}$ and $K_{\rm d}$ act on velocity terms.} \label{fig:1-DOF}
\end{figure}
\begin{figure}[!htb]
\centering
\subfloat[]{
	\includegraphics[width=0.5\columnwidth]{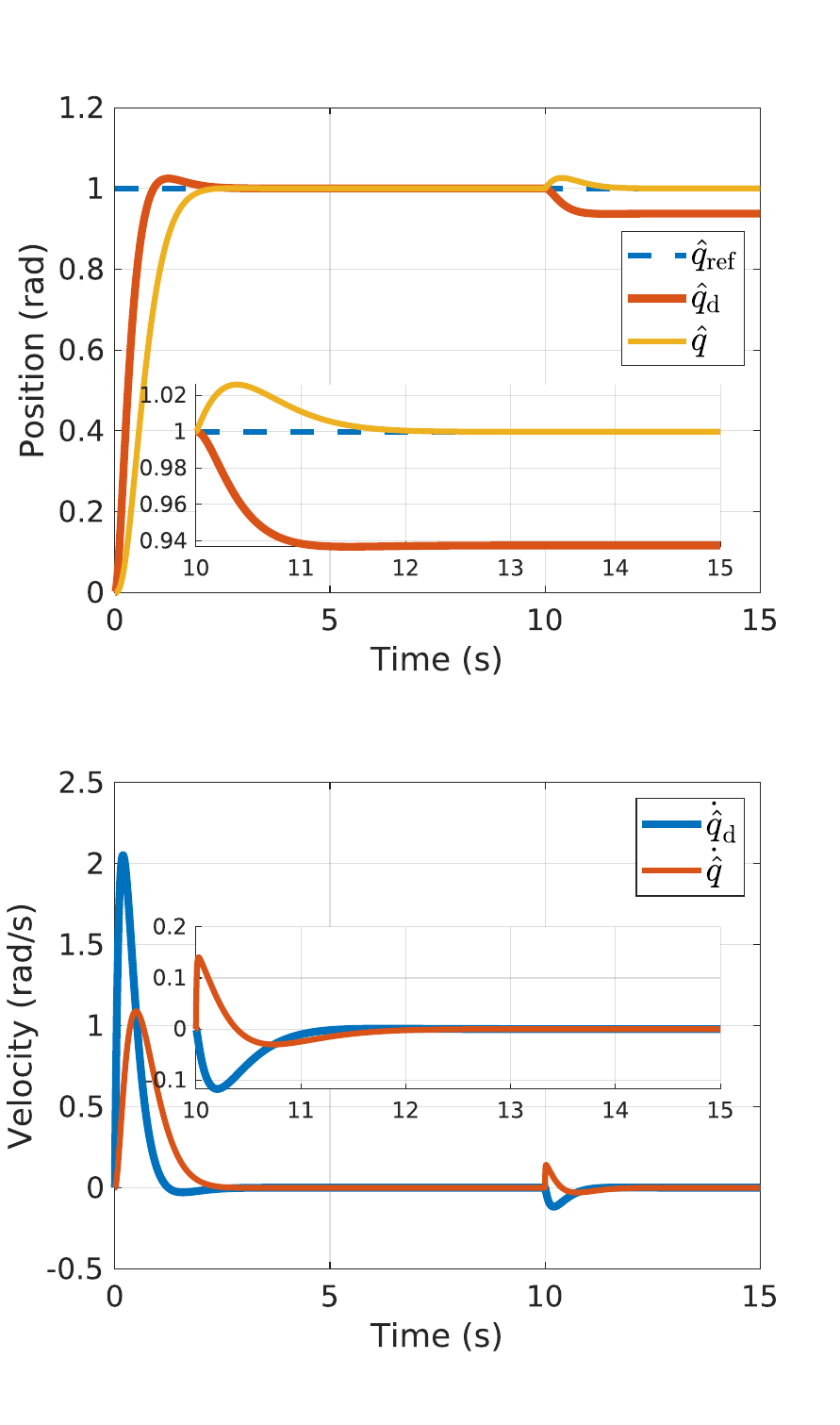}
	\label{subfig:1dof-integral1-WoCompliance}}
\subfloat[]{
	\includegraphics[width=0.5\columnwidth]{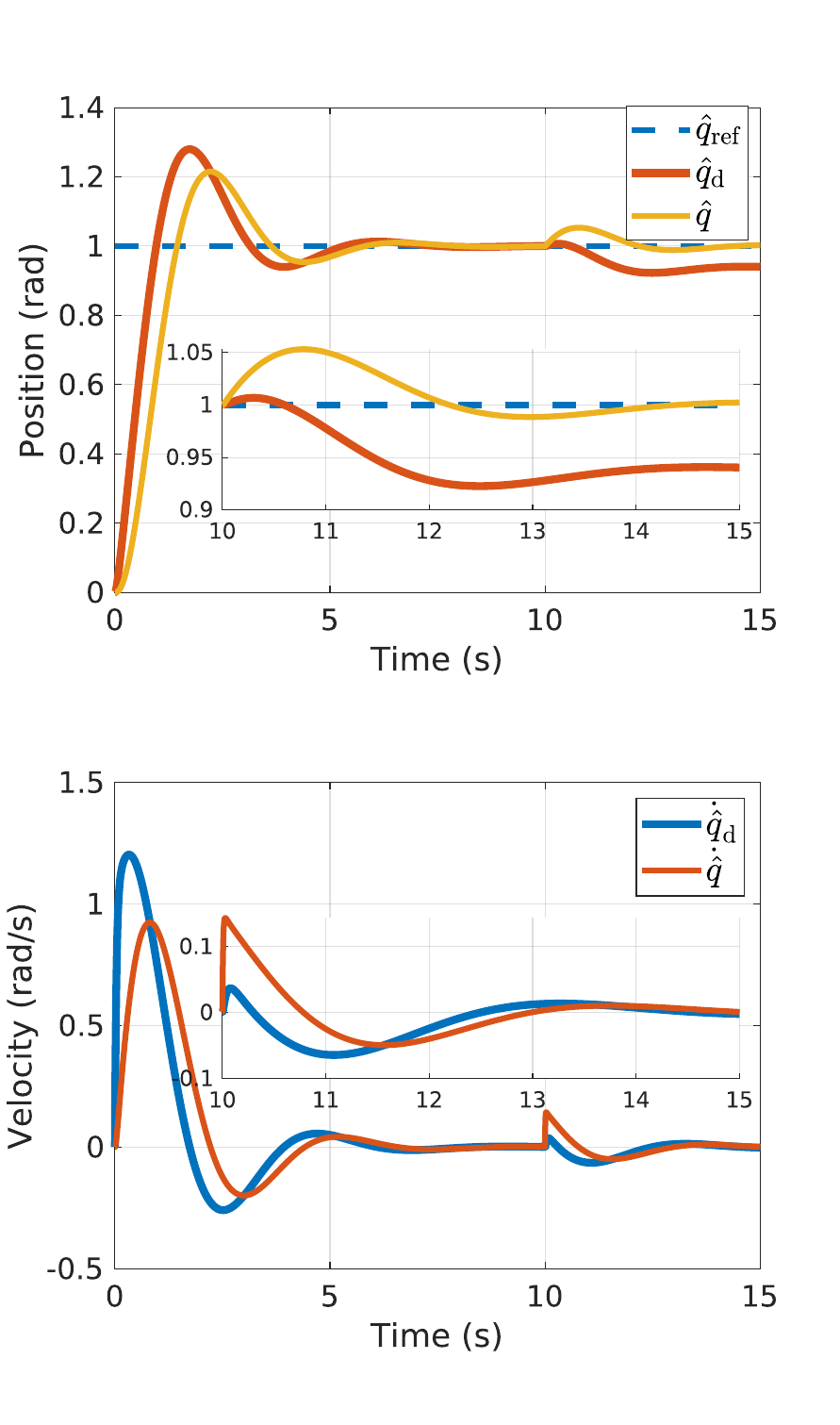}
	\label{subfig:1dof-integral1-WCompliance}}
\caption{System~\eqref{eq:1DoF example} response, with $a_i$ of System 1 in \cref{tab:parameters}, $\tau_{\rm l} = 5$~N.m is applied at $t=10$~s, under heterogeneous feedback~\eqref{eq:heterogeneous feedback mu}. \subref{subfig:1dof-integral1-WoCompliance} $\mu = \ddot{\hat{q}}_{\rm d} = -K_{\rm s}\hat{q} -K_{\rm d}\dot{\hat{q}} - K_{\rm i}\dot{\hat{q}}_{\rm d}$, $K_{\rm s} = 30$, $K_{\rm d}=2\sqrt{K_{\rm s}}$, $K_{\rm i} = K_{\rm d}$. \subref{subfig:1dof-integral1-WCompliance} $\mu = \ddot{\hat{q}}_{\rm d} = -K_{\rm s}\hat{q} - K_{\rm i}\dot{\hat{q}}_{\rm d} - |K_{\rm d}|\left(\dot{\hat{q}}_{\rm d}-\dot{\hat{q}}\right)$, $K_{\rm s} = 30$, $K_{\rm d}=-1.8\sqrt{K_{\rm s}}$, $K_{\rm i} = 3.2\sqrt{K_{\rm s}}$.}
\label{fig:1-DOF-compliance}
\end{figure}

\subsection{Set Robust  Stability Formulation}	\label{subsec:Robust constraint formulation}
\begin{figure}[!htb]
\centering
\includegraphics[width=0.5\columnwidth]{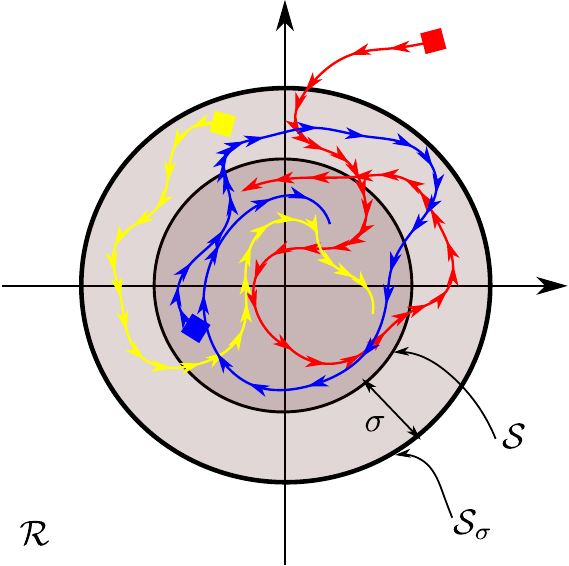}
\caption{The sets $\setS$, $\setS_\sigma$ and $\setR$. $\setR$ is open, $\setS_\sigma$ is asymptotically stable and forward invariant, and $\setS$ is robustly stable. If $\sigma=0$, $\setS$ and $\setS_\sigma$ coincide. The colored trajectories denote three possible cases depending on the initial condition: in $\setR$ (red), in $\setS_\sigma$ (yellow), and in $\setS$ (blue): the red one converges to $\setS_\sigma$ and remains inside because of the set asymptotic stability; the yellow one cannot go out of $\setS_\sigma$ because it is forward invariant; finally the blue one can slightly go out of $\setS$ but remains inside $\setS_\sigma$ (robust stability).}
\label{fig:Sets}
\end{figure}
\begin{figure}[!htb]
\centering
\subfloat[$\varepsilon=0.02$]{
\includegraphics[width=0.5\columnwidth]{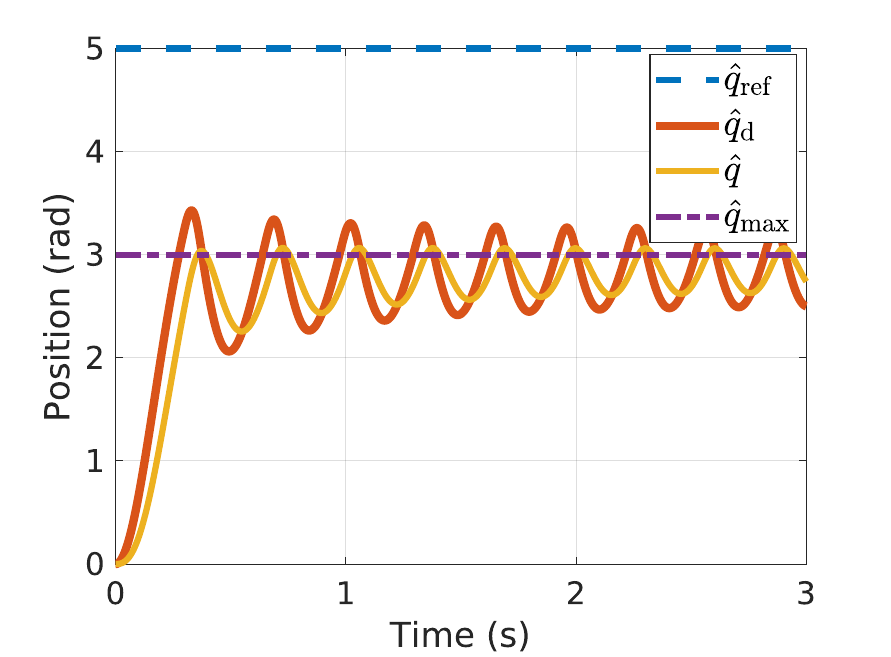}
\label{subfig:1dof-RECBF-integral0.02}}
\subfloat[$\varepsilon=0.2$]{
\includegraphics[width=0.5\columnwidth]{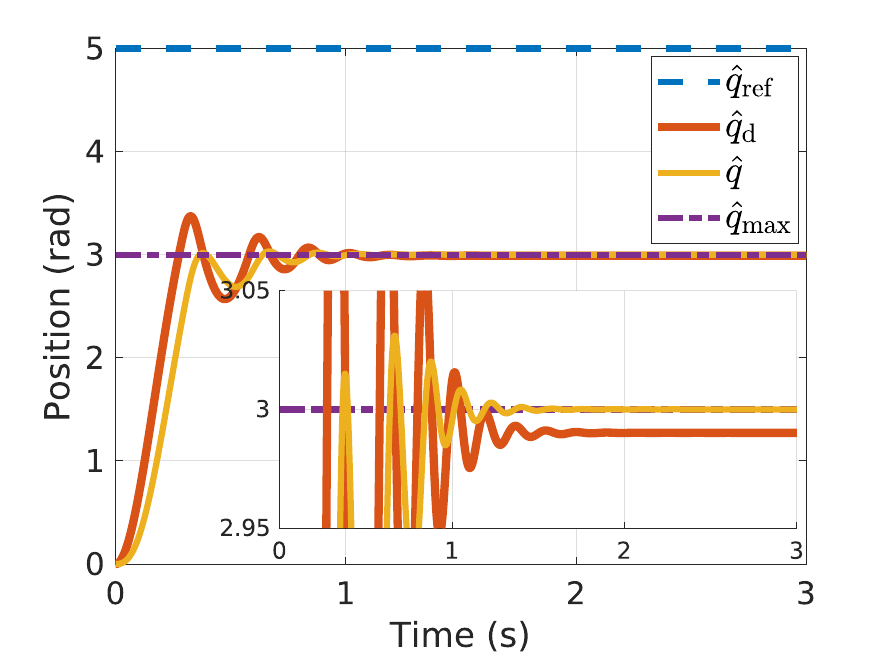}
\label{subfig:1dof-RECBF-integral0.2}}
\hfil
\subfloat[$\varepsilon=2$]{
\includegraphics[width=0.5\columnwidth]{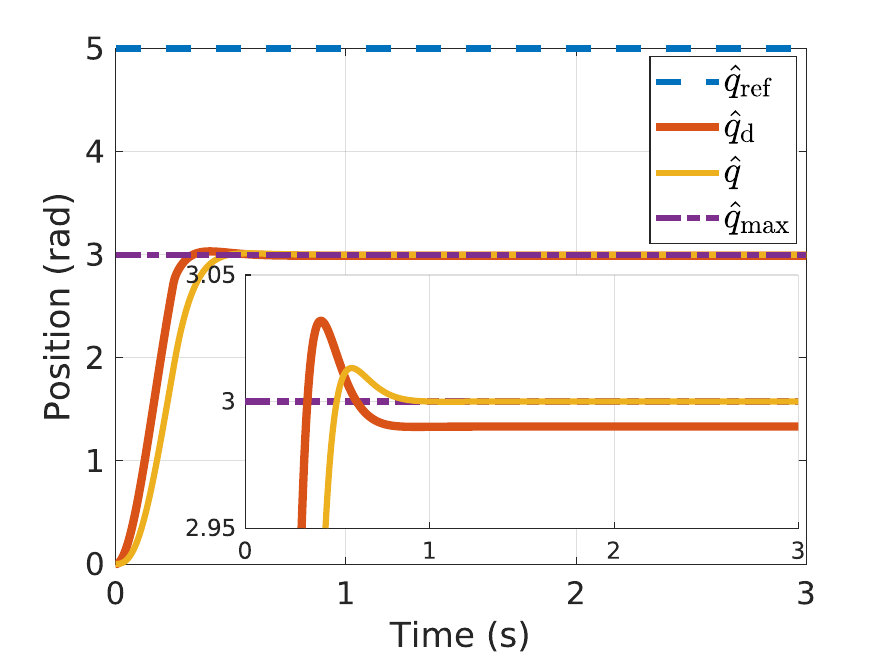}
\label{subfig:1dof-RECBF-integral2}}
\subfloat[$\varepsilon=5$]{
\includegraphics[width=0.5\columnwidth]{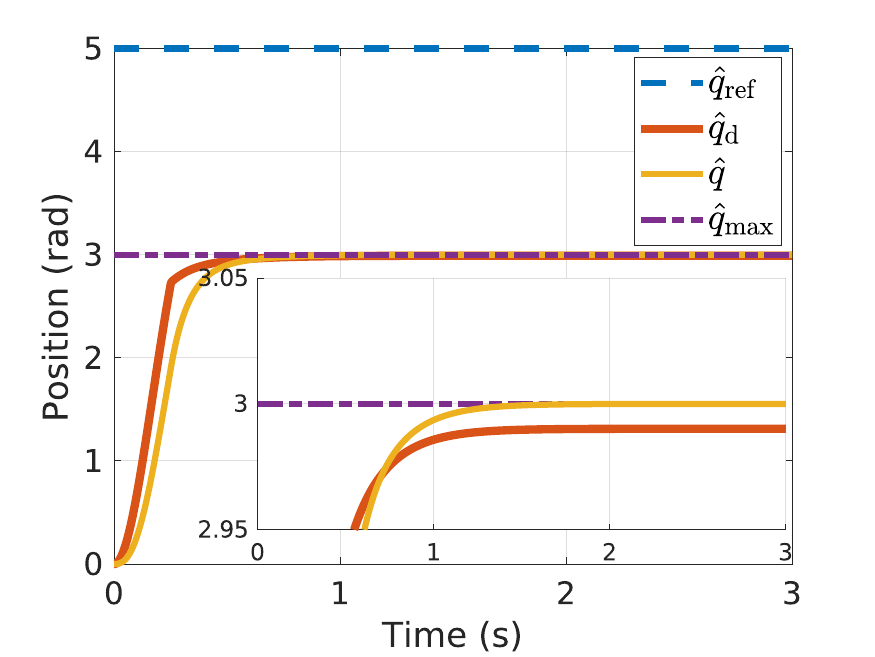}
\label{subfig:1dof-RECBF-integral5}}
\caption{System~\eqref{eq:1DoF example} response with $a_i$ parameters of System~2 in \cref{tab:parameters} with $\tau_{\rm l}=5$ N.m, under RECBF~\eqref{eq:RECBF formulation} with $\constraintIntegralDamping=\varepsilon\constraintDamping$.  }
\label{fig:RECBF}
\end{figure}

Let us define the sets $\setC_\sigma,\setCd_{\sigma}\subset\mathbb{R}^{13+2n}$ with $\sigma\geq0$ 
\begin{align}
\label{eq:act C_sigma}\setC_\sigma &=\left\{\genActJointDyn\in \mathbb{R}^{13+2n}:\bfunc+ \sigma\geq0\right\}, \\
\label{eq:des C_sigma}\setCd_{\sigma} &=\left\{\genDesJointDyn\in \mathbb{R}^{13+2n}:\desBfunc+ \sigma\geq0\right\}.
\end{align}

Before introducing the main result of this subsection, we define set robust stability and RECBF.  The former is illustrated in \cref{fig:Sets}.  
\begin{definition}\label{def:RS-sigma}
A closed set $\setS\subset\mathbb{R}^{13+2n}$  is said to be  \emph{robustly stable} for a forward complete system~\eqref{eq:double integrator} if $\exists \; \sigma\geq0$, a closed and forward invariant set $\setS_\sigma\!\subset\!\mathbb{R}^{13+2n}$, and an open $\setR\!\subseteq\!\mathbb{R}^{13+2n}$ with $\setS\!\subseteq\!\setS_\sigma\!\subset\!\setR$ such that $\setS_\sigma$ is asymptotically stable.  
\end{definition}
Let us now define RECBF to enforce set robust  stability.
\begin{definition}\label{def:RECBF}
Given a set $\setCd\subset\mathbb{R}^{13+2n}$ 
defined as the superlevel set of a 2-times continuously differentiable function $\desBfunc:\mathbb{R}^{13+2n}\rightarrow\mathbb{R}$, then $\desBfunc$ is a RECBF if there exists $\bfuncCrtlIn\in\mathbb{R}$ such that for~\eqref{eq:double integrator}, 
\begin{equation}
\label{eq:RECBF constraint}
\underset{\UgenDesJointDyn\in {\cal U}}{\sup} \left[\desJacBfuncDot\genDesConfDot + \desJacBfunc \UgenDesJointDyn + \bfuncCrtlIn\right]\geq0,
\end{equation} 
results in $\setCd$ is robustly stable.
\end{definition}
\begin{figure}
\centering
\includegraphics[width=0.7\columnwidth]{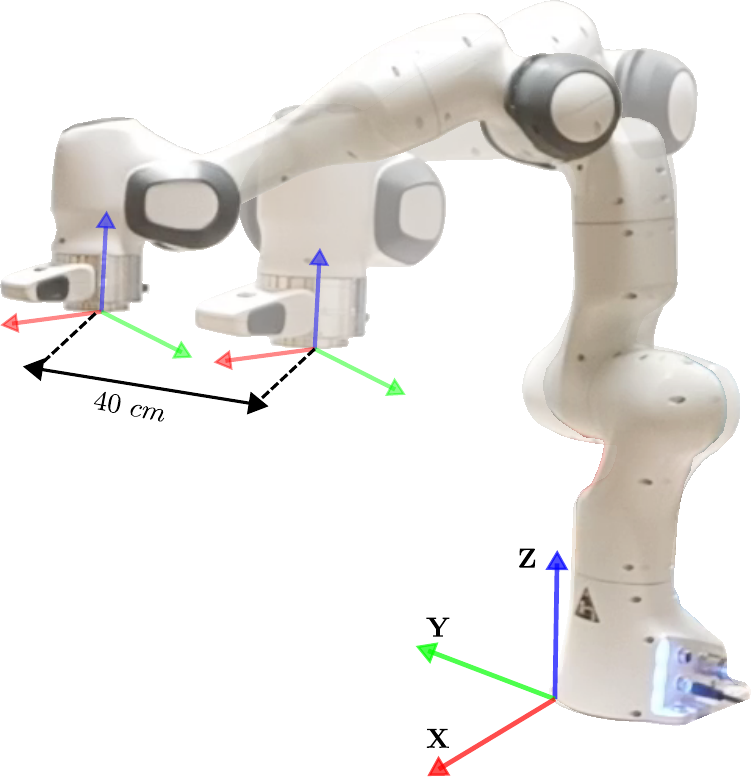}
\caption{Two superposed snapshots showing Panda end-effector converging to the two defined set-point pose targets.}
\label{fig:panda_snap}
\end{figure}
Inspiring from \cref{thm:heterogeneous feedback}, \cref{thm:RECBF} proposes a formulation for $\bfuncCrtlIn$ that guarantees the robust stability of $\setCd$. 
\begin{theorem}\label{thm:RECBF} Let us assume $\bfuncTrackErr $ bounded. If
\begin{equation}\label{eq:RECBF formulation}
\bfuncCrtlIn\geq-\constraintGainsPsi\bfuncPsi, \ \constraintGainsPsi=\begin{bmatrix}
\constraintStiffness & \constraintDamping & \constraintIntegralDamping
\end{bmatrix},
\end{equation}
where $\constraintIntegralDamping>0$ and  $\check{\mathbf{K}}^{\bfunc} = \begin{bmatrix}
\constraintStiffness & \constraintDamping + \constraintIntegralDamping
\end{bmatrix}\in\mathbb{R}^{1\times2}$ is chosen according to  ECBF definition in \cref{subsubsec:ECBF}, then there exists $\constraintIntegralDamping$ such that $\desBfunc$ in \eqref{eq:desBfunc def} is a RECBF, and the set $\setCd$ is robustly stable. 
\end{theorem}
\begin{proof}
See \cref{proof:RECBF}.
\end{proof}
By virtue of \cref{prop 2} and following from \cref{thm:RECBF}, $\actBfuncOut$ is uniformly ultimately bounded with ultimate bound $\tilde{\sigma}$ and thereby it converges to asymptotically stable set $\setC_\sigma$ defined as~\eqref{eq:act C_sigma}. Hence, the set $\setC$ is rendered robustly stable as well.
\begin{remark}
Conceptually, \cref{thm:RECBF} proposes the same solution as \cref{thm:heterogeneous feedback} and that is why the robustness conditions \eqref{eq:cdt on norm - Lv} and~\eqref{eq:cdt on Lv RECBF} are similar. This is because \cref{thm:RECBF} does only perceive the `origin' slightly different from \cref{thm:heterogeneous feedback}: the former considers it as a `set' whereas the latter perceives it as a `point'. This fact can also be seen from \cref{fig:RGUAS,fig:Sets}: shrinking the set $\setS$ to the origin in the latter directly brings us to the former. Hence, the task gains properties discussed for~\eqref{eq:heterogeneous feedback mu} apply as well to the RECBF constraint gains in~\eqref{eq:RECBF formulation}.
\end{remark}


As in \cref{fig:1-DOF}, \cref{fig:RECBF} shows the conservativeness of tuning the RECBF integral gain $\constraintIntegralDamping$, where the same scenario in \cref{subsubsec:control challenge for constraint} is performed. Low $\constraintIntegralDamping$ values do not ensure robust stability. Increasing $\constraintIntegralDamping$ leads to meet the robustness condition~\eqref{eq:cdt on Lv RECBF}, and thereby damp the oscillations at the set boundary. Further increasing $\constraintIntegralDamping$  helps to remove the oscillations but at the expense of slower convergence to the boundary. The latter requires high deceleration amount especially that RECBF constraint~\eqref{eq:RECBF formulation} is only inserted in QP at the neighborhood of $\setC$ boundary\footnote{The number of rows of the constraint matrix is increased in consequence. This is very usual/common in our controller which is based on a Finite State Machine (FSM) which role is to elect the tasks and constraints depending on the current goals.}.


\subsection{Weight-Prioritized Multi-Objective Robust QP Formulation}\label{subsec:Robust QP Lipschitz continuity}
By replacing~\eqref{eq:heterogeneous feedback mu} and~\eqref{eq:RECBF formulation} in~\eqref{eq:QP task} and~\eqref{eq:safety constraint}, respectively, the weight-prioritized multi-objective robust QP writes as
\begin{subequations}\label{eq:robust QP for combination}
\begin{align}
\begin{split}\label{eq:robust QP cost function}
\left[\UgenDesJointDyn^*,\force^*\right]&\!=\;\!\arg\min \frac{\weight_0}{2}\!\norm{\selectMat\UgenDesJointDyn\!+ \! \boldsymbol{\kappa}(\actJointDyn)}^2 +\\	&\frac{1}{2}\!\sum_{j=1}\!\weight_{j}\norm{\desJacTaski \UgenDesJointDyn\! + \!\desJacDotTaski \genDesConfDot \! -\! \fkinRefDDot^j(t) \!+\!\highlight{\taskGainsPsii\taskPsii}}^2 
\end{split}	\\
\text{s.t.}&\;\;\; \eqref{eq:dynamics constraint},\eqref{eq:contact forces}\\ 
&-\desJacBfunc\UgenDesJointDyn\!\leq\! \desJacBfuncDot\genDesConfDot \!+\! \highlight{\constraintGainsPsi\bfuncPsi}\\ \label{subeq:contact constraint}
&\;\;\;\; \desJacContact\UgenDesJointDyn  = - \desJacContactDot\genDesConfDot   -k\desJacContact\genDesConfDot
\end{align}
\end{subequations} with $j$ indexing the task $\fkini$. From the structural standpoint, the proposed robust QP~\eqref{eq:robust QP for combination} is similar to the classical weighted-prioritized templated QP~\eqref{eq:QP for combination}. Our contribution is in modifying the task and constraint formulations based on the sum of stiffness and damping terms (computed based on the robot measurements) to include feedback integral terms (computed based on the desired robot state).
Our robust formulation does not call for a substantial modification in the QP formulation as in~\cite{feng2015journalOfFieldRobotics}, nor for much additional computation as the integral terms are obtained by forward velocity, see~\cref{fig:whole control scheme}.

Relatively to~\eqref{eq:QP task},~\eqref{eq:robust QP cost function}  shows that the robust task formulation~\eqref{eq:heterogeneous feedback mu} is straightforwardly extended to the multi-task case where the integral term corresponding to each task $\fkini$ is added.
Due to the subsequent conflicts that may arise between the different tasks,  $\fkini$ is likely to be achieved partially according to its associated weight $\weighti\!\geq\!0$ and the potential active constraints\footnote{An inequality constraint is active if it is enforced as equality.} in QP~\eqref{eq:robust QP for combination}. This implicit relaxation writes
\begin{equation}\label{eq:heterogeneous feedback relaxed}
\taskCrtlIni = -\taskGainsPsii\taskPsii+ \bm{\delta}^{j}(t),
\end{equation} 
with $\bm{\delta}^{j}(t)\!\in\!\mathbb{R}^{m}$ assumed bounded $\norm{\bm{\delta}^{j}(t)}\!\leq\!{\delta}^{j}_{\max}, \ \forall t \!\geq\!0$. 
\cref{prop:relaxed heterogeneous feedback} generalizes \cref{thm:heterogeneous feedback} to  the multi-task case.
\begin{proposition}\label{prop:relaxed heterogeneous feedback}
Consider~\eqref{eq:heterogeneous feedback relaxed}  such that \cref{thm:heterogeneous feedback} conditions hold. Then, there exists $\taskIntegralDampingi\in\mathbb{R}^{m\times m}$ positive-definite such that $\desTaskOuti$ is practically robustly stable. 
\end{proposition}
\begin{proof}
See \cref{proof:prop3}.
\end{proof}
Note that thanks to \cref{prop:relaxed heterogeneous feedback,thm:RECBF,prop 1}, $\genDesJointDyn$ is bounded after double integration of $\UgenDesJointDyn$ solution of~\eqref{eq:robust QP for combination}.
\begin{remark}\label{rem2}
In~\eqref{eq:robust QP for combination}, only one RECBF constraint is considered. In the more general case, the QP constraints set encompasses:  (i) several RECBFs (joint constraints, collision avoidance, CoM equilibrium region, etc.), and (ii) explicit bounds on $\jointCrtlIn$. Therefore, we highlight two important aspects: First, \cref{thm:RECBF} assumes that ${\cal U}=\mathbb{R}^{6+n}$ which may not hold. Second, since all the constraints have the same level of priority, the QP will be infeasible if these constraints are in conflict (incompatible, i.e., empty feasibility domain)~\cite{decre2009icra,rubrecht2010iros,delprete2018ral}. 
\end{remark}

\section{Experimental Results and Discussion}\label{sec:Experiments}
To assess our robust QP controller and demonstrate its applicability to different use-cases, experiments are conducted with two different robots: a fixed-base 7-DoF robotic arm Panda from Franka Emika, and a (floating-base) 34-DoF humanoid robot HRP-4 from Kawada Robotics. The latter is controlled in position at a frequency of $200$~Hz, whereas the former can be controlled either in position or in velocity modes at a control frequency of $1$~kHz.

Both robots are controlled using the open source code implementation of the QP controller \mcrtc\footnote{\url{https://jrl-umi3218.github.io/mc_rtc/index.html}} with LSSOL solver. It includes a user task specification interface, debugging, data recording... for simulation as well as for real-time control. Based on the embedded sensors data (encoders, IMU, Force/Torque (F/T) sensors...), \mcrtc~builds at each control-cycle the QP problem, based on user-defined tasks and constraints, and solves it. The QP decision variables are $\UgenDesJointDyn$ and the contact forces $\force$. 

As it is shown in~\eqref{eq:double integrator} and~\cref{fig:whole control scheme}, $\UgenDesJointDyn$ is integrated twice to obtain $\genDesJointDyn$, then the joint commands $\desJointDyn$ are sent to the actuated joint controllers. $\actConfDot$ is estimated using numerical derivation for HRP-4, whereas it is readily available in the Franka Control Interface (FCI) for Panda\footnote{The FCI joint-velocity estimation is a low-pass filtered finite-difference.}. The floating-base state $\actFBDyn$ of HRP-4 is estimated by a built-in kinematic-inertial observer based on an extended filter. 
The control is performed using a laptop Dell Precision~5540 with Intel Xeon(R)~E-2276M processor (CPU $12 \times  2.80$~GHz) and $32$~GB of RAM running under Linux Ubuntu~18.04.5 LTS.

\begin{figure*}
\centering
\subfloat[]{
\includegraphics[width=0.95\columnwidth]{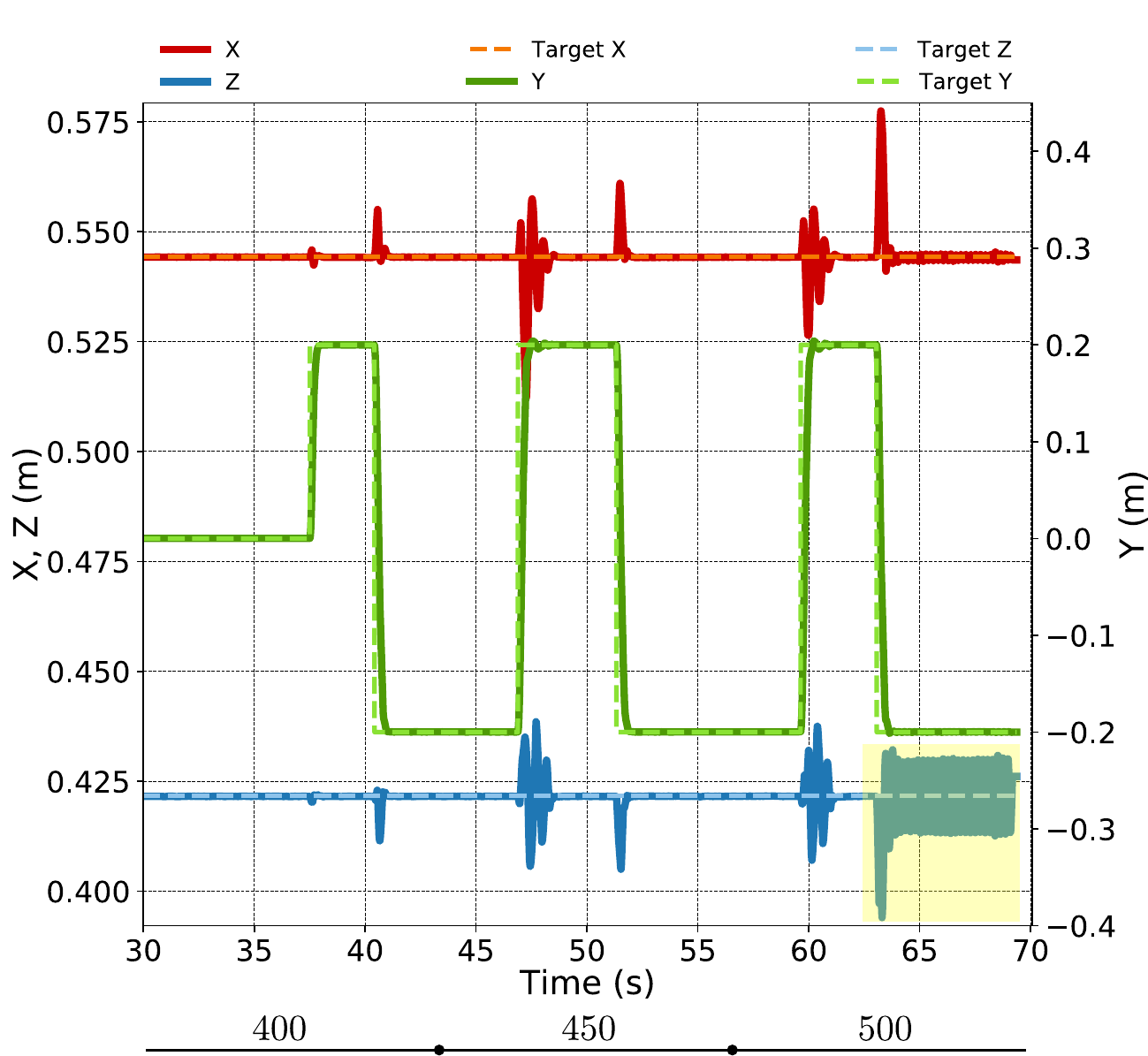}
\label{subfig:pandaFail_Task}}
\subfloat[]{
\includegraphics[width=0.95\columnwidth]{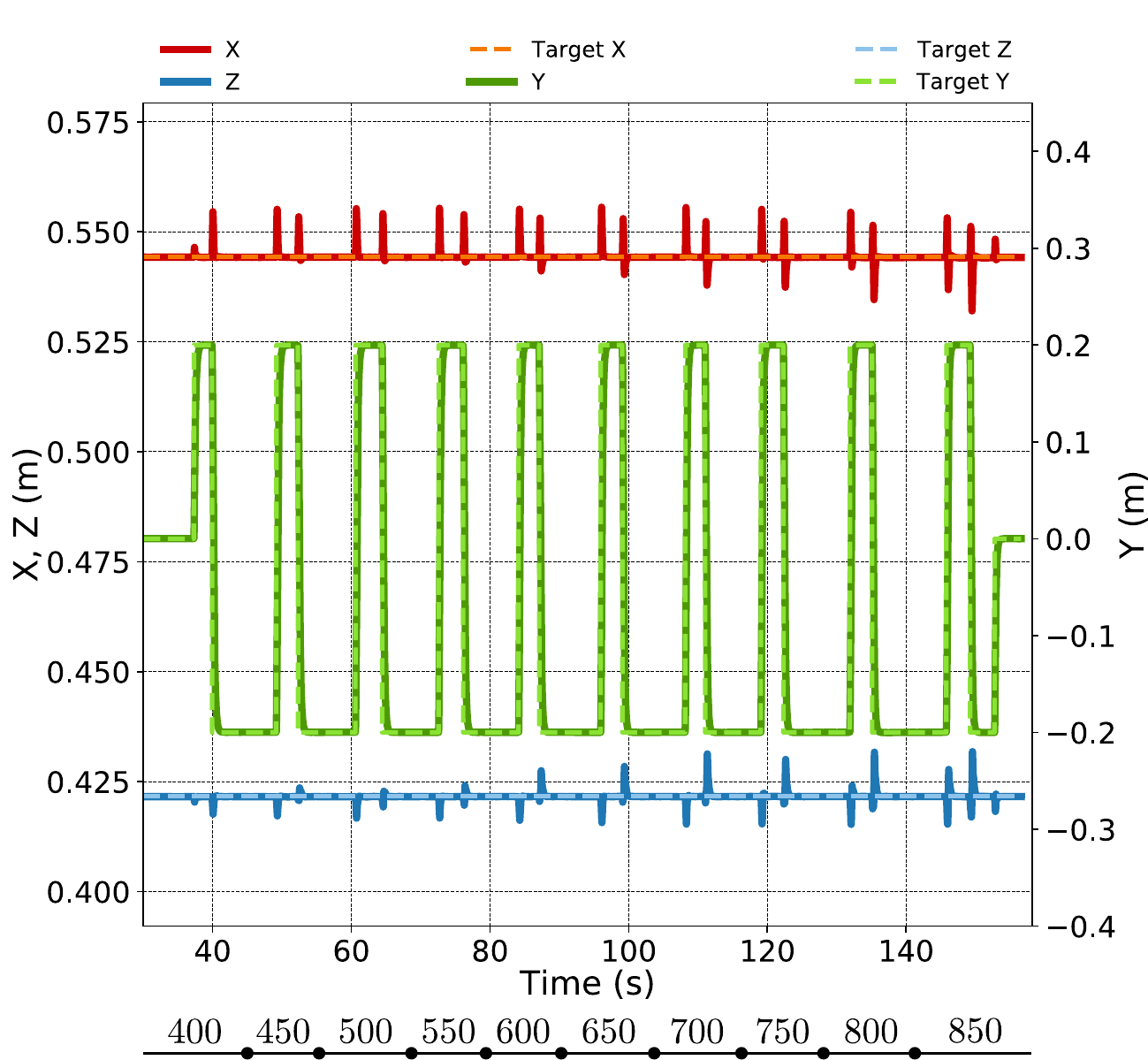}
\label{subfig:pandaSuccess_Task}}
\hfil
\subfloat[]{
\includegraphics[width=0.95\columnwidth]{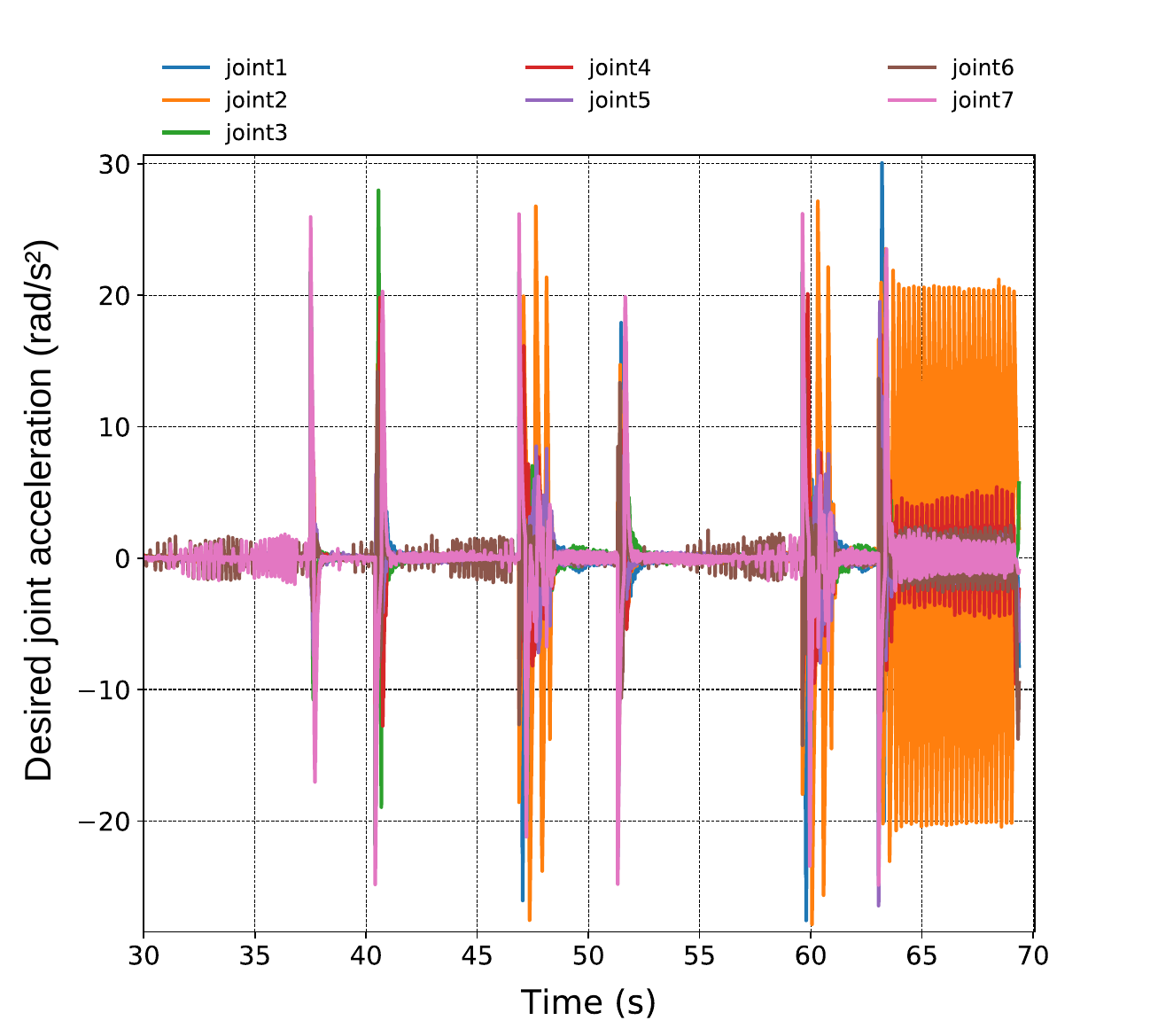}
\label{subfig:pandaFail_u}}
\subfloat[]{
\includegraphics[width=0.95\columnwidth]{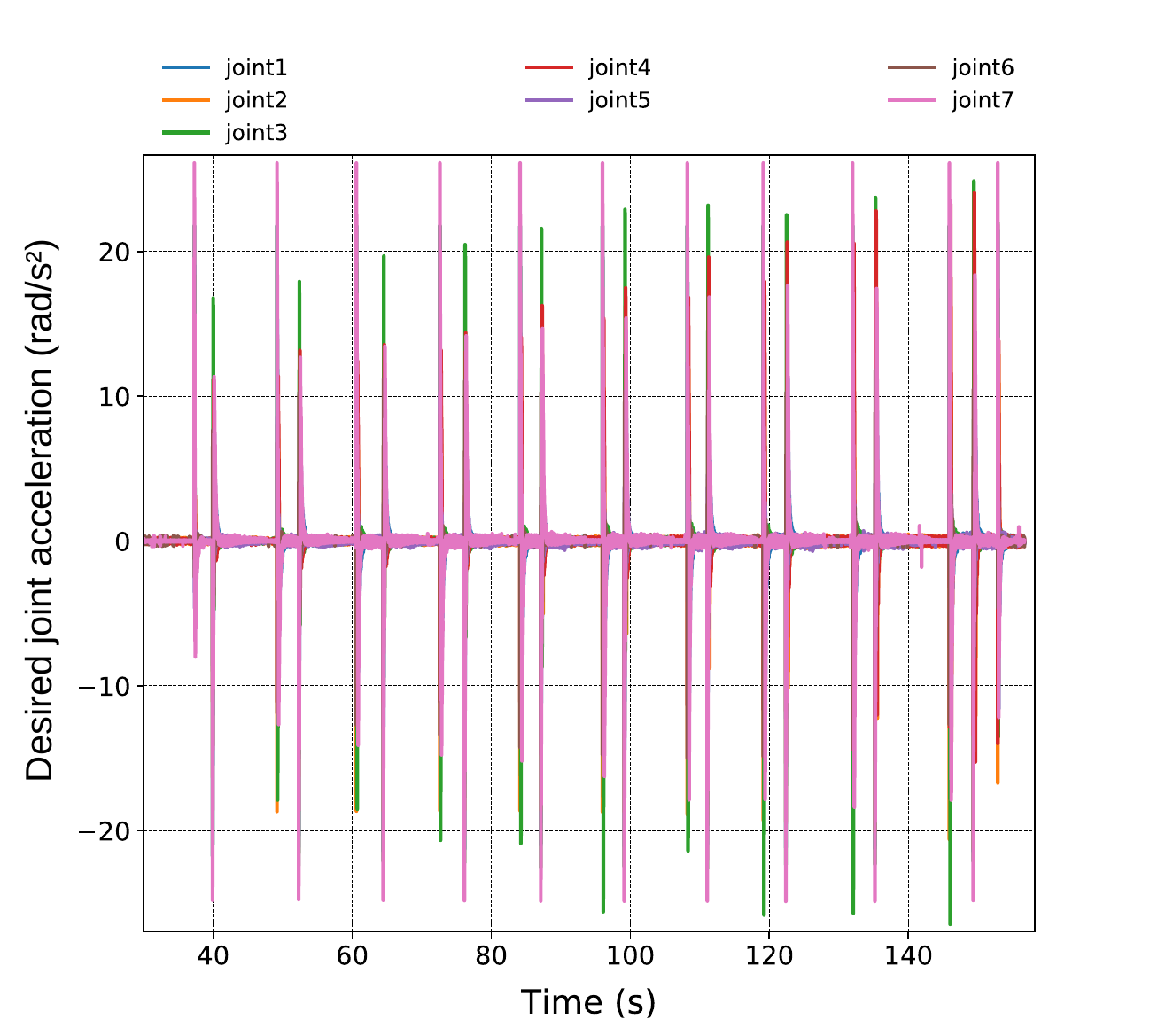}
\label{subfig:pandaSucess_u}
}
\caption{Panda response for the ‘pick-and-place’ task under different feedback controls. End-effector Cartesian coordinates and $\jointCrtlIn$ evolution under: output feedback~\eqref{eq:mu output feedback} \subref{subfig:pandaFail_Task}--\subref{subfig:pandaFail_u},  heterogeneous feedback~\eqref{eq:heterogeneous feedback mu} \subref{subfig:pandaSuccess_Task}--\subref{subfig:pandaSucess_u}. The horizontal scales under \subref{subfig:pandaFail_Task} and \subref{subfig:pandaSuccess_Task} denote the stiffness gain $\taskStiffness$ within different time periods.} 
\label{fig:pandaExperiment}
\end{figure*}
\begin{figure}[t!]
\centering
\subfloat[]{
\includegraphics[width=0.5\columnwidth]{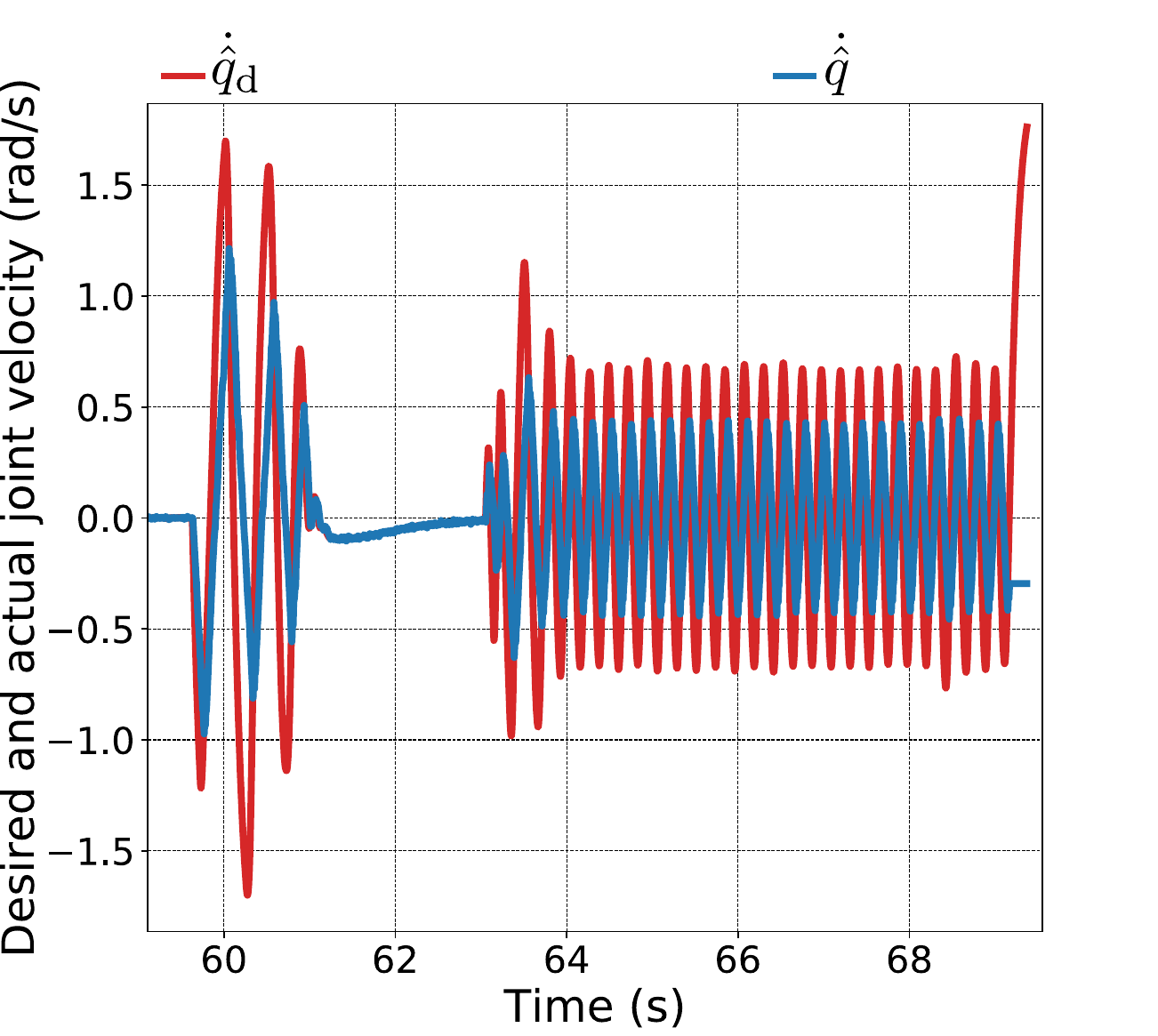}
\label{subfig:panda-velTracking-Fail}}
\subfloat[]{
\includegraphics[width=0.5\columnwidth]{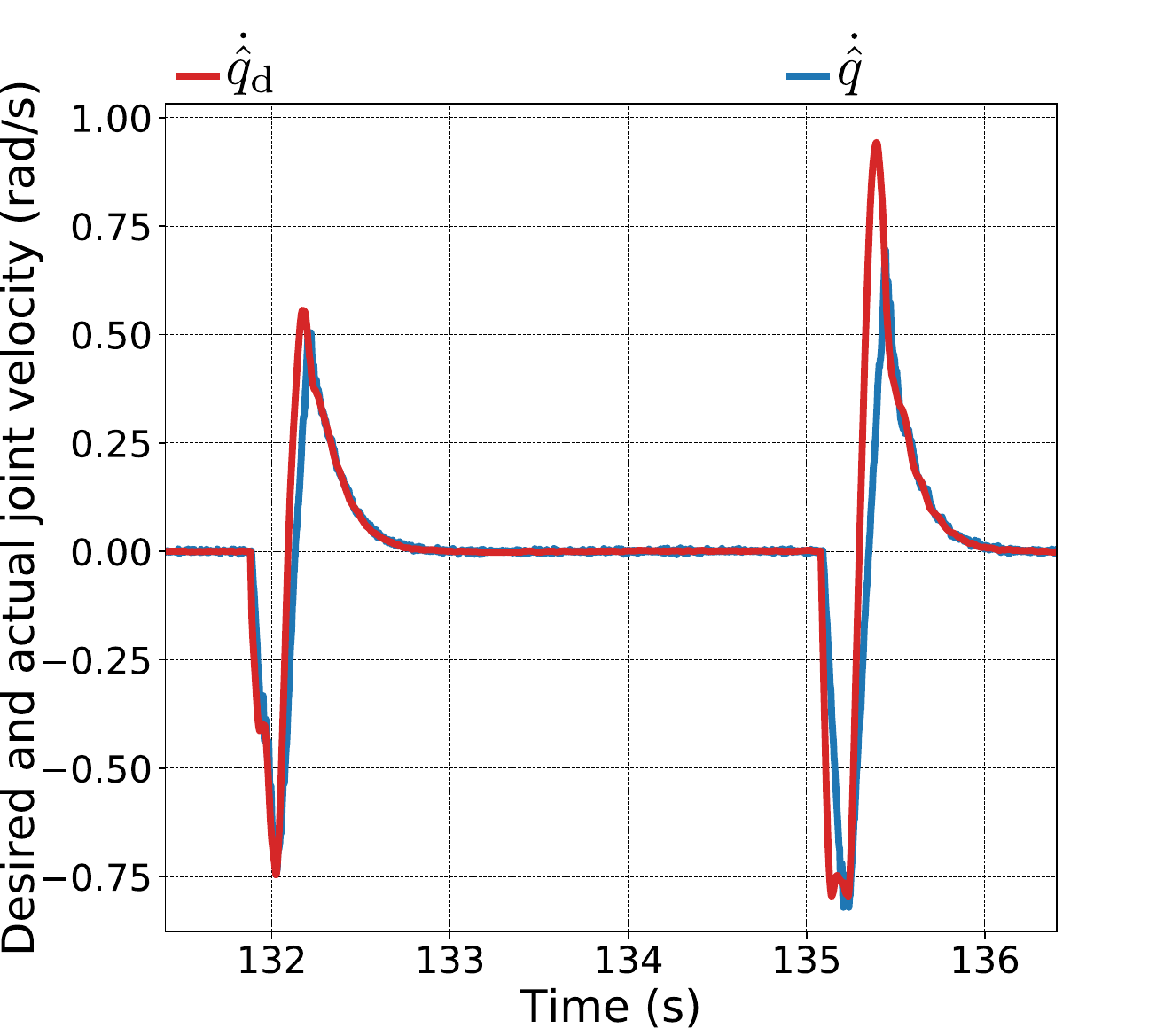}
\label{subfig:panda-velTracking-Success}}
\caption{Joint velocity tracking of the 2$^{\text{nd}}$ joint.~\subref{subfig:panda-velTracking-Fail} Closed-loop system under output feedback~\eqref{eq:mu output feedback} at stiffness $\taskStiffness=500\eye$ leads to instability shown as fast oscillation of $\desConfDot$ tracked by $\actConfDot$. \subref{subfig:panda-velTracking-Success} Closed-loop system under heterogeneous feedback~\eqref{eq:heterogeneous feedback mu} at stiffness $\taskStiffness=800\eye$ where $\desConfDot$ is kept bounded even though it is not well tracked by $\actConfDot$, leading to a stable response.}
\label{fig:panda-velTracking}
\end{figure}
\begin{figure}[t!]
\centering
\includegraphics[width=.85\columnwidth]{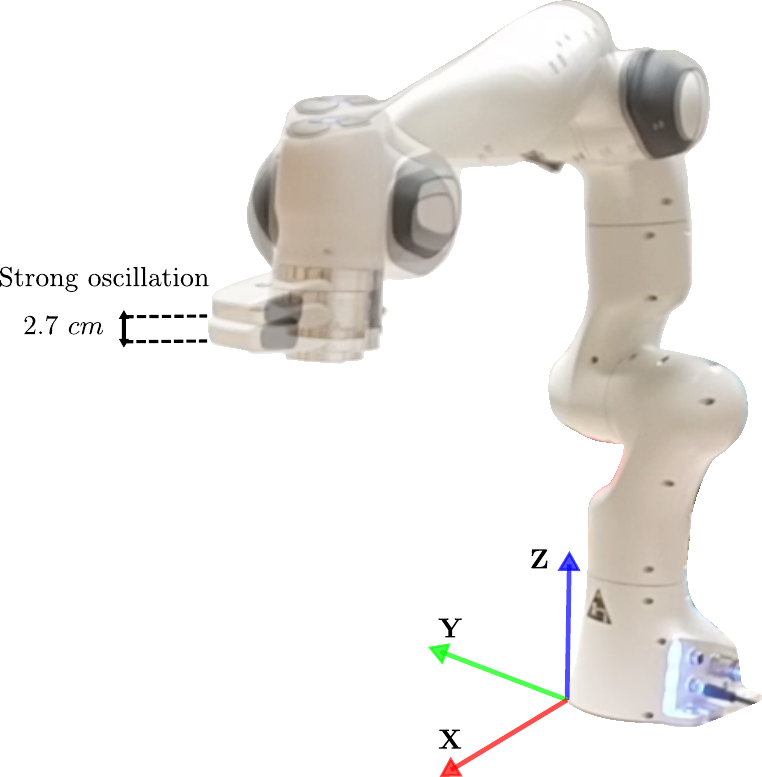}
\caption{Strong oscillations (highlighted in the yellow spot in \cref{fig:pandaExperiment}\subref{subfig:pandaFail_Task}) due to non-robustness of output feedback control~\eqref{eq:mu output feedback}. The two superposed snapshots are taken with a time interval of $ T = 133$~ms.}
\label{fig:panda_vib}
\end{figure}
\begin{figure}
\centering
\includegraphics[width=1\columnwidth]{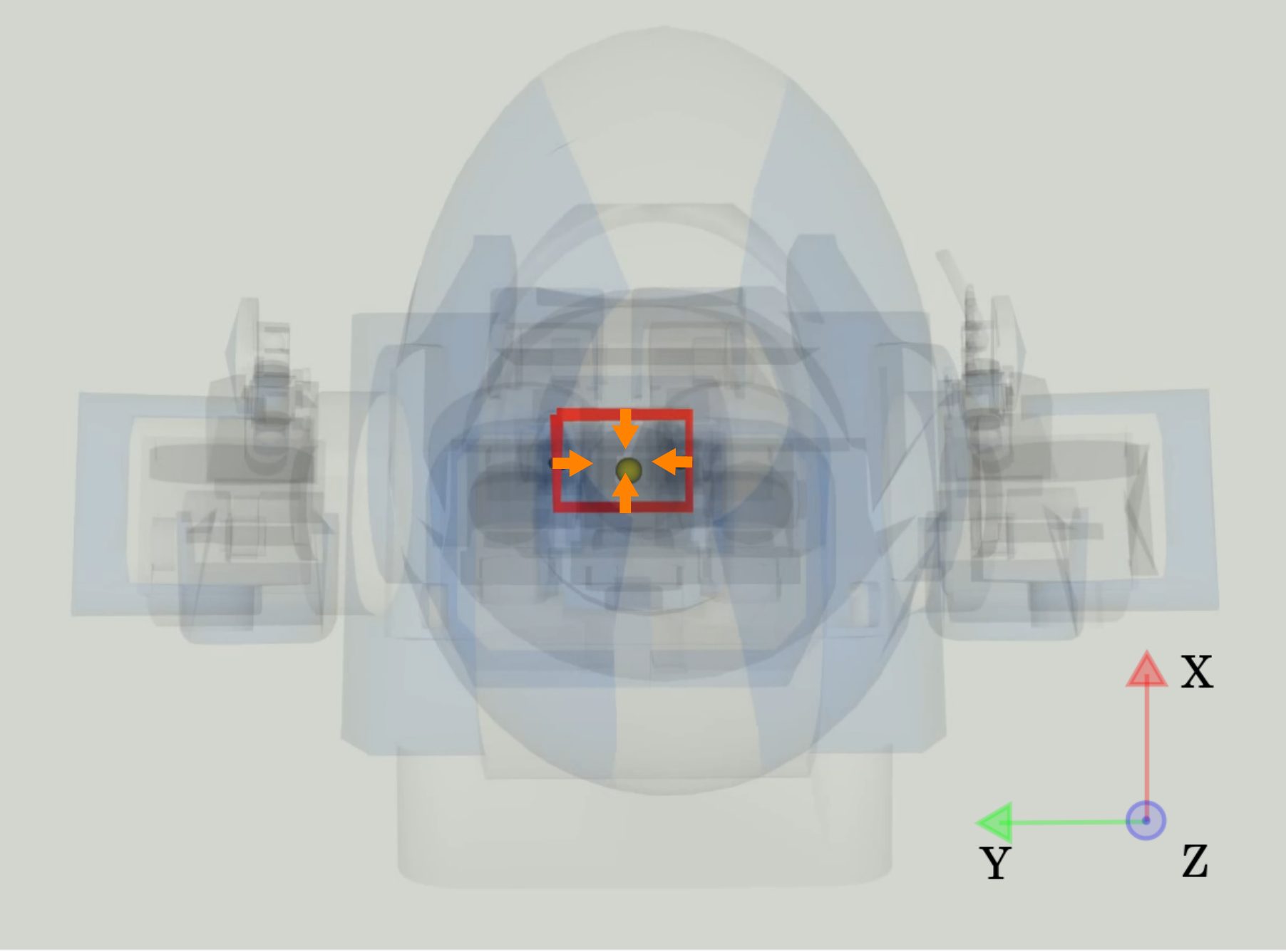}
\caption{Top-view of the humanoid robot HRP-4. The conservative equilibrium polygon is shown in red, the CoM in a yellow dot and the edges normal vectors in orange. The polygon is a rectangle in $XY$ plane such that $X_{\max} = 5$~cm, $X_{\min} = -2$~cm, $Y_{\max} = 5$~cm, $Y_{\min} = -5$~cm.}
\label{fig:CoM polgon}
\end{figure}
\begin{figure}[t!]
\centering
\includegraphics[width=.65\columnwidth]{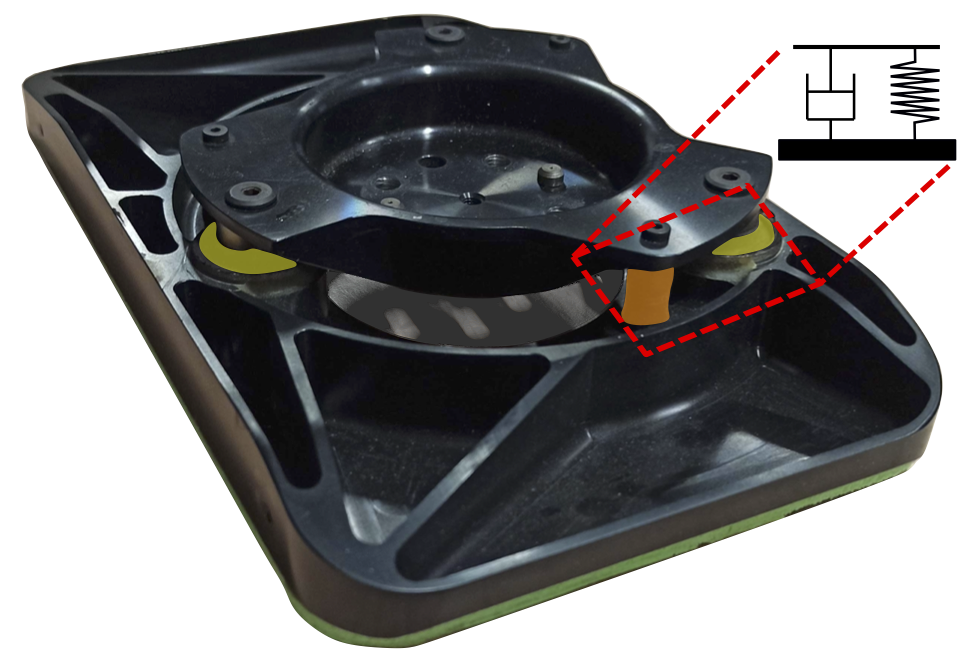}
\caption{Rubber bushes (yellow) and dampers (orange) under HRP-4 ankles that induce non-modeled flexibilities.}
\label{fig:hrp soles}
\end{figure}
\begin{figure}[t!]
\centering
\subfloat[]{
\includegraphics[width=0.95\columnwidth]{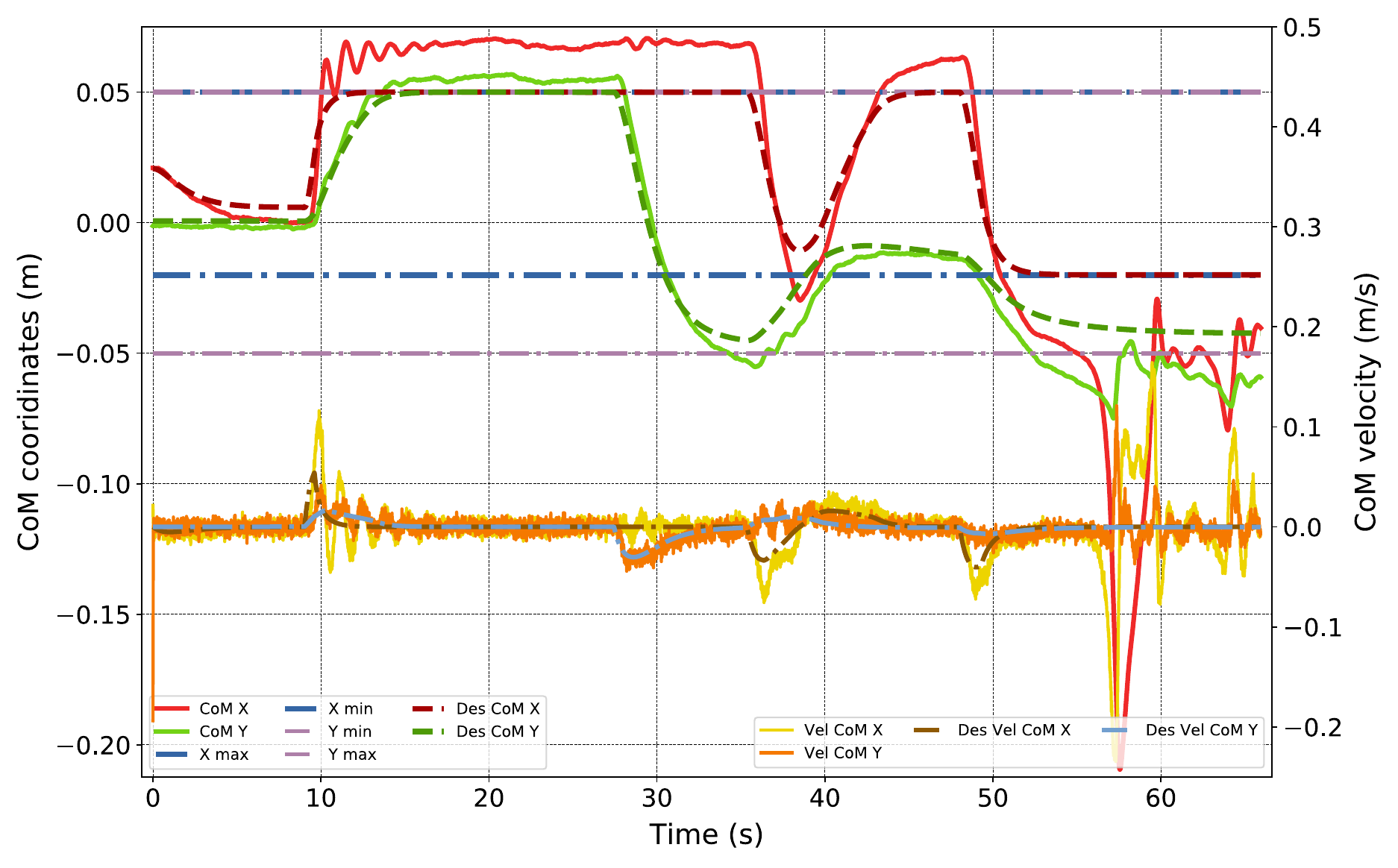}
\label{subfig:com_OL}
}\hfil
\subfloat[]{
\includegraphics[width=0.95\columnwidth]{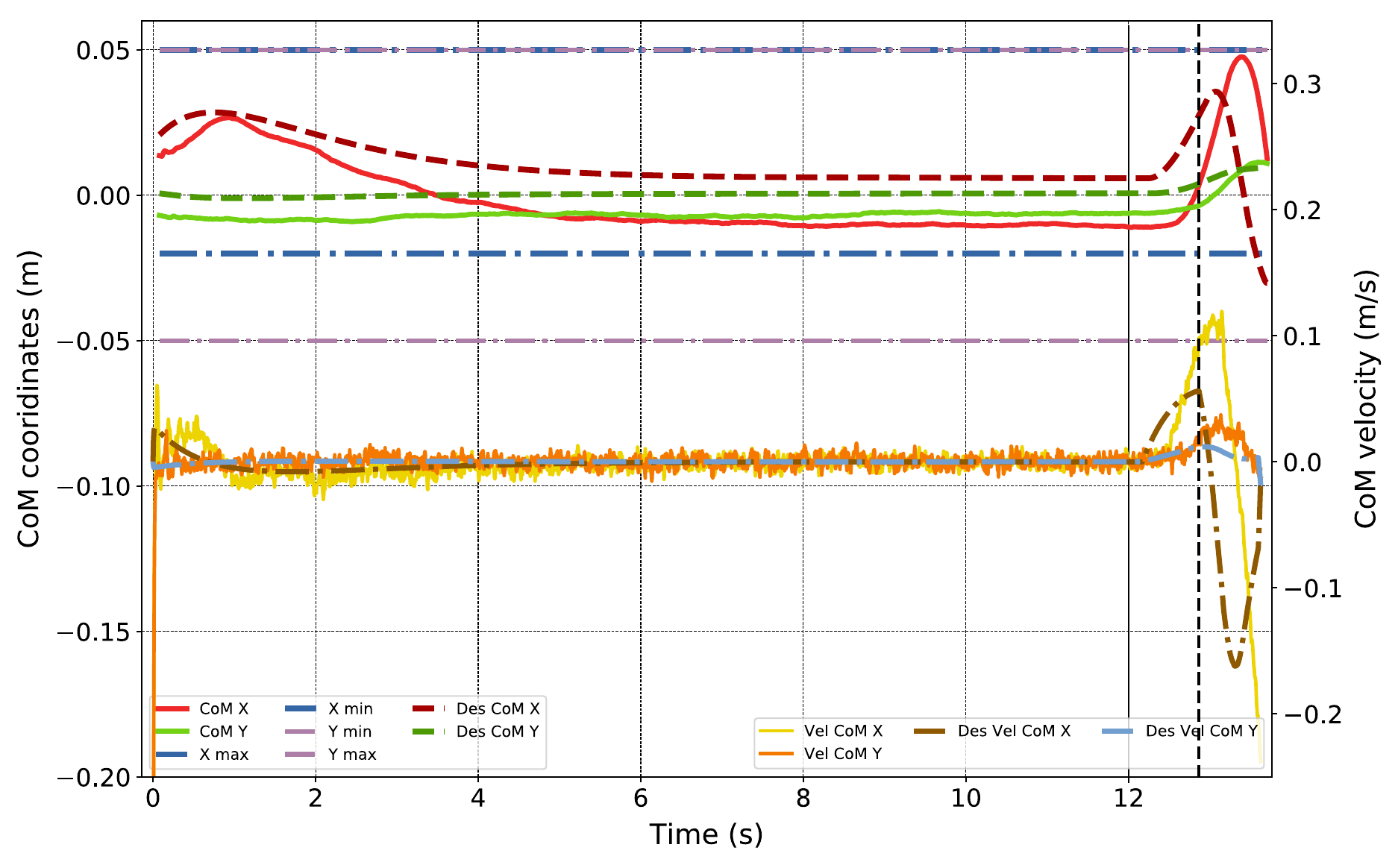}
\label{subfig:com_CL}
}\hfil
\subfloat[]{
\includegraphics[width=0.95\columnwidth]{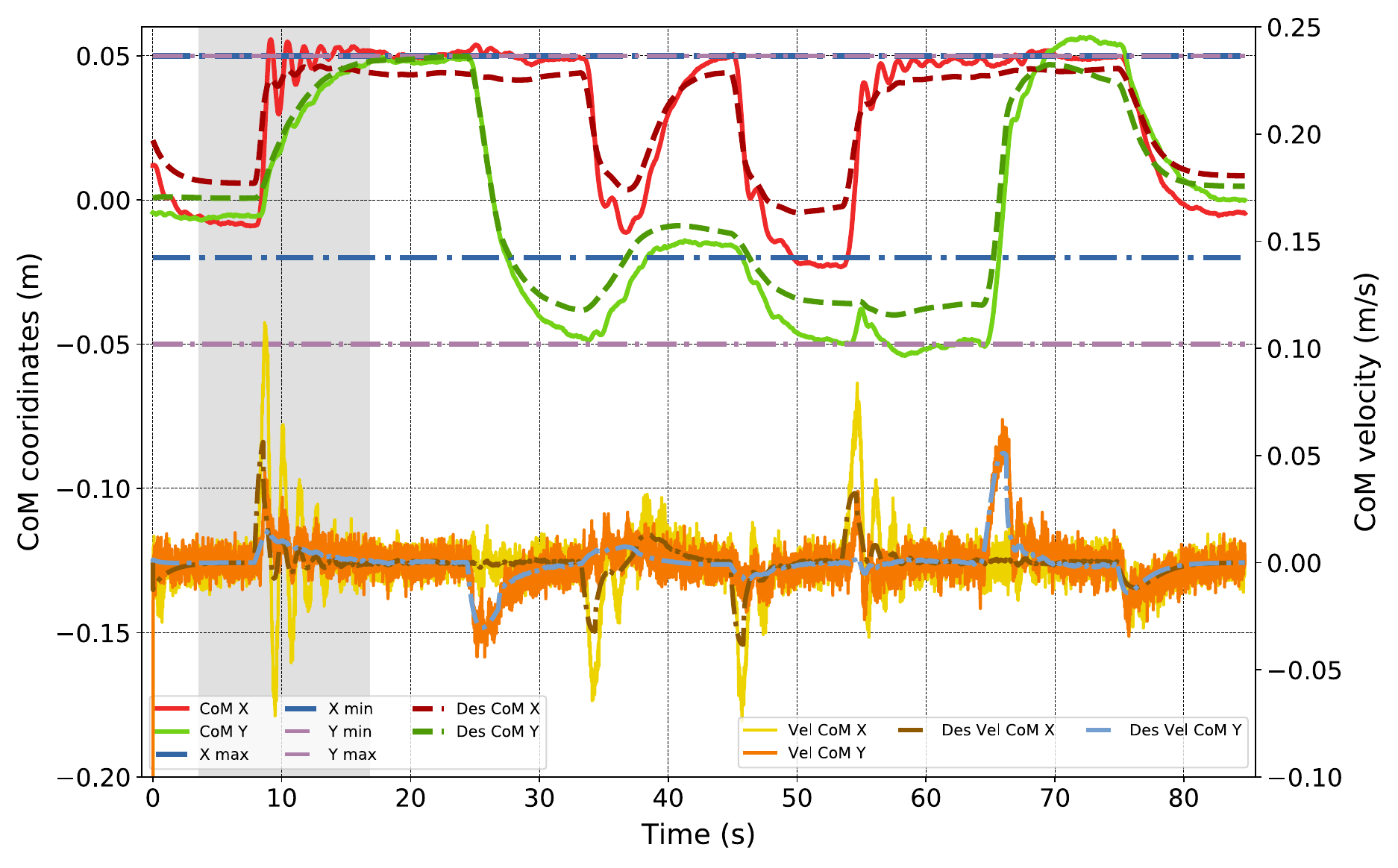}
\label{subfig:com_Robust}}
\caption{Time Evolution of $\actCoM$, $\desCoM$ and their respective velocities coordinates along $X$ and $Y$ axes.~\subref{subfig:com_OL} Feedforward ECBF constraint~\eqref{eq:ECBF constraint}.~\subref{subfig:com_CL} Feedback ECBF constraint~\eqref{eq:ECBF in feedback}.~\subref{subfig:com_Robust} RECBF constraint~\eqref{eq:RECBF formulation}. The gray time slot is zoomed-in in \cref{fig:com_Robust Zoom}.}
\label{fig:com}
\end{figure}
\begin{figure}
\centering
\includegraphics[width=1\columnwidth]{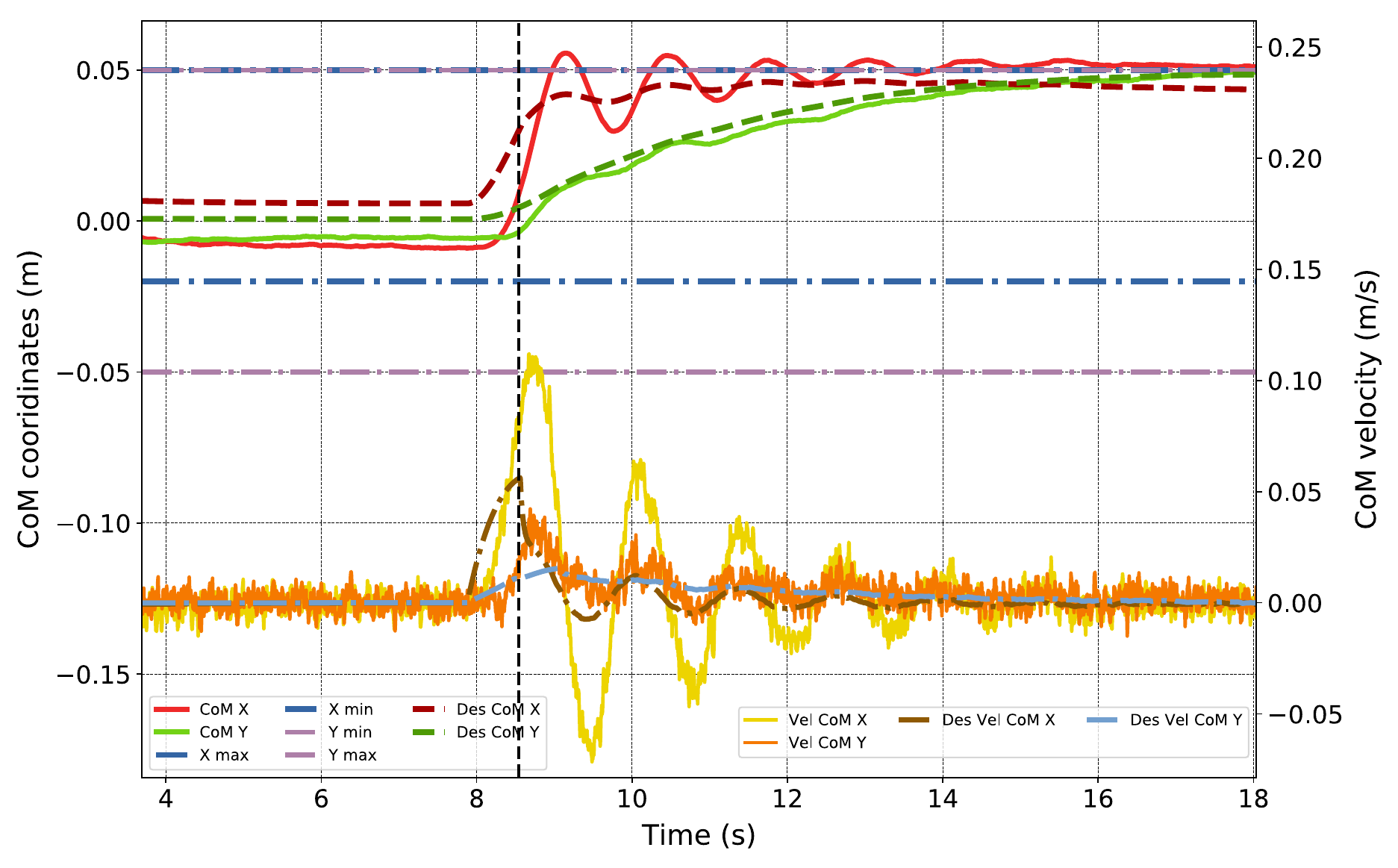}
\caption{Zoom-in of the gray time slot in \cref{fig:com}\subref{subfig:com_Robust}. The bold dashed line denotes the moment when the RECBF constraint relative to $X_{\max}$ boundary is inserted in QP~\eqref{eq:robust QP for combination}.}
\label{fig:com_Robust Zoom}
\end{figure}
\begin{figure}
\centering
\includegraphics[width=0.49\columnwidth]{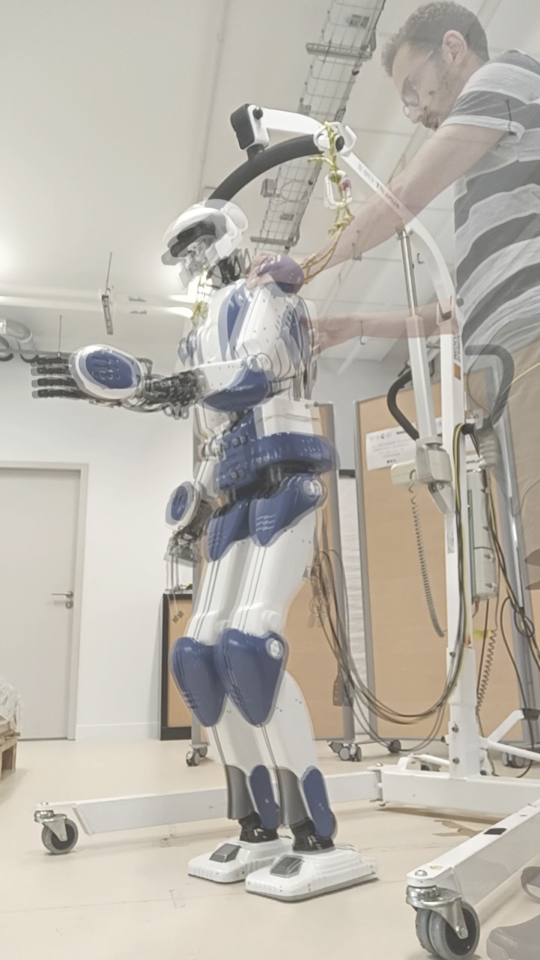}
\includegraphics[width=0.49\columnwidth]{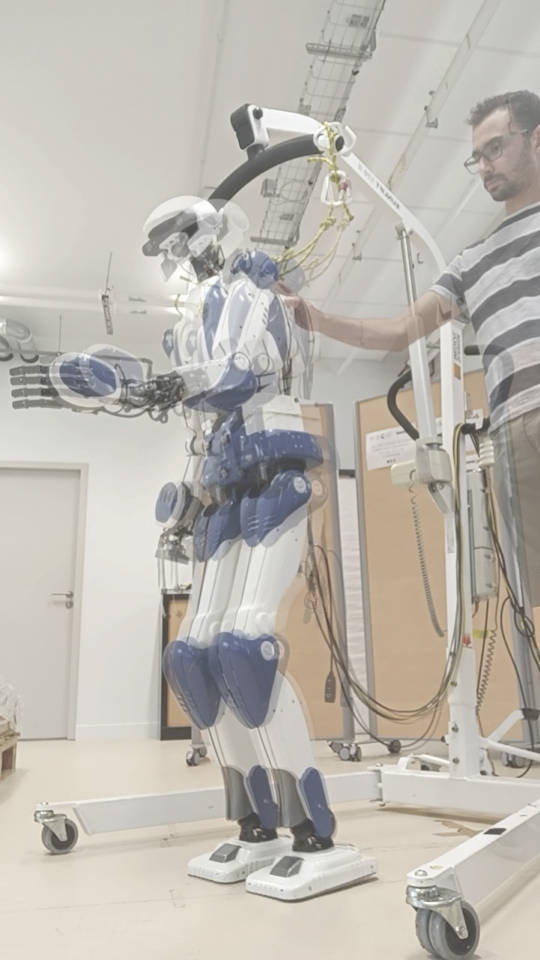} \\
\includegraphics[width=0.49\columnwidth]{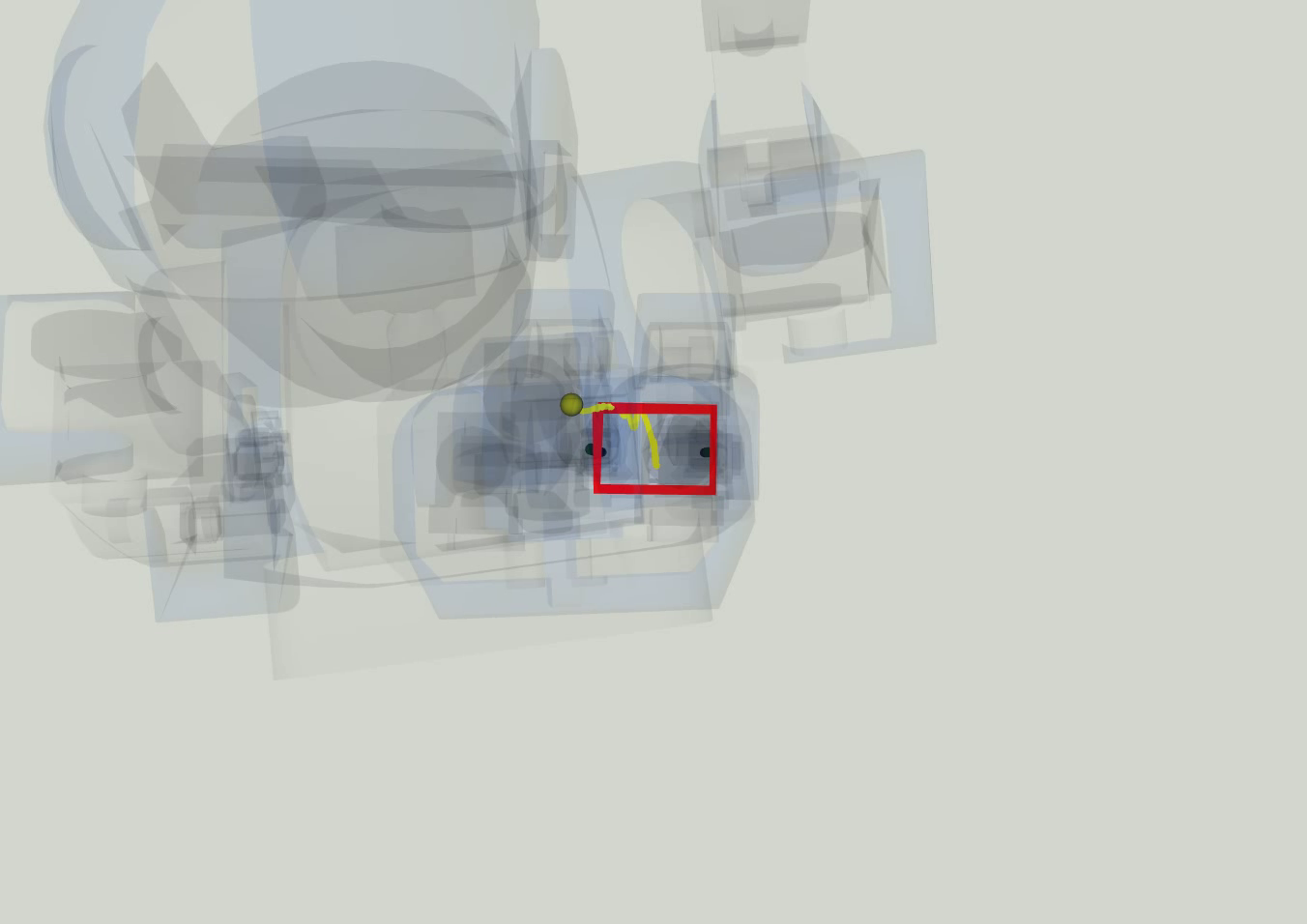}
\includegraphics[width=0.49\columnwidth]{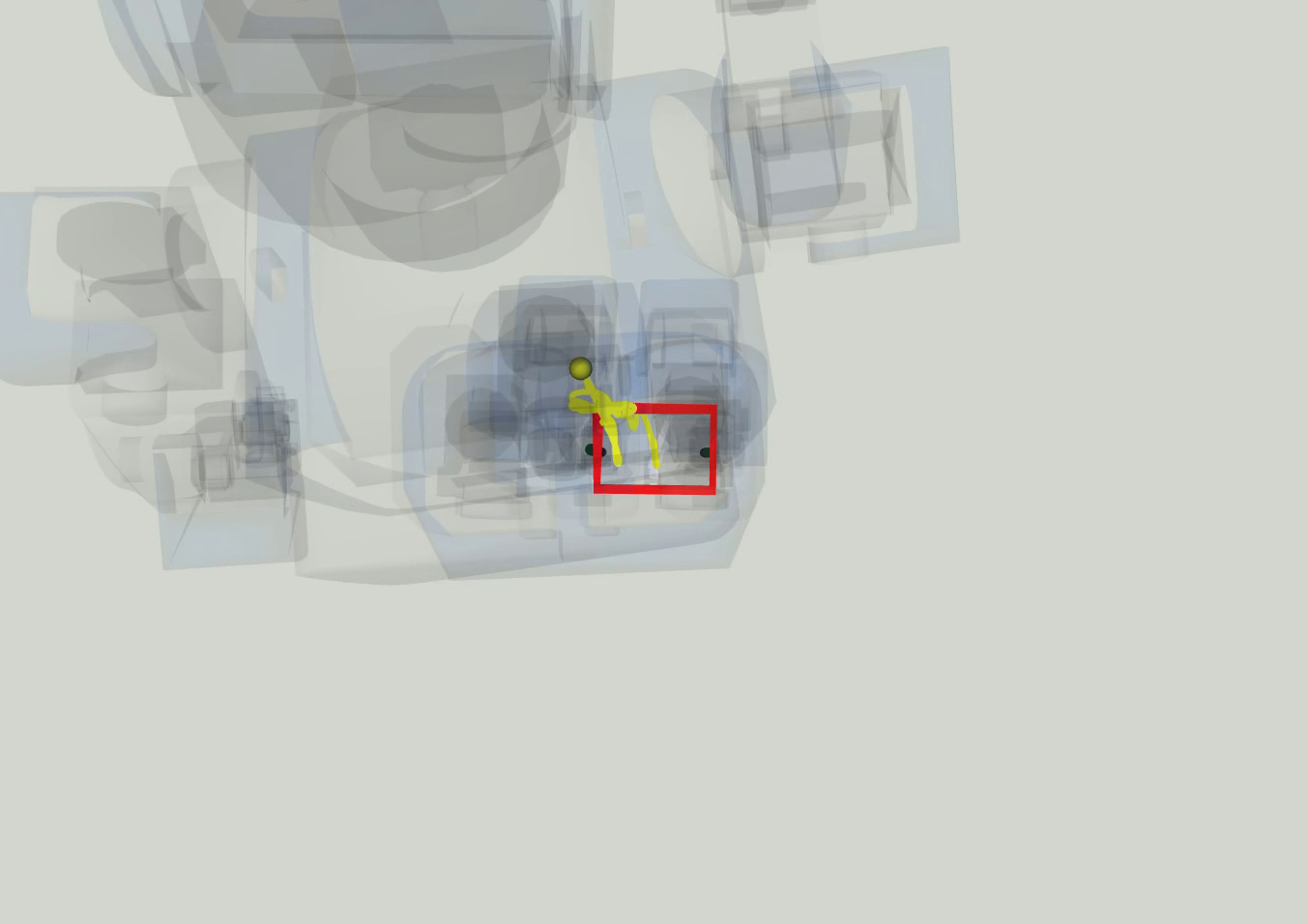} \\
\caption{Superposed snapshots of robust against pushing (experiment 4) in $X$ (right-top) and $Y$ (left-top) directions, with the corresponding top-view perspectives (bottom).}
\label{fig:com push snapshots}
\end{figure}
\begin{figure}[t!]
\centering
\includegraphics[width=\columnwidth]{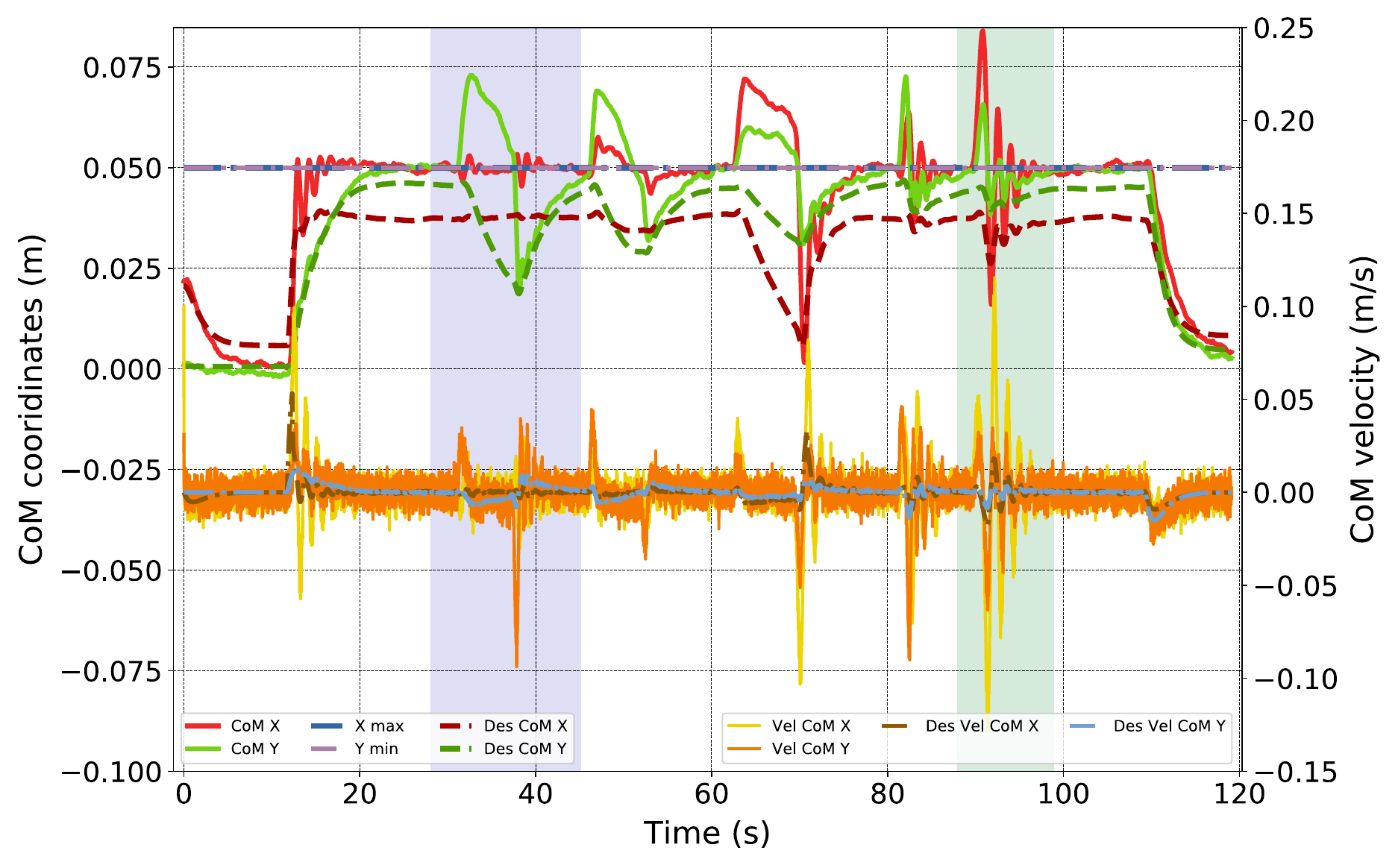}
\caption{Robustness of RECBF~\eqref{eq:RECBF formulation} against external pushes. The blue time slot denotes a persistent external push (\cref{fig:zoom com push}\subref{subfig:com persistent push}), and the green one denotes a brief external push (\cref{fig:zoom com push}\subref{subfig:com brief push}).}
\label{fig:com push}
\end{figure}
\begin{figure}
\centering
\subfloat[]{
\includegraphics[width=\columnwidth]{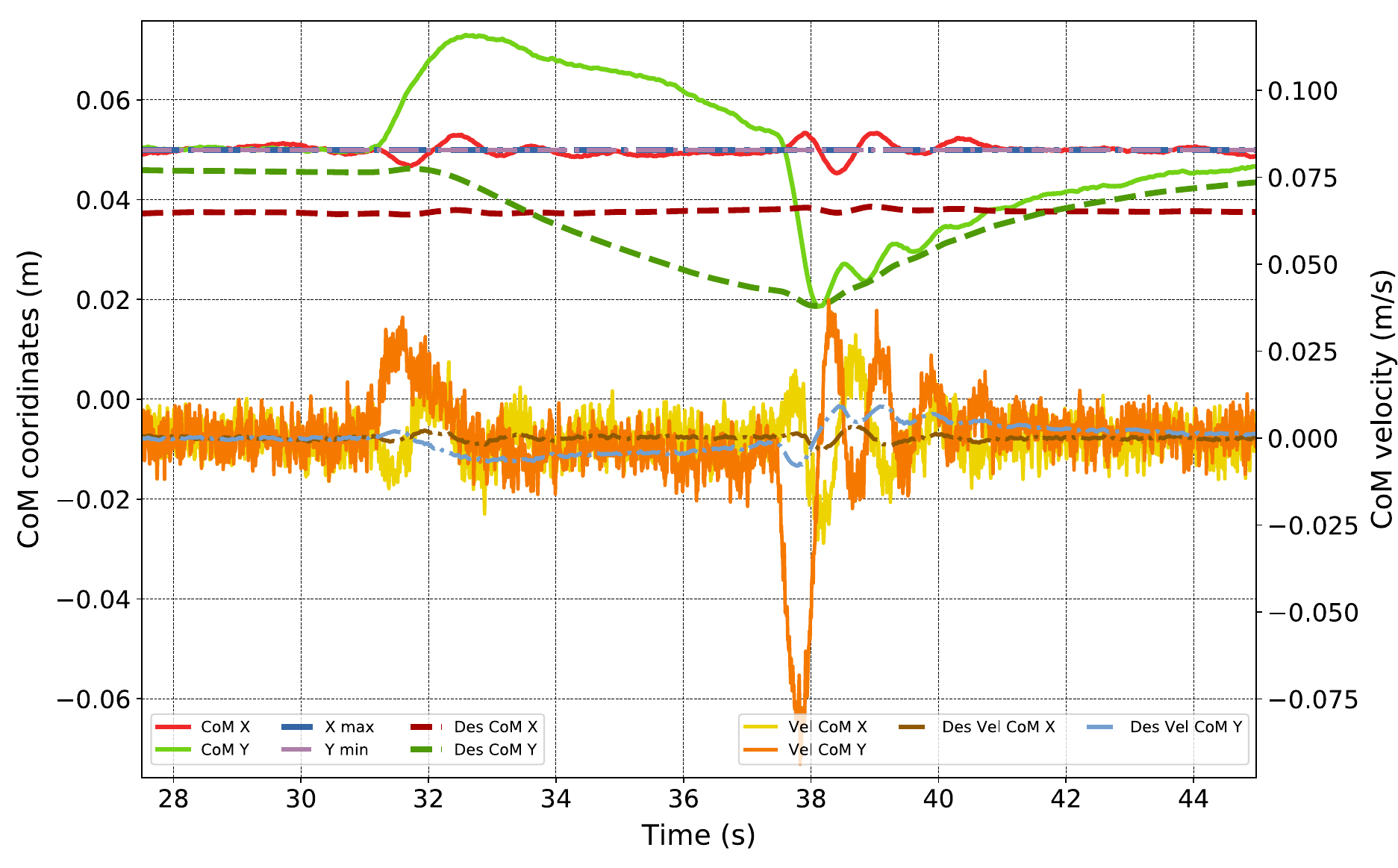}
\label{subfig:com persistent push}
}
\hfil
\subfloat[]{
\includegraphics[width=\columnwidth]{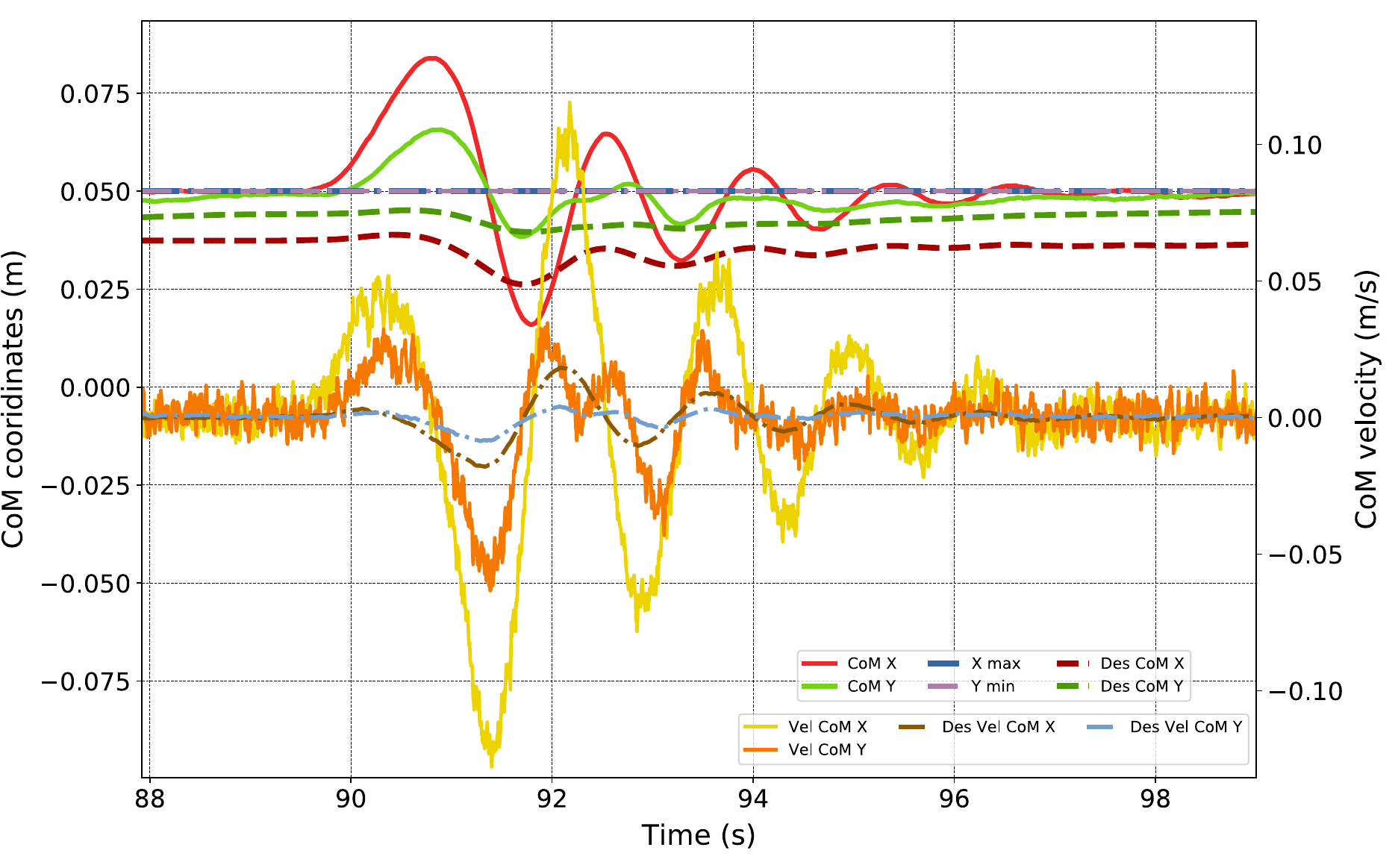}
\label{subfig:com brief push}
}
\caption{  \subref{subfig:com persistent push} Zoom-in the blue time slot in~\cref{fig:com push} showing the response against a persistent  push. \subref{subfig:com brief push} Zoom-in the green time slot in~\cref{fig:com push} showing the response against a brief push.}
\label{fig:zoom com push}
\end{figure}

\subsection{Task Robust Stability}\label{subsec:experiment_robust stability} 
Tasks gains correlate to accuracy (performing motions with high precision) and execution speed (controlling the overall task time). Consequently, performing fast and precise motion requires increasing the task gains which leads to instability.
In this context, the Panda end-effector is controlled to perform a pick-and-place-like motion. Two set-point targets (position and orientation) are defined (to which the end-effector converges back and forth); $\desConfDot$ is the input command for the joint controllers. To have simple plots, only the target coordinate along the $Y$-axis varies with an amplitude of $\pm20$~cm (\cref{fig:panda_snap}).

QP~\eqref{eq:robust QP for combination} is formulated (the notations in \cref{rem1} are followed). such that the constraint set contains the kinematic constraints (joint-position and velocity constraints~\cite{djeha2020ral}); whereas, for comparison purpose, the task is formulated using: (i) output feedback~\eqref{eq:mu output feedback}, (ii) heterogeneous feedback~\eqref{eq:heterogeneous feedback mu}. 
The task gains are set as $\taskDamping\!=\!2\sqrt{\taskStiffness}$ and $\taskIntegralDamping\!=\! \varepsilon \taskDamping$ with $\varepsilon\!=\!0$ for~\eqref{eq:mu output feedback} and $\varepsilon\!=\!1$ for~\eqref{eq:heterogeneous feedback mu}. Since we do not know \emph{a priori} the task stiffness that will turn the closed-loop system instable, $\taskStiffness$ is initially set to $400\eye$, then it is increased over time by increments of $50\eye$ ($\taskDamping$ and $\taskIntegralDamping$ are updated accordingly). 

\Cref{fig:pandaExperiment} shows the experiment results. For $\taskStiffness\!\leq \!450\eye$, both feedback controls~\eqref{eq:mu output feedback} and~\eqref{eq:heterogeneous feedback mu} lead to stable convergence to the targets. However, for  $\taskStiffness\!=\!500\eye$, the closed-loop system with output feedback~\eqref{eq:mu output feedback} becomes instable (\cref{fig:pandaExperiment}\subref{subfig:pandaFail_Task}) where strong oscillations and jerky motion appear at the end-effector mostly visible along the $Z$-axis (\cref{fig:panda_vib}). This chattering can be very dangerous for the robot (accelerating drastically the wear of actuators and robot's structure) and the surrounding people or objects in the robot's neighborhood. Conversely, heterogeneous feedback~\eqref{eq:heterogeneous feedback mu} allows reaching robustly the targets while $\taskStiffness$ keeps increasing up to $850\eye$ (\cref{fig:pandaExperiment}\subref{subfig:pandaSuccess_Task}).

Increasing the task gains results in high values of desired joint acceleration $\jointCrtlIn$ (\cref{fig:pandaExperiment}\subref{subfig:pandaFail_u}--\subref{subfig:pandaSucess_u}) which generates desired joint commands $\desConfDot$ with fast variations that can not be well tracked by the joint controllers (\cref{fig:panda-velTracking}\subref{subfig:panda-velTracking-Fail}--\subref{subfig:panda-velTracking-Success}) due to the different rate limitations (acceleration, jerk) and the limited bandwidth. This leads to increase the joint tracking error $\jointTrackErr$ in~\eqref{eq:joint tracking error}, and correspondingly its task-space mapping $\taskTrackErr$ in~\eqref{eq:DL act task}. 
Consequently, if the perturbation term $\taskGains\taskTrackErr$ 
is not sufficiently bounded, the closed-loop system under output feedback~\eqref{eq:mapping mu to u} is instable.
In particular, adding the task-space integral term in~\eqref{eq:heterogeneous feedback mu} allows to withstand the perturbation by the gain $\taskIntegralDamping$ as shown in~\eqref{eq:cdt on norm - Lv} enforcing $\desTaskOut $ to remain bounded which leads to the boundedness of $\desJointDyn $ (by virtue of \cref{prop 1}).  

In~\cref{fig:pandaExperiment}\subref{subfig:pandaSuccess_Task}--\subref{subfig:pandaSucess_u}, we decided to stop at stiffness $\taskStiffness=850\eye$ as, due to hardware limits, further increasing the task gains has no effect on the convergence performance (\cref{fig:panda-velTracking}\subref{subfig:panda-velTracking-Success}). Although the joint controllers reached their maximum tracking performances, our approach enables a stable motion even though the task gains keep increasing and imposing desired dynamics that cannot be performed by the robot. Hence, the closed-loop stability has to be ensured whatever the task gains and the joint-dynamics. Yet, ensuring the stability goes at the expanse of $\taskIntegralDamping$ conservative tuning.  	

\subsection{Set Robust Stability}\label{subsec:expreiment_robust safety}

Among usual safety constraints common to all robots such as joint limits and self-collisions avoidance that are implemented here as extension of~\cite{djeha2020ral}, a critical safety feature in humanoids is enforcing balance, which is given a higher priority over manipulation tasks. When there are no contact transitions, constraining the CoM position from acceleration bounds enforces robust balance~\cite{audren2018tro}.
For co-planar feet contact, it is enough to confine the CoM of the HRP-4 humanoid robot to remain inside a conservative polygon (\cref{fig:CoM polgon}) such that its boundaries are reached with zero CoM velocity and acceleration. It is a conservative balance region because it is a subset of 3D balance set (polyhedron in multi-contact or prism in co-planar) that we use mainly for validation purpose. The robot CoM is pushed to the polygon boundaries by defining a sequence of Cartesian targets for the right hand to be reached. 
Rubber bushes and dampers are present under the robot ankles to absorb impacts at the feet while walking (\cref{fig:hrp soles}). This shock-absorbing mechanism creates non-modeled underdamped flexibilities between the ankles and the feet. The latter are often observed in the form of small passive oscillations at the ankles which amplify through the whole structure. Their effect is not observed at the joints' encoders, but at the floating base state estimation (based on IMU measurements). Nevertheless, the joint feedback is used (along with the floating base state) to compute CoM Cartesian position and linear velocity.
Moreover,  the floating-base IMU noise affects the CoM estimation (denoted $\actCoM\in\mathbb{R}^3$).  

The robot CoM is constrained to be within the balance polygon by defining inequality constraints on the distance between the CoM and the polygon features. The 3D case would simply result in more inequalities constraining the CoM within a precomputed polyhedron~\cite{audren2018tro}.
The barrier functions $\bfunc^i$ and $\desBfunc^i$ corresponding to each balance polygon feature $i$ are
\begin{align}
\label{eq:bfunc EXP}\bfunc^i & = \bm{n}^{i\tp}\actCoM + \Delta^i,\\
\label{eq:des bunc EXP}\desBfunc^i & = \bm{n}^{i\tp}\desCoM + \Delta^i,
\end{align} 
where $\bm{n}^{i}\in\mathbb{R}^3$ is the $i^{\text{th}}$ feature's normal vector and $\Delta^i\in\mathbb{R}$ is the distance at which the feature is placed w.r.t the origin along $\bm{n}^{i}$. From~\eqref{eq:bfunc EXP}--\eqref{eq:des bunc EXP}, 	the sets $\setC^i$ and $\setC^i_{\rm d}$ are defined as in~\eqref{eq:act C} and~\eqref{eq:des C}, respectively. 
QP~\eqref{eq:robust QP for combination} is formulated to steer the right hand to its targets while the constraints set contains the contact forces constraint~\eqref{eq:contact forces} and non-slipping contacts~\eqref{subeq:contact constraint}. RECBF constraint~\eqref{eq:RECBF formulation} is then inserted if $\bfunc^i \!\leq\! 4$~cm\footnote{The value of this threshold is conservative (user choice), but it can be decided by a high-level task scheduler or a task planner. In this experiment, we chose this value to confine the region within which the CoM acceleration is unconstrained. This helps to keep the CoM acceleration at low values.}. If QP is feasible then the CoM deceleration generated by~\eqref{eq:RECBF formulation} is achieved relying on feasible contact forces and consistent floating-base solution $\desFBDDot$.
In this experiment, $\desConf$ is the input command for the joint controllers.

For comparison purpose, three identical experimental scenarios have been conducted using the different closed-loop QP control schemes shown in \cref{fig:QP scheme for kinematic-control robots}:
\begin{itemize}
\item Experiment~1: feedforward ECBF constraint~\eqref{eq:ECBF constraint};
\item Experiment~2: feedback ECBF constraint~\eqref{eq:ECBF in feedback};
\item Experiment~3: RECBF constraint~\eqref{eq:RECBF formulation}.
\end{itemize}

The results are shown in \cref{fig:com}. In the three experiments,  $\constraintStiffness$ is computed as shown in~\cite{djeha2020ral} (\cref{thm:RECBF}). For Experiments~1 and~2, $\constraintDamping = 2.4\sqrt{\constraintStiffness}$. 
\subsubsection{Experiment 1} see~\cref{fig:com}\subref{subfig:com_OL}; the actual robot state $\genActJointDyn$ is not fed back to QP.  
We can see that $\desCoM$ is within the limits and the set $\setCd$~\eqref{eq:des C} is made forward invariant. However, since the robot is accounted for in the closed-loop system, forward invariance is not ensured for the set $\setC$~\eqref{eq:act C}. The mismatch between $\desCoM$ and $\actCoM$  leads the latter to overshoot   $X_{\max}$ limit with an amount of $2$~cm, then to completely drift away from the polygon boundary at $t=50$~s leading the robot to lose balance.
\subsubsection{Experiment 2} see~\cref{fig:com}\subref{subfig:com_CL}; the robot state is considered in closed-loop, but the ECBF formulation leads to instability. The bold dashed line shows the moment when ECBF constraint relative to ${X}_{\max}$ limit is inserted. Instantaneously, $\desCoM$ velocity along ${X}$-axis starts to decrease (brown dash-dotted line). Nevertheless, $\actCoM$ velocity does not decrease immediately causing $\actCoM$ to keep moving toward ${X}_{\max}$ boundary. This lag is due to the underdamped flexibilities dynamics. 
In fact, the ECBF produced deceleration is mapped mainly by QP to the ankles joints through~\eqref{eq:mapping Bfunc mu to u} as a little motion at these joints leads to a larger motion of the robot whole-body. However, the flexibilities underdamped dynamics leads $\actCoM$ to overshoot $\desCoM$ and thereby the error $\bfuncTrackErr$ in~\eqref{eq:bfunc track err} increases:  at $t=13$~s, $\desCoM$ velocity is zero whereas $\actCoM$ is heading toward ${X}_{\max}$ boundary with a velocity of $0.10$~m/s. Then at $t=13.3$~s, $\actCoM$ velocity reaches zero while $\desCoM$ is close to ${X}_{\min}$  boundary with a velocity of $-0.16$~m/s. When $\actCoM$ starts moving backward, its velocity increases highly leading to insert the ECBF constraint (relative to ${X}_{\min}$ boundary) with $\actCoM$ velocity reaching $-0.22$~m/s. At this point, the needed deceleration to stop $\actCoM$ at $X_{\min}$ boundary is high enough so that the QP fails to find corresponding feasible contact forces fulfilling~\eqref{eq:contact forces}, and the feet tip over.
\subsubsection{Experiment 3}
see~\cref{fig:com}\subref{subfig:com_Robust}; the constraint integral gain $\constraintIntegralDamping = 8.4\sqrt{\constraintStiffness}$ and $\constraintDamping = -1.2\sqrt{\constraintStiffness}<0$. Note that \cref{thm:RECBF} requirements are satisfied since $\constraintDamping+\constraintIntegralDamping>0$  and thereby the eigenvalues of $\FdesTaskOutRobust$ in~\eqref{eq:bfunc des dynamics - perturbed system} are strictly negative. In particular, RECBF constraint~\eqref{eq:RECBF formulation} writes similarly to~\eqref{eq:heterogeneous feedback with negative Kv}
\begin{align}\label{eq:RECBF formulation with negative Lv}
\begin{split}
\bfuncCrtlIn&\!\geq\! - 7.2\sqrt{\constraintStiffness}\desBfuncDot \!- 0.6\sqrt{\constraintStiffness}\left(\desBfuncDot\!-\!\bfuncDot\right) \!-\! \constraintStiffness\bfunc.
\end{split}	
\end{align}
The feedback term $(\desBfuncDot\!-\!\bfuncDot)$ in~\eqref{eq:RECBF formulation with negative Lv} helps to withstand the flexibilities effect. In fact, \cref{fig:com_Robust Zoom} shows that, when~\eqref{eq:RECBF formulation with negative Lv} relative to ${X}_{\max}$ boundary is inserted, $\desCoM$ velocity converges to zero while drifting toward $\actCoM$ velocity (${X}$ coordinates). Consequently, the delay between $\desCoM$ and $\actCoM$ states is lowered. This compliance behavior is the key factor behind avoiding over-regulation that leads to excessive deceleration in Experiment~2.
Also, compared to Experiment~1, $\desCoM$ compensates for joint-dynamics static error allowing $\actCoM$  to converge asymptotically to the polygon boundaries. 
More related to \cref{thm:RECBF}, the set $\setCd$ is made robustly stable, where the maximum overshoots along $X$ and $Y$ axes are:  $X_{\max}^{\rm{overshoot}} = 5.5$~mm, $X_{\min}^{\rm{overshoot}} = 3.2$~mm, $Y_{\max}^{\rm{overshoot}} = 6.5~$mm, $Y_{\min}^{\rm{overshoot}} = 4$~mm.  
\subsubsection{Experiment 4}
A fourth experiment is conducted to show the robust stability of $\setCd$ against external pushes using the same RECBF constraint~\eqref{eq:RECBF formulation with negative Lv}. Because of the high-stiffness joint controllers and high gear-ratio, the effect of the external disturbance forces is hardly observed at the joints' encoders. Yet, it can be measured by the floating-base observer affecting thereby the $\actCoM$ state. 

First, a Cartesian target is defined for HRP-4 right hand such that $\actCoM$ reaches the  polygon boundaries ${X}_{\max}$ and ${Y}_{\max}$. Then, the robot receives multiple external pushes from the operator (at the back and the shoulders) along ${X}$ and ${Y}$ axes (\cref{fig:com push snapshots}). Three persistent disturbance forces are applied, followed by two brief disturbances leading the flexibilities effect to enter into play (\cref{fig:com push}). During the whole experiment,  $\actCoM$ is pushed away from the polygon boundaries with a distance of at least $2$~cm. 
Here again, we can see the effect of the compliance feedback term. At the beginning of the persistent disturbance forces, $\desCoM$ slightly complies with the disturbance. Then, when the compliance term is less predominant, the compliance is lost and  QP  generates solutions such that $\desCoM$ counterbalances stiffly the disturbances enforcing $\actCoM$ to converge back smoothly to the polygon boundaries (\cref{fig:com push}\subref{subfig:com persistent push}). When applying brief perturbations (\cref{fig:com push}\subref{subfig:com brief push}), $\actCoM$ converges asymptotically to the polygon boundaries while complying to the transient flexibility response.  As in Experiment~3, the set $\setCd$ is made robustly stable. Nevertheless and similarly to \cref{subsec:experiment_robust stability}, it goes at the expanse of $\constraintIntegralDamping$  conservative tuning. 

\section{Conclusion}\label{sec:Conclusion}
In this paper, we propose a stable and robust closed-loop implementation of task-space QP control scheme for kinematic-controlled robots. Our solution allows free task-gains tuning and robust constraints design in the presence of non-modeled dynamics like joint-dynamics, flexibilities, and external disturbances. 
Our approach is proved to ensure the closed-loop stability by including integral feedback terms at both task and constraint levels leading to a robust convergence of their trajectories to the respective residual sets. Our method does not need the exact knowledge of the joint-dynamics model, but requires it to be ISS. Several experiments have been conducted on both floating-base and fixed-base robots to assess our new QP controller. Although not tackled in this paper, our approach can be extended to contact force control formulated as an admittance task~\cite{bouyarmane2019tro}.		
Future works will focus on reducing the conservativeness on the choice of the integral feedback gains $\taskIntegralDamping$ and $\constraintIntegralDamping$. Also, the conflict of RECBF with other constraints that leads to QP infeasibility is still an open problem. Up to now, constraints compatibility in QP control paradigms is among the main open questions that have not been well addressed and where model predictive control could be a candidate approach. 



\appendix
\subsection{Notations and Definitions}\label{app:notations}
Bold small letters stand for vectors, bold capital letters for matrices, and normal letters for scalars. In this work, there are three classes of variables:
\begin{enumerate}
\item those with subscript $_{\rm ref}$  are the task-space reference targets given either by the operator or a task planner;  
\item those with subscript $_{\rm d}$ are the desired variables in (i) the joint-space resulting from the integration of the desired acceleration (direct output of the QP), or (ii) in the task-space which are the mapping of the former; and
\item without any subscript are the variables tracking the desired once in~2)
\end{enumerate}
--~$\mathbb{R}$ and $\mathbb{R}^{+}$ are the sets of real and non-negative real numbers, respectively. For $\boldsymbol{\chi}\in \cal X$, $\boldsymbol{\alpha}_{\boldsymbol{\chi}}$ is the \emph{velocity} of $\boldsymbol{\chi}$. If $\cal X$ is Euclidean then $\boldsymbol{\alpha}_{\boldsymbol{\chi}} = {\boldsymbol{\dot{\chi}}}$.
$|\boldsymbol{\chi}|$ and $\norm{\boldsymbol{\chi}}$ denote the component-wise absolute value and the Euclidean norm of $\boldsymbol{\chi}$, respectively.
$\norm{\boldsymbol{\chi}}_\infty = \underset{t\geq0}{\sup}\norm{\boldsymbol{\chi}(t)}$, $\boldsymbol{\chi}\in\mathbb{R}^{x}$ is said to be bounded if~$\norm{\boldsymbol{\chi}}_\infty\!<\!\infty$. 
The transpose of $\boldsymbol{\chi}$ is denoted $\boldsymbol{\chi}\tp$.
$\underline{\lambda}(\mathbf{A}),\overline{\lambda}(\mathbf{A})$ denote respectively the minimum and maximum eigenvalues of matrix $\mathbf{A}$. All Jacobian matrices used in this work are assumed to be non-singular.

--~$\gamma\!:\!\mathbb{R}^+\!\rightarrow\! \mathbb{R}^+$ is a class $\cal K_\infty$  function if it is continuous, strictly increasing, $\gamma(0)=0$ and $\gamma(s)\!\overset{s \!\rightarrow\!\infty}{\longrightarrow} \!\infty$. 

--~$\beta\!:\mathbb{R}^{+}\!\times\!\mathbb{R}^{+}\!\rightarrow\! \mathbb{R}^{+}$ is a class $\cal KL$  function if for each fixed $t \!\geq\!0$, $\beta(s,t)$ is a class $\cal K$ function, and for each fixed $s \!\geq\!0$, it  decreases to $0$ as $t \!\rightarrow\!\infty$.

$\norm{\text{ }}_\Omega$ denotes the Euclidean point-to-set distance: $\norm{\boldsymbol{\chi}}_\Omega\!= \! \mathrm{dist}(\boldsymbol{\chi};\Omega)\!=\!\inf\left\{\mathrm{dist}(\boldsymbol{\chi},\bm{a})|\bm{a} \!\in\!\Omega\right\}\!=\!\underset{\bm{a} \in\Omega}{\inf}\norm{\boldsymbol{\chi}\!-\! \bm{a}}$.

--~Let us consider the system 
\begin{equation}\label{eq:system with input}
\boldsymbol{\dot{\chi}} = f_{\boldsymbol{\chi}} (\boldsymbol{\chi},\boldsymbol{\upsilon}),
\end{equation} 
where $\boldsymbol{\chi}\in\mathbb{R}^x$ and $\boldsymbol{\upsilon}\in\mathbb{R}^u$. 
For any initial condition $\boldsymbol{\chi}(t_0)\in\mathbb{R}^x$, there exists a maximum time interval $I\left(\boldsymbol{\chi}(t_0)\right)=[t_0,t_{\max}]$ such that $\boldsymbol{\chi}(t)$ is the unique solution of~\eqref{eq:system with input} on $I(\boldsymbol{\chi}(t_0))$. If $t_{\max}=\infty$ then $f_{\boldsymbol{\chi}(t_0)}$ is forward complete. System~\eqref{eq:system with input} is said to be \emph{autonomous} when $\boldsymbol{\upsilon}=0$.  A set $\setS\subset\mathbb{R}^x$ is called \emph{forward invariant} w.r.t autonomous system~\eqref{eq:system with input} if $\forall \boldsymbol{\chi}(t_0)\in\setS$ then $\boldsymbol{\chi}(t)\in\setS, \ \forall t\in I(\boldsymbol{\chi}(t_0))$. In addition, a closed and forward invariant set ${\setS}\subset\mathbb{R}^x$ is asymptotically stable for a forward-complete autonomous system~\eqref{eq:system with input} if there exist on open set $\cal R\supseteq\setS$, and a class $\classKL$ function $\beta$ such that
\begin{equation*}\label{eq:asymptotic stability of a set}
\norm{\boldsymbol{\chi}(t)}_{\setS}\leq\beta\left(\norm{\boldsymbol{\chi}(t_0)}_{\setS},t-t_0\right), \ \forall \boldsymbol{\chi}(t_0)\in\cal R.
\end{equation*}
--~\textbf{Robust Global Uniform Asymptotic Stability:}~\cite{freeman1994cdc,freeman1996bookRobustNNlinearControl}
Consider the system
\begin{equation}\label{eq:system with input and time}
\boldsymbol{\dot{\chi}} = f_{\boldsymbol{\chi}} (\boldsymbol{\chi},\boldsymbol{\upsilon},t).
\end{equation}
Fix a control $\boldsymbol{\upsilon}$, and let $\Omega\subset \cal X$ be a compact set containing the origin. The solutions of the system~\eqref{eq:system with input and time} are Robustly Globally Uniformly Asymptotically Stable w.r.t $\Omega$ (RGUAS-$\Omega$) when there exists class ${\cal KL}$ function $\beta$ such that for all admissible measurements, admissible disturbance, and initial conditions $(\boldsymbol{\chi}(t_0),t_0)\in {\cal X}\times\mathbb{R}$, all solutions $\boldsymbol{\chi}(t)$ exist and satisfy
\begin{equation}\label{eq:def of RGUAS-omega}
\norm{\boldsymbol{\chi}(t)}_\Omega\leq \beta(\norm{\boldsymbol{\chi}(t_0)}_\Omega,t-t_0).
\end{equation}
System~\eqref{eq:system with input and time} is \textbf{robustly practically stabilizable} when $\forall \epsilon>0$ there exist an admissible control and a compact set $\Omega\subset{\cal X}$ satisfying $0\in\Omega\subset\epsilon \cal B$, with $\cal B$ the unit ball set, such that the solutions  $\boldsymbol{\chi}(t)$ are RGUAS-$\Omega$.\\
--~\textbf{Rayleight-Ritz Inequality}:~\cite{rugh1996LinearSystemTheory}\label{app:rayghleight}
Given a symmetric matrix $\mathbf{A}\in\mathbb{R}^{x\times x}$ the following inequality  holds $\forall \boldsymbol{\chi}\in\mathbb{R}^x$:
\begin{equation}\label{eq:rayleight-ritz inequality}
\underline{\lambda}(\mathbf{A})\norm{\boldsymbol{\chi}}^2\leq \boldsymbol{\chi}\tp \mathbf{A}\chi\leq\overline{\lambda}(\mathbf{A})\norm{\boldsymbol{\chi}}^2 .
\end{equation}
--~\textbf{Schwartz inequality}:~\cite{strang1988bookLinearAlgebra}\label{app:schwartz}
$\forall \boldsymbol{\chi},\boldsymbol{\zeta} \in\mathbb{R}^x$,
\begin{equation}\label{eq:schwartz inequality}
|\boldsymbol{\chi}\tp\boldsymbol{\zeta}|\leq\norm{\boldsymbol{\chi}}\norm{\boldsymbol{\zeta}}.
\end{equation}

\subsection{Proof of \cref{prop 1}}\label{proof:prop1}
\begin{proof} 
The proof is established for $\desTaskOut $, the same steps apply for $\desBfuncOut $. Here, the dependency on time $(t)$ is made explicit. Given~\eqref{eq:des task dyn}, let us assume  $\exists\taskCrtlIn\in\mathbb{R}^{m} \;|\; \desTaskOut (t)$ is bounded
\begin{align}\label{eq:boundedness of xi}
\begin{split}
\norm{\desTaskOut(t)} \leq M &\Rightarrow \left\{\begin{matrix}
\norm{\desOut(t)}  \leq M\\ 
\norm{\desOutDot(t)}  \leq M 
\end{matrix}\right., \forall t\geq0, \\
&\eqref{eq:des output vector}\Leftrightarrow \left\{\begin{matrix}
\norm{\fkin(\genDesConf(t)) - \fkinRef(t)}  \leq M,\\ 
\norm{\desJac\genDesConfDot(t) - \fkinRefDot(t) }  \leq M ,
\end{matrix}\right.
\end{split} 
\end{align} where $\fkinRef(t)$ and $\fkinRefDot(t)$ are obviously bounded. 
Let us prove that if $\desOut(t) $ is bounded then $\genDesConf(t) $ is bounded. 
Let us define $\genRefConf(t)\in{\cal Q}$\footnote{In~\eqref{eq:cal Q def}, $\fkinRef$ is assumed to be strictly reachable. Otherwise, we can define $\cal Q$ as ${\cal Q}\! =\!\left\{\genDesConf^*(t)\! \in\!\mathbb{R}^{7+n}\!:\!\genDesConf^* \!=\! \arg\min\norm{\fkin\left(\genDesConf(t) \right)\!-\!\fkinRef(t)}\right\}$ which leads to $\fkin\left(\genRefConf(t)\right) \!=\! \fkinRef(t) \!+\! \boldsymbol{\delta}_{\fkin}(t)$, with $\boldsymbol{\delta}_{\fkin}(t)\!\inR^{m}$ bounded. The remaining of the proof is not affected.} with 
\begin{equation}\label{eq:cal Q def}
{\cal Q} =\left\{\genDesConf(t) \in\mathbb{R}^{7+n}:\fkin\left(\genDesConf(t) \right)=\fkinRef(t)\right\}.
\end{equation}
Thus, by putting $\Delta \genDesConf(t)\!=\! \genDesConf(t)\!  -\! \genRefConf(t)$, $\fkin\left(\genDesConf(t) \right)$  writes using Taylor expansion
\begin{align}\label{eq:DL y_d}
\begin{split}
\fkin\left(\genDesConf(t)\right) &= \fkin\left(\genRefConf(t) + \Delta \genDesConf(t)\right),\\
&= \fkinRef(t) + R(\Delta \genDesConf(t)),\\ R(\Delta \genDesConf(t))&=\left.\begin{matrix}
\frac{\partial \fkin(\genDesConf(t))}{\partial \genDesConf(t)}
\end{matrix}\right|_{\genDesConf(t)=\check{\mathbf{\actConf}}_{\mathrm{d}}(t)}\Delta \genDesConf(t),\\ 
\end{split}
\end{align} 
where $R(\Delta \genDesConf(t))$ is the Lagrange remainder with $\check{\genActConf}_{\mathrm{d}}(t)\!=\! \genRefConf(t)\!+\theta\Delta \genDesConf(t)$, $0\leq\theta\leq1$~\cite[Chapter IV, Section 6]{piskunov1969book}.
From~\eqref{eq:boundedness of xi},~\eqref{eq:DL y_d} we have
\begin{equation}
\norm{R\left(\Delta \genDesConf(t)\right)}\leq M.
\end{equation}
If $\frac{\partial \fkin(\genDesConf(t))}{\partial \genDesConf(t)}$ is non-singular, then $\forall \theta$ with $0\leq\theta\leq1$ there exist $b,b_0$, $b\geq b_0>0$ such that~\cite{golub2013matrixComputations}
\begin{align}\label{eq:qd bounded}
\begin{split}
& \norm{R\left(\Delta \genDesConf(t)\right)}  = b \norm{\Delta \genDesConf(t)}  \Rightarrow  \norm{\Delta \genDesConf(t)} \leq \frac{M}{b}, \\
&\Leftrightarrow  \norm{\genDesConf(t)}  -  \norm{\genRefConf(t)}  \leq  \norm{\Delta \genDesConf(t)} \leq \frac{M}{b}, \\
&\Rightarrow  \norm{\genDesConf(t)}  \leq \frac{M}{b} + \norm{\genRefConf(t)} .
\end{split}
\end{align}
Given that $\genRefConf(t)$ is bounded, then $\genDesConf(t)$ is bounded.  

Now, let us prove that if $\desOutDot(t)$ is bounded then $\genDesConfDot(t)$ is bounded. 
$\genDesConfDot(t)$ can be written as $\genDesConfDot(t) = \hat{\boldsymbol{\alpha}}_{{\genActConf}_\mathrm{d}}(t) + \boldsymbol{\alpha}^{\#}_{{\genActConf}_\mathrm{d}}(t)$ such that  $\boldsymbol{\alpha}^{\#}_{{\genActConf}_\mathrm{d}}(t)\in\ker\{\desJac\}$ with $\boldsymbol{\alpha}^{\#}_{{\genActConf}_\mathrm{d}}(t) = \left(\mathbf{I}-{\desJac}^{+}\desJac\right)\boldsymbol{\nu}(t)$, where ${\desJac}^{+}$ is the Moore-Penrose Jacobian inverse and $\boldsymbol{\nu}(t)\in\mathbb{R}^{6+n}$ denotes the remaining velocity redundancy.  In QP~\eqref{eq:robust QP for combination}, the redundancy state is bounded by a secondary (posture) task. Furthermore,~\eqref{eq:contact forces} and~\eqref{subeq:contact constraint} ensures bounded and feasible floating base solutions.  Hence, $\boldsymbol{\nu}(t)$ is bounded. 		
Let us show the boundedness of $\hat{\boldsymbol{\alpha}}_{{\genActConf}_\mathrm{d}}(t)$. From~\eqref{eq:boundedness of xi} 
\begin{align}
\begin{split}
\norm{\desJac\hat{\boldsymbol{\alpha}}_{{\genActConf}_\mathrm{d}}(t)} \!-\! \norm{\fkinRefDot(t)}\leq\norm{\desJac\hat{\boldsymbol{\alpha}}_{{\genActConf}_\mathrm{d}}(t) \!-\! \fkinRefDot(t) }  \!\leq\! M, \\ 
\Rightarrow \norm{\desJac\hat{\boldsymbol{\alpha}}_{{\genActConf}_\mathrm{d}}(t)} \leq M + \norm{\fkinRefDot(t)}.
\end{split}
\end{align}
Given that $\desJac$ is non-singular, then there exist $b',b'_0$ with $b'\!\geq\! b'_0\!>\!0$ such that~\cite{golub2013matrixComputations}
\begin{align}
\begin{split}
\norm{\desJac \hat{\boldsymbol{\alpha}}_{{\genActConf}_\mathrm{d}}(t)}  = b' \norm{\hat{\boldsymbol{\alpha}}_{{\genActConf}_\mathrm{d}}(t)}&\leq M+ \norm{\fkinRefDot(t)},\\ \Rightarrow  \norm{\hat{\boldsymbol{\alpha}}_{{\genActConf}_\mathrm{d}}(t)} &\leq \frac{M+ \norm{\fkinRefDot(t)}}{b'}.
\end{split}
\end{align}
Hence, $\genDesConfDot(t)$ is bounded such that
\begin{equation}\label{eq:dot_qd bounded}
\norm{\genDesConfDot(t)}\leq\norm{\hat{\boldsymbol{\alpha}}_{{\genActConf}_\mathrm{d}}(t)} + \norm{\mathbf{I}-{\desJac}^+\desJac}\norm{\boldsymbol{\nu}(t)}.
\end{equation}
From~\eqref{eq:gen joint dyn def} and following from~\eqref{eq:dot_qd bounded} and~\eqref{eq:qd bounded} $\genDesJointDyn(t)$ is bounded 
implying that, given~\eqref{eq:mapping mu to u},  $\exists\UgenDesJointDyn\in {\cal U}$ such that $\genDesJointDyn(t)$ is bounded. 
\end{proof}
\subsection{Proof of \cref{prop 2}}\label{proof:prop2}
\begin{proof}
As in \cref{prop 1} proof, \cref{prop 2} proof is established for $\desTaskOut $ (the same steps apply for $\desBfuncOut $) and the dependency on time $(t)$ is made explicit.
Let us consider system~\eqref{eq:des task dyn} and assume that there exists $\taskCrtlIn$ such that $\desTaskOut(t) $ is (uniformly) ultimately bounded. Then, there exists an ultimate bound $\varrho_{\desTaskOut}>0$, such that $\forall {M}_{\desTaskOut}>0$, $\exists T_{\desTaskOut}=T_{\desTaskOut}(M_{\desTaskOut},\varrho_{\desTaskOut})>0$ such that~\cite[Definition~4.6]{khalil2002NonLinearSystems}
\begin{equation}\label{eq:uniform ultimate boundedness def}
\norm{\desTaskOut(t_0)}\leq M_{\desTaskOut} \Rightarrow \norm{\desTaskOut(t) } \leq \varrho_{\desTaskOut}, \ \forall t\geq t_0 +T_{\desTaskOut}.
\end{equation}
From \cref{prop 1} proof in \cref{proof:prop1} and~\eqref{eq:uniform ultimate boundedness def}, it yields that there exists $\varrho_{\desJointDyn}=\varrho_{\desJointDyn}(\varrho_{\desTaskOut})>0$ such that $\norm{\desJointDyn(t)}\leq\varrho_{\desJointDyn}, \ \forall t\geq t_0 +T_{\desTaskOut}$.
Hence, from \cref{assum1}, we get 
\begin{align*}
\begin{split}
\norm{\jointDynTrackErr(t)} &\!=\! \norm{\actJointDyn(t) - \desJointDyn(t)}\leq \sigma  \\	&\!\Rightarrow\! \norm{\actJointDyn(t)} \!\leq \sigma \!+ \! \norm{\desJointDyn(t)}\leq \sigma \!+\! \varrho_{\desJointDyn}, \ \forall t\geq t_0 \!+ \! T_{\desTaskOut}.
\end{split}
\end{align*}
In addition, the robot floating-base state $\actFBDyn(t) $ is bounded by assumption which leads to the boundedness of $\genActJointDyn$~\eqref{eq:gen joint dyn def}. 
Thus, from~\eqref{eq:DL act task} and~\eqref{eq:uniform ultimate boundedness def}, 
there exists an ultimate bound $\varrho_{\actTaskOut}= \varrho_{\desTaskOut} +\norm{\taskTrackErr }_\infty>0$, such that $\forall {M}_{\actTaskOut}=M_{\desTaskOut} + \norm{\taskTrackErr(t_0)}>0$, there exists $ T_{\actTaskOut}=T_{\actTaskOut}({M}_{\actTaskOut},\varrho_{\actTaskOut})>0$ such that 
\begin{equation}
\norm{\actTaskOut(t_0)}\leq {M}_{\actTaskOut} \Rightarrow \norm{\actTaskOut(t)} \leq \varrho_{\actTaskOut}, \ \forall t\geq t_0 + T_{\actTaskOut},
\end{equation} which yields to $\actTaskOut(t) $ is (uniformly) ultimately bounded.
\end{proof}

\subsection{Proof of \cref{thm:heterogeneous feedback}}\label{proof:thm hetero feedback}
\begin{proof}
Replacing~\eqref{eq:heterogeneous feedback mu} in~\eqref{eq:des task dyn} yields to 
\begin{align}\label{eq:des task dynamics around equilibrium-hetero approach} 
\desTaskOutDot = \FdesTaskOutRobust\desTaskOut-\BdesTaskOut \taskGains\taskTrackErr,
\end{align} where $\FdesTaskOutRobust$ and $\taskGains$ are defined as in~\cref{thm:heterogeneous feedback} and~\eqref{eq:mu output feedback}, respectively.
Let us consider the following Lyapunov function associated to~\eqref{eq:des task dynamics around equilibrium-hetero approach}
\begin{align}\label{eq:Lyapunov function xi heterogeneous feedback}
\gamma_1\left(\norm{\desTaskOut}\right)\leq V(\desTaskOut) = \frac{1}{2}\desTaskOut\tp \mathbf{P}_\desTaskOut\desTaskOut\leq\gamma_2\left(\norm{\desTaskOut}\right)
\end{align}
where $\gamma_1\left(\norm{\desTaskOut}\right)\!=\!\frac{\underline{\lambda}(\mathbf{P}_\desTaskOut)}{2}\norm{\desTaskOut}^2$ and $\gamma_2\left(\norm{\desTaskOut}\right)\!=\!\frac{\overline{\lambda}(\mathbf{P}_\desTaskOut)}{2}\norm{\desTaskOut}^2$ are class $\cal K_\infty$ functions, and $\mathbf{P}_\desTaskOut=\mathbf{P}_\desTaskOut\tp>0$ is the solution of the following Algebraic Riccati Equation (ARE)
\begin{equation}\label{eq:riccati eq for heterogeneous feedback}
\FdesTaskOutRobust\tp\mathbf{P}_\desTaskOut + \mathbf{P}_\desTaskOut \FdesTaskOutRobust = -\mathbf{Q}_\desTaskOut = -\begin{bmatrix}
\taskIntegralDamping & 0 \\ 0 & \taskIntegralDamping
\end{bmatrix}.
\end{equation}
Given~\eqref{eq:riccati eq for heterogeneous feedback}, $\dot{V} = -\frac{1}{2}\desTaskOut\tp \mathbf{Q}_\desTaskOut\desTaskOut - \desTaskOut\tp \mathbf{P}_\desTaskOut \BdesTaskOut \taskGains\taskTrackErr$.  
%
Given that $\norm{\BdesTaskOut}=1$,  $\norm{\mathbf{P}_\desTaskOut}=\overline{\lambda}(\mathbf{P}_\desTaskOut)$ and $\underline{\lambda}(\mathbf{Q}_\desTaskOut) =\underline{\lambda}(\taskIntegralDamping)>0$, and using Rayleight-Ritz~\eqref{eq:rayleight-ritz inequality} and Schwartz~\eqref{eq:schwartz inequality} inequalities, $\dot{V}$ is bounded such that
\begin{align}
\begin{split}
\dot{V}&\leq -\frac{1-\vartheta}{2}\underline{\lambda}(\taskIntegralDamping)\norm{\desTaskOut}^2\\
& - \frac{\vartheta}{2}\underline{\lambda}(\taskIntegralDamping)\norm{\desTaskOut}\left(\norm{\desTaskOut}\!-\frac{2\overline{\lambda}(\mathbf{P}_\desTaskOut)\norm{\taskGains}\norm{\taskTrackErr}}{\vartheta\underline{\lambda}(\taskIntegralDamping)}\right),
\end{split}
\end{align} with $0<\vartheta<1$.
Thus, if $\taskIntegralDamping$ is chosen such that
\begin{equation}\label{eq:cdt on norm - Lv}
\norm{\desTaskOut}\geq\frac{2\overline{\lambda}(\mathbf{P}_\desTaskOut)\norm{K}}{\vartheta\underline{\lambda}(\taskIntegralDamping)}\norm{\taskTrackErr}_\infty,
\end{equation}
then $\dot{V}\leq-\frac{1-\vartheta}{2}\underline{\lambda}(\taskIntegralDamping)\norm{\desTaskOut}^2$.
By the virtue of~\cite[Theorem~4.18]{khalil2002NonLinearSystems}, $\desTaskOut $ is uniformly ultimately bounded with ultimate bound $\varrho= \sqrt{\frac{\overline{\lambda}(\mathbf{P}_\desTaskOut)}{\underline{\lambda}(\mathbf{P}_\desTaskOut)}}\frac{2\overline{\lambda}(\mathbf{P}_\desTaskOut)\norm{\taskGains}}{\vartheta\underline{\lambda}(\taskIntegralDamping)}\norm{\taskTrackErr}_\infty$.
Furthermore, $\forall \norm{\desTaskOut(t_0)} \leq M$ there exist $T= T(M,\varrho)>0$, a class $\classKL$ function $\beta$, and a closed set  $\Omega_\desTaskOut \!=\! \left\{\desTaskOut \!\in\! H:\norm{\desTaskOut }\!\leq\!\varrho\right\}$  such that 
\begin{align}\label{eq:UB des task state}
\begin{split}
\norm{\desTaskOut(t) }_{\Omega_\desTaskOut}&\!\leq\!\beta\left(\norm{\desTaskOut(t_0)}_{\Omega_\desTaskOut},t-t_0\right), \ \forall t_0\!\leq\! t \!\leq\! t_0 +T, \\
\norm{\desTaskOut(t) }_{\Omega_\desTaskOut} &=0, \ \forall t \leq t_0 +T.
\end{split}
\end{align}
Given~\eqref{eq:cdt on norm - Lv}, the residual set $\Omega_\desTaskOut$  can be made arbitrarily small by $\taskIntegralDamping$. Hence,  $\desTaskOut $ is robustly practically stable w.r.t $\Omega_\desTaskOut$~\cite[Defintion~3.2]{freeman1996bookRobustNNlinearControl}.
\end{proof}

\subsection{Proof of \cref{thm:RECBF}}\label{proof:RECBF}
\begin{proof}
The matrix gain $\KBfuncEq$ is chosen to ensure that $\desBfunc$ is ECBF for the nominal system $\bfuncTrackErr=\zero$. 

Now, let us prove that $\desBfuncOut $ is uniformly ultimately bounded. Inequality~\eqref{eq:RECBF formulation} can be expressed as
\begin{equation}\label{eq:mu for RECBF}
\bfuncCrtlIn = -\constraintGainsPsi\bfuncPsi + \bfuncDelta(t), \ 0\leq\bfuncDelta(t)\leq\bfuncDeltaMax,
\end{equation}
with $\bfuncDelta(t)$ a slack variable that facilitates the manipulation of~\eqref{eq:RECBF formulation}. Given~\eqref{eq:mu for RECBF}, system~\eqref{eq:Bfunc transverse dyn} becomes
\begin{equation}\label{eq:bfunc des dynamics - perturbed system}
\desBfuncOutDot = \FdesBfuncOutRobust\desBfuncOut + \BdesBfuncOut\left(-\constraintGains\bfuncTrackErr+\bfuncDelta(t)\right) ,
\end{equation}
where $\constraintGains$ defined as in~\eqref{eq:ECBF in feedback}, and $\FdesBfuncOutRobust= \AdesBfuncOut-\BdesBfuncOut\KBfuncEq$. Let us consider the following Lyapunov function~\cite{xu2015ifac}\footnote{Note that \eqref{eq:Lyapunov function for RECBF} allows to use the same theoretical tools as in the proof of \cref{thm:heterogeneous feedback} in \cref{proof:thm hetero feedback}.}
\begin{equation}\label{eq:Lyapunov function for RECBF}
V= \left\{\begin{matrix}
0, \text{ if } \genDesJointDyn\in\setCd\\
\frac{1}{2}\desBfuncOut\tp\PdesBfuncOut\desBfuncOut, \text{ otherwise } 
\end{matrix}\right.
\end{equation}
where $\PdesBfuncOut=\PdesBfuncOut\tp>0$ is the solution of the following ARE
\begin{equation}\label{eq:riccati eq for RECBF}
\FdesBfuncOutRobust\tp\PdesBfuncOut + \PdesBfuncOut \FdesBfuncOutRobust = -\QdesBfuncOut = -\begin{bmatrix}
\constraintIntegralDamping & 0 \\ 0 & \constraintIntegralDamping
\end{bmatrix}.
\end{equation}
The goal is to show that there exists a set $\setCd_{\sigma}\supseteq\setCd$ such that $\dot{V}<0,  \forall \genDesJointDyn \in \mathbb{R}^{13+2n}\setminus \setCd_{\sigma}$.  
Using~\eqref{eq:riccati eq for RECBF},  $\dot{V}$ is computed as 
\begin{align}\label{eq:Vdot RECBF 1}
\begin{split}
\dot{V}\!= \!-\frac{1}{2}\desBfuncOut\tp\QdesBfuncOut\desBfuncOut \!+\! \desBfuncOut\tp\PdesBfuncOut\BdesBfuncOut\left(-\constraintGains\bfuncTrackErr\!+\!\bfuncDelta(t)\right).
\end{split}
\end{align}
Using Rayleight-Ritz~\eqref{eq:rayleight-ritz inequality} and Schwartz~\eqref{eq:schwartz inequality} inequalities,~\eqref{eq:Vdot RECBF 1} becomes
\begin{align}
\begin{split}
\dot{V}&\leq -\frac{1}{2}\underline{\lambda}(\QdesBfuncOut)\norm{\desBfuncOut}^2 \\  &+ \norm{\desBfuncOut}\norm{\PdesBfuncOut}\norm{\BdesBfuncOut}\left(\norm{\constraintGains}\norm{\bfuncTrackErr}+\norm{\bfuncDelta(t)}\right).
\end{split}
\end{align}
By putting $\varphi = \norm{\constraintGains}\norm{\bfuncTrackErr}+\norm{\bfuncDelta(t)}$, and  given that $\underline{\lambda}(\QdesBfuncOut)=\constraintIntegralDamping>0$, $\norm{\PdesBfuncOut}=\overline{\lambda}(\PdesBfuncOut)$, $\norm{\BdesBfuncOut} = 1$,  then 
\begin{align}
\begin{split}
\dot{V}&\leq-\frac{1-\vartheta}{2}\constraintIntegralDamping\norm{\desBfuncOut}^2 \\
&-\frac{\vartheta}{2}\constraintIntegralDamping\norm{\desBfuncOut}\left(\norm{\desBfuncOut}-\frac{2\overline{\lambda}(\PdesBfuncOut)}{\vartheta\constraintIntegralDamping}\varphi\right),
\end{split}
\end{align} with $0\!<\!\vartheta\!<\!1$. 
Hence, if $\constraintIntegralDamping$ is chosen such that
\begin{equation}\label{eq:cdt on Lv RECBF}
\norm{\desBfuncOut}\geq\frac{2\overline{\lambda}(\PdesBfuncOut)}{\vartheta\constraintIntegralDamping}\varphi_\infty, \ \varphi_\infty = \norm{\constraintGains}\norm{\bfuncTrackErr}_\infty+\bfuncDeltaMax,
\end{equation}  then $\dot{V}\leq  -\frac{1}{2}\constraintIntegralDamping\norm{\desBfuncOut}^2$. 
By the virtue of~\cite[Theorem~4.18]{khalil2002NonLinearSystems}, $\desBfuncOut $ is uniformly ultimately bounded with ultimate bound $\sigma=\sqrt{\frac{\overline{\lambda}(\PdesBfuncOut)}{\underline{\lambda}(\PdesBfuncOut)}}\frac{2\overline{\lambda}(\PdesBfuncOut)}{\vartheta\constraintIntegralDamping}\varphi_\infty$. In addition, there exists a closed set $\setCd_{\sigma}$ which is asymptotically stable and forward invariant\footnote{Forward invariance and asymptotic stability follow from the uniform ultimate boundedness property of $\desBfuncOut$.}.
Given that $\setCd\subseteq\setCd_{\sigma}$ then following from \cref{def:RS-sigma}, $\setCd$ is robustly stable, and thereby, from \cref{def:RECBF}, $\desBfunc$ is a RECBF.
\end{proof}	

\subsection{Proof of \cref{prop:relaxed heterogeneous feedback}}\label{proof:prop3}
\begin{proof}
The superscript $^i$ is dropped for the sake of clarity.
Substituting~\eqref{eq:heterogeneous feedback relaxed} in~\eqref{eq:des task dyn} yields to 
\begin{align*}
\begin{split}
\desTaskOutDot&=\FdesTaskOutRobust\desTaskOut-\BdesTaskOut \taskGains\taskTrackErr+\BdesTaskOut\delta(t).
\end{split}
\end{align*}
Let us consider Lyapunov function $V$ in~\eqref{eq:Lyapunov function xi heterogeneous feedback} such that~\eqref{eq:riccati eq for heterogeneous feedback} holds. Following the same steps in \cref{thm:heterogeneous feedback} proof, $\dot{V}$ is bounded such that
\begin{align*}\label{eq:bounds on Vdot for relaxed task}
\begin{split}
&\dot{V} \! \leq\! -\frac{1}{2}(1\!-\!\vartheta)\underline{\lambda}(\taskIntegralDamping)\norm{\desTaskOut}^2 \\ 
&\!-\!\frac{\vartheta\underline{\lambda}(\taskIntegralDamping)}{2}\norm{\desTaskOut}\left(\norm{\desTaskOut}\!\!-\!\frac{2\overline{\lambda}(\mathbf{P}_\desTaskOut)}{\vartheta\underline{\lambda}(\taskIntegralDamping)}\left(\norm{\taskGains}\norm{\taskTrackErr}\!\!+\!\!\norm{\delta(t)}\right)\right).
\end{split}
\end{align*}
If $\taskIntegralDamping$ is chosen such that $\norm{\desTaskOut}\geq\frac{2\overline{\lambda}(\mathbf{P}_\desTaskOut)}{\vartheta\underline{\lambda}(\taskIntegralDamping)}\left(\norm{\taskGains}\norm{\taskTrackErr}_\infty\!+\!\delta_{\max}\right)$ then $\dot{V}\leq -\frac{1}{2}(1-\vartheta)\underline{\lambda}(\taskIntegralDamping)\norm{\desTaskOut}^2$. From~\cite[Theorem~4.18]{khalil2002NonLinearSystems}, it follows that $\desTaskOut$ is uniformly ultimately bounded with ultimate bound $\varrho= \sqrt{\frac{\overline{\lambda}(\mathbf{P}_\desTaskOut)}{\underline{\lambda}(\mathbf{P}_\desTaskOut)}}\frac{2\overline{\lambda}(\mathbf{P}_\desTaskOut)}{\vartheta\underline{\lambda}(\taskIntegralDamping)}\left(\norm{\taskGains}\norm{\taskTrackErr}_\infty + \delta_{\max}\right)$. Following the same steps in~\eqref{eq:UB des task state}, $\desTaskOut$ is robustly practically stable w.r.t the residual set $\Omega_\desTaskOut = \{\desTaskOut\in H: \norm{\desTaskOut}\leq\varrho\}$.
\end{proof}	

\bibliographystyle{ieeetr}
\bibliography{biblio.bib}

\vspace{-1cm}
\begin{IEEEbiography}[{\includegraphics[width=1in,height=1.25in,clip,keepaspectratio]{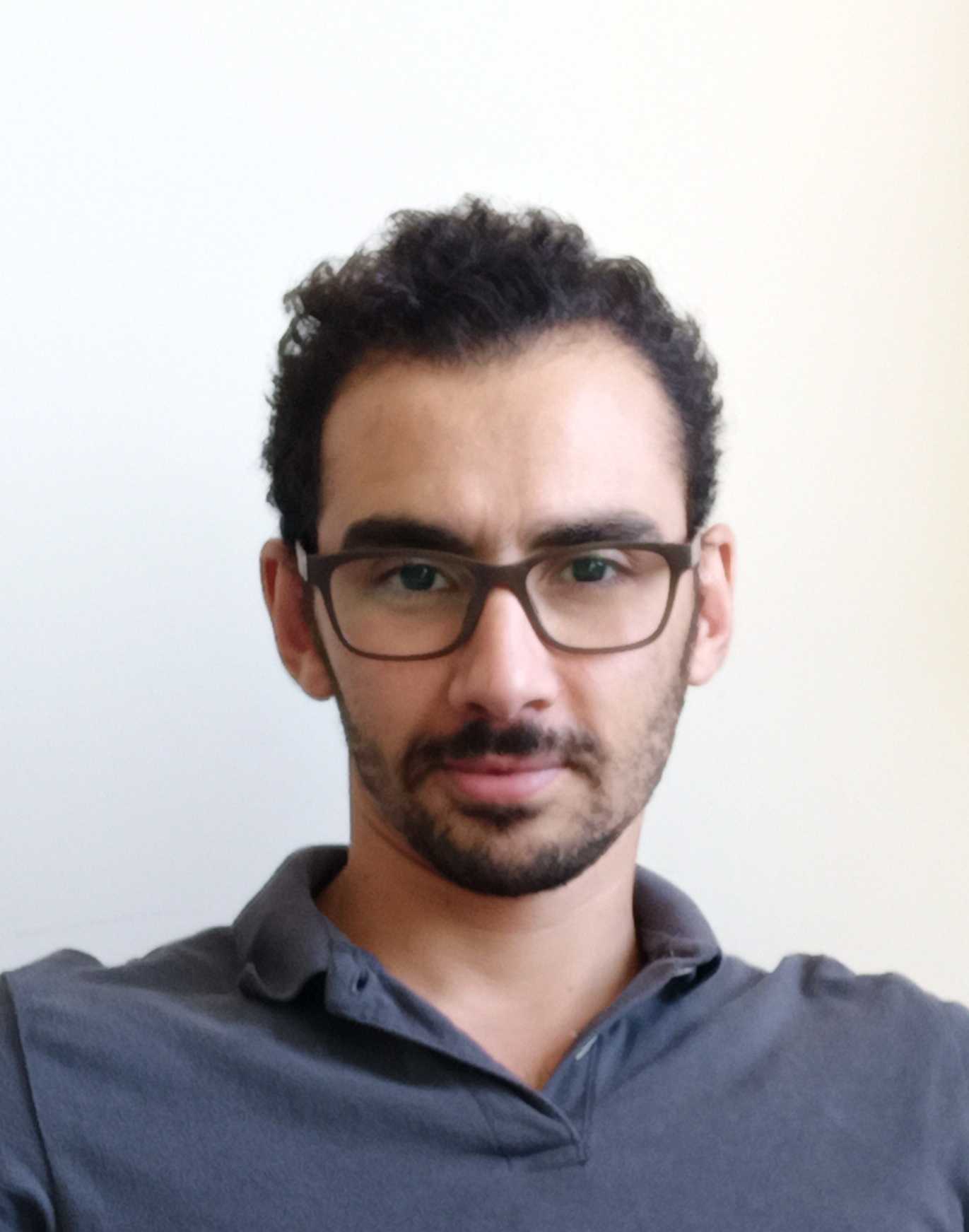}}]{Mohamed Djeha} received both the B.S. and M.S. degree in 2017 from Ecole Militaire Polytechnique (EMP), Algiers in Systems Control. He then received the Ph.D. degree in robotics from the University of Montpellier in September 2022. 
	
	His MS research was focused on exploiting electrophysiological signals to control robotics systems using machine learning and data fusion techniques. His Ph.D. research has been conducted at CNRS-University of Montpellier LIRMM in Montpellier on multi-objective and optimization-based control for redundant robotic manipulators and humanoids. Currently, he is a researcher at the Laboratory of Complex Systems Control and Simulators at EMP.
\end{IEEEbiography}
\vspace{-1cm}
\begin{IEEEbiography}[{\includegraphics[width=1in,height=1.25in,clip,keepaspectratio]{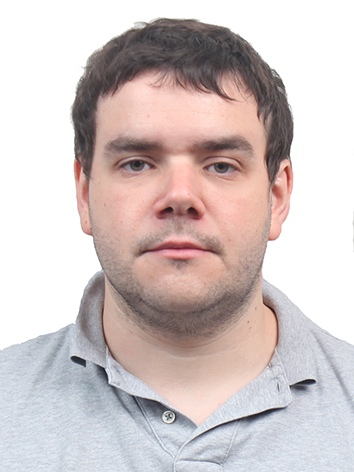}}]{Pierre Gergondet} received the M.S. degree in 2010 from \'Ecole Centrale de Paris with a speciality in embedded systems. He then received the Ph.D. degree in robotics from the University of Montpellier in 2014. 
	
	His Ph.D. research was conducted on controlling humanoid robot using brain-computer interfaces at the CNRS-AIST Joint Robotics Laboratory (JRL), IRL 3218 in Tsukuba Japan. He continued to work in JRL as a CNRS Research Engineer leading the software developments of the multi-contact real time framework: mc\_rtc. Between 2019 and 2022, he joined the Beijing Advanced Innovation Center for Intelligent Robots and Systems (BAICIRS) at the Beijing Institute of Technology (BIT) as a special associate researcher, he has since resumed his position at JRL. His current research interests include humanoid robots, control software for robotics and robotics applications.
\end{IEEEbiography}
\vspace{-1cm}
\begin{IEEEbiography}[{\includegraphics[width=1in,height=1.25in,clip,keepaspectratio]{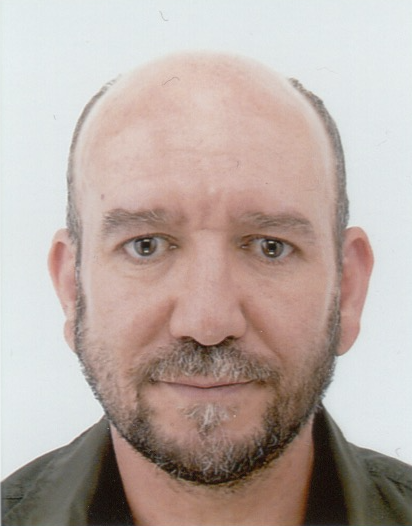}}]{Abderrahmane Kheddar} (F'22, SM’12, M'08) received the B.S. in Computer Science degree from the Institut National d'Informatique (ESI), Algiers, Algeria in 1990, and the M.Sc. and Ph.D. degree in robotics, both from Pierre et Marie Curie University, Sorbonne University, Paris, France 1993 and 1997, respectively. 

He is presently Directeur de Recherche at CNRS at the CNRS-AIST Joint Robotic Laboratory (JRL), IRL, Tsukuba, Japan (that he created in 2008, Director from 2008 to 2018). He also created and led the Interactive Digital Humans (IDH) team at CNRS-University of Montpellier LIRMM (from 2010 to 2020), France. His research interests include haptics, humanoids and thought-based control using brain machine interfaces. He is a founding member of the IEEE/RAS chapter on haptics, the co-chair and founding member of the IEEE/RAS Technical committee on model-based optimization, he is a member of the steering committee of the IEEE Brain Initiative, Editor of the IEEE Robotics and Automation Letters, Founding member and Deputy Editor-in-Chief of Cyborg and Bionics System (a Science partner journal). He was Editor of the IEEE Transactions on Robotics (2013-2018) and within the editorial board of other robotics journals; he is a founding member of the IEEE Transactions on Haptics and served in its editorial board during three years (2007-2010). He is an IEEE Fellow, AAIA Fellow and titular full member of the National Academy of Technology of France and knight of the national order of merits of France.
\end{IEEEbiography}

\vfill

\end{document}